\documentclass[11pt]{article}

\usepackage[margin=1in]{geometry}
\usepackage{natbib}

\usepackage[utf8]{inputenc} %
\usepackage[T1]{fontenc}    %
\usepackage{amsfonts}       %
\usepackage{nicefrac}       %
\usepackage{xcolor}         %
\usepackage{tikz}
\usepackage{float}
\usepackage{titletoc}
\usetikzlibrary{arrows.meta,calc,positioning,shapes.geometric}

\usepackage{page_setup}
\usepackage{notation}

\hypersetup{
  colorlinks = true,
  citecolor = blue,
}

\title{\papertitle}

\author{%
  Alexander Soen\textsuperscript{$\diamond$}
  \qquad
  Ragnar Thobaben\textsuperscript{$\diamond$}
  \qquad
  Joakim Jald\'{e}n\textsuperscript{$\diamond$}
  \qquad
  Richard Nock\textsuperscript{\dag}
}
\date{%
  \small
  KTH Royal Institute of Technology\textsuperscript{$\diamond$}
  \qquad
  Google Research\textsuperscript{\dag}
}

\begin{document}
\maketitle

\begin{abstract}
  We study post-hoc Learning to Defer (L2D) through the lens of ideal distributions: divergence-regularized reweightings of the data distribution under which a model attains low task loss.
  We define deferral via the density-ratio between a model's and an expert's ideals.
  Using the reduction from density-ratio estimation to class-probability estimation, we derive the \emph{DR CPE} losses for post-hoc L2D scorers.
  Deferral decisions are then made by thresholding the scorer.
  For KL-based ideal distributions, our deferral rule recovers Chow's rule under the original distribution and a connection to an expert-tilted Bayes posterior---which incorporates the expert's performance---depending on if the ideal distributions are marginal or joint distributions.
  This leads to two variants of the DR CPE losses, with the joint variant being computationally convenient in practice.
  We establish consistency properties of the joint DR CPE losses when considering the zero-one task loss and strongly proper composite CPE losses.
  Experimentally, our approach is competitive compared to common baselines and is more robust across dataset settings.
  More broadly, our results cast post-hoc L2D as density-ratio estimation between ideal distributions, bridging Chow-style rules, expert comparison, and elucidating connections to related learning settings including anomaly detection.
\end{abstract}

\section{Introduction}%
\label{sec:intro}

Machine learning (ML) models are now embedded in many tasks we encounter every day. As these systems are deployed in increasingly high-stakes domains, including healthcare~\citep{esteva2017dermatologist,kadampur2020skin,zoabi2021machine,ma2024segment,singhal2025toward} and finance~\citep{burrell2016machine,yousefi2019comprehensive,li2023large}, it is crucial that they make predictions only when their outputs are reliable.
In such high-stakes decision-making systems, ideally, a synthesis of both ML prediction and human expertise can be used~\citep{topol2019high}: ML models can defer to experts whenever their own predictions are deemed unreliable, and the burden on human experts can be reduced by leveraging ML models.
Dual to the problem of predicting responsibly, the rise of huge ML models and edge deployment has created strong demand for systems that improve efficiency by routing only difficult inputs to a larger model while handling easier cases with a smaller model~\citep{viola2001rapid,mamou2022tangobert,varshney2022model,khalili2022babybear,dohan2022language,gupta2024language,narasimhan2025faster}.
Learning to defer (L2D) proposes a solution to these problems by modifying the standard supervised learning setting to allow (smaller) models to defer their predictions to an upstream (potentially human) expert~\citep{madras2018predict}.

To allow models to abstain from predictions, learning to reject (L2R) was initially proposed~\citep{chow1957optimum,chow1970,cortes2016learning}, where rejection is learned without access to the downstream expert used at inference.
A useful categorization of these approaches are by how abstention is represented.
In {score-based} approaches, abstention is derived from the predictor's own scores, encompassing classical reject-option rules, convex-surrogate formulations, and confidence-threshold methods~\citep{chow1957optimum,chow1970,herbei2006classification,bartlett2008classification,yuan2010classification,cao2022generalizing,el2010foundations,geifman2017selective,pmlr-v97-geifman19a,ni2019calibration}. In {predictor-rejector} approaches, a separate rejector is learned~\citep{cortes2016learning,mao2024predictor}. Finally, {reduction-based} approaches cast rejection as a related supervised learning problem~\citep{charoenphakdee2021classification,ramaswamy2018consistent}.
L2D extends this picture by incorporating an expert into training~\citep{madras2018predict}. Here too, one can distinguish score-based formulations, where deferral is encoded through shared scores or defer classes~\citep{mozannar2020consistent,verma2022calibrated,mao2024principled}, in contrast to predictor-rejector formulations with a separate deferral function~\citep{okati2021differentiable,mozannar2023should,mao2023}. A second, orthogonal distinction is between end-to-end methods and post-hoc (two-stage) methods.

The post-hoc regime is especially important since retraining modern pretrained models can be prohibitively expensive. We focus on post-hoc L2D~\citep{kag2022efficient,jitkrittum2023does,soen2024} (also called two-stage L2D~\citep{mao2023}), where the base model is fixed, and only the deferral rule is learned. In this setting, simple score-based heuristics can fail when downstream experts are specialized or under distribution shift~\citep{jitkrittum2023does}, while many existing methods still amount to estimating pointwise expert-comparison targets, which aim to predict when the expert will outperform the model~\citep{kag2022efficient,jitkrittum2023does}.
A desirable property in this regime is for deferral models to be adaptable to different rates of deferral without full retraining.

Recently, a distributional perspective on post-hoc L2R was proposed~\citep{soen2024,soen2026}.
Its central object is an \emph{ideal distribution}: a divergence-regularized reweighting of the data distribution under which a fixed predictor attains low loss.
Rejection is then defined by thresholding a density-ratio, using either a marginal construction over $\mathcal{X}$~\citep{soen2024} or a joint construction over $\mathcal{X}\times\mathcal{Y}$ that is subsequently marginalized to $\mathcal{X}$~\citep{soen2026}.
Extending this framework to post-hoc L2D is nontrivial: the resulting rules depend on the Bayes-optimal class-posterior and a direct adaption would additionally require pointwise access of the expert at test time.

In this paper, we make the following contributions (\cref{fig:teaser} summarizes our overall approach):
\begin{itemize}[leftmargin=10pt, itemindent=0pt, parsep=1pt]
  \item \emph{L2D via ideal distributions}: We introduce a distributional formulation of post-hoc L2D based on ideal distributions. Unlike L2R, we consider a pair of ideal distributions, one for the model $h$ and one for the expert $h^{(\expert)}$.
    Deferral is then defined by thresholding the density-ratio between these ideal distributions (see \cref{def:ideal-dist,def:dr-defer}).
  \item \emph{Marginal versus joint ideals}: Our framework accommodates both marginal and joint ideal distributions. For KL-ideal distributions, we show that the marginal rule recovers Chow's rule \wrt the data distribution, whereas the joint rule corresponds to a Chow-style decision under an expert-tilted distribution. This tilt is practically meaningful as it augments the expert model's performance into the optimal deferral.
    We further relate the induced learning objectives (see \cref{thm:rel-to-chow,lem:obj-jensen-ordering}).
  \item \emph{The DR CPE losses}: Using the reduction from Density-Ratio Estimation (DRE) to Class-Probability Estimation (CPE), we derive a family of \emph{DR CPE} post-hoc deferral losses.
    Consistency of the joint DR CPE loss objective is examined through a general regret bound with an approximation gap and an exact result for zero-one task loss (see \cref{thm:approx-consistency,thm:exact-consistency}).
    We further connect our method to expert-comparison deferral estimation~\citep{kag2022efficient,jitkrittum2023does} and anomaly detection~\citep{steinwart2005classification} (see \cref{def:ideal-loss}).
  \item \emph{Empirical performance}: Our DR CPE approach performs competitively against other L2D baselines and is also more robust against different dataset corruption combinations than these baselines.
    Furthermore, performance is maintained when utilizing different rates of deferral without retraining.
\end{itemize}

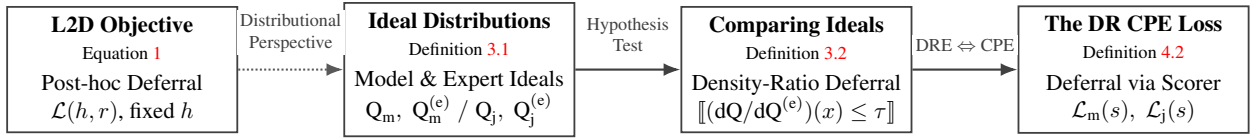
\begin{figure}%
  \centering
  \resizebox{\textwidth}{!}{\begin{tikzpicture}[
    >=Latex,
    font=\small,
    box/.style={
      rectangle,
      draw=black!75,
      fill=white,
      line width=0.8pt,
      minimum width=3.2cm,
      minimum height=1.45cm,
      align=center,
      inner sep=4pt
    },
    flow/.style={
      -{Latex[length=2.6mm,width=2mm]},
      line width=0.95pt,
      draw=black!75
    },
    weakflow/.style={
      -{Latex[length=2.6mm,width=2mm]},
      line width=0.9pt,
      draw=black!65,
      densely dotted
    },
    lab/.style={
      font=\scriptsize,
      fill=white,
      inner sep=1.5pt,
      align=center,
      text=black!75
    }
  ]
  \node[box] (eq1) at (0,0) {%
    \textbf{L2D Objective}\\
    {\scriptsize Equation~(\ref{eq:l2d}})\\[2pt]
    Post-hoc Deferral \\
    $\mathcal{L}(h,r)$, fixed $h$
  };

  \node[box] (ideal) at (5.2,0) {%
    \textbf{Ideal Distributions}\\
    {\scriptsize Definition~\ref{def:ideal-dist}}\\[2pt]
    Model \& Expert Ideals \\
    $\meas{Q}_{\mtype},\;\meas{Q}^{(\expert)}_{\mtype}\; / \;\meas{Q}_{\jtype},\;\meas{Q}^{(\expert)}_{\jtype}$
  };

  \node[box] (dr) at (10.4,0) {%
    \textbf{Comparing Ideals}\\
    {\scriptsize Definition~\ref{def:dr-defer}}\\[2pt]
    Density-Ratio Deferral \\
    $\llbracket (\dd \meas{Q} / \dd \meas{Q}^{(\expert)})(x) \leq \tau \rrbracket$
  };

  \node[box] (loss) at (15.6,0) {%
    \textbf{The DR CPE Loss}\\
    {\scriptsize Definition~\ref{def:ideal-loss}}\\[2pt]
    Deferral via Scorer \\
    $\mathcal{L}_{\mtype}(s),\;\mathcal{L}_{\jtype}(s)$
  };

  \draw[weakflow] (eq1.east) -- (ideal.west)
  node[midway, above=6pt, lab] {Distributional\\Perspective};
  \draw[flow] (ideal.east) -- (dr.west)
  node[midway, above=6pt, lab] {Hypothesis\\Test};
  \draw[flow] (dr.east) -- (loss.west)
  node[midway, above=6pt, lab] {DRE $\Leftrightarrow$ CPE};
\end{tikzpicture}}
  \caption{Post-hoc L2D to learnable density-ratio (DR) deferral. Starting from the post-hoc L2D objective in \cref{eq:l2d}, we reformulate deferral through model and expert ideal distributions (\cref{def:ideal-dist}), compare them via a DR (\cref{def:dr-defer}), and then derive learnable scorer objectives using the reduction from DR Estimation (DRE) to Class-Probability Estimation (CPE) (\cref{def:ideal-loss}).}%
  \label{fig:teaser}
\end{figure}

\section{Settings}%
\label{sec:settings}

We consider a space of observations $\mathcal{X}$ and labels $\mathcal{Y} \defeq [L] = \{ 1, 2, \ldots, L \}$, with a data distribution $\meas{P} \in \triangle(\mathcal{X} \times \mathcal{Y})$.
The learning task we consider is to learn classifiers $f \colon \mathcal{X} \to \mathcal{Y}$ \wrt the data distribution $\meas{P}$.
Instead of directly learning a classifier $f$, we will consider scorers $h \colon \mathcal{X} \to \mathbb{R}^L$ which are related via $f(x) = \argmin_{y \in \mathcal{Y}} h_y(x)$.
To measure performance, we consider a task loss function $\ell \colon \mathcal{X} \times \mathcal{Y} \times \mathbb{R}^L$ to define the expected risk $\expect_\meas{P}[\ell(\X, \Y, h(\X))]$.
We shorthand a decomposition of the data distribution by $\meas{P}(x, y) = \meas{P}_{\rm x}(x) \bpos_y(x)$, where $\meas{P}_{\rm x} \in \triangle(\mathcal{X})$ is the marginal over observations and $\bpos(x) \in \triangle(\mathcal{Y})$ corresponds to the Bayes class-posterior function, \ie, $\bpos_y(x) = \meas{P}[\Y = y \mid \X = x]$.

\paragraph{Learning to Defer}%
To move from the standard multi-class classification paradigm to learning to defer (L2D), alongside the typical optimization of classification model $h \colon \mathcal{X} \to \mathbb{R}^L$, we introduce an \emph{expert} model $h^{(\expert)} \colon \mathcal{X} \to \mathbb{R}^L$. Using the expert, alongside the classifier $h$, we can also jointly learn a rejector $r \colon \mathcal{X} \to \{ 0, 1 \}$ via an objective~\citep{mozannar2020consistent}:
\begin{equation}%
  \label{eq:l2d}
  \scalemath{0.95}{%
    \mathcal{L}(h, r; c) \defeq \mathop{\expect}_{\X, \Y \sim \meas{P}} \left[ (1 - r(\X)) \cdot \ell(\X, \Y, h(\X)) + r(\X) \cdot \ell(\X, \Y, h^{(\expert)}(\X)) \right] + c \cdot \mathop{\prob}_{\X \sim \meas{P}_{\rm x}}[r(\X) = 1],
  }
\end{equation}
In \cref{eq:l2d}, we explicitly separate the expert's performance $\ell(\X, \Y, h^{(\expert)}(\X))$ and the additional cost of calling the expert $c \geq 0$.
If one removes the loss term corresponding to the expert $\ell(\X, \Y, h^{(\expert)}(\X))$, the objective simplifies to the learning to reject (L2R) objective where the performance of how deferred observations are dealt with are not explicitly modeled~\citep{chow1970}.

The L2D objective balances two tasks. First, it wants to decrease the loss of $h$ by smartly deferring observations $x$ to the expert $h^{(\expert)}$ for inputs it will predict poorly for. However, we do not want the rejector $r$ to defer too much, and thus \cref{eq:l2d} balances the rate of deferral in its second term.
Once one has a pair of classifier and rejector $(h, r)$, one can derive a classifier which can defer to an expert $h^{(\expert)}$ through a combination
$ \bar{h}(x) = \llbracket r(x) = 0 \rrbracket \cdot h(x) + \llbracket r(x) = 1 \rrbracket \cdot h^{(\expert)}(x)
$.

While \cref{eq:l2d} is posed as a joint optimization over $(h, r)$, we can also imbue a fixed classifier $h$ with a deferral mechanism, yielding the Bayes-optimal rejector in the post-hoc L2D setting.

\begin{lemma}[{Chow's Rule~\citep{chow1970}}]%
  \label{lem:chow}
  For fixed classifier $h$, the Bayes-optimal deferral rule $r^\star$ of \cref{eq:l2d} is given by
  \begin{equation}
    r^\star(x; c, \meas{P}) = \left \llbracket \expect_{\Y \sim \bpos(x)} \left[ \ell(x, \Y, h(x)) - \ell(x, \Y, h^{(\expert)}(x)) \right] \geq c \right \rrbracket,
  \end{equation}
  where $\llbracket \textup{\texttt{pred}} \rrbracket$ are Iverson brackets which evaluate to $1$ when the predicate $\textup{\texttt{pred}}$ is true and is equal $0$ when the predicate is false~\citep{knuth1992}.
\end{lemma}

\cref{lem:chow} presents an intuitive mechanism for deciding which examples to defer to an expert: if the loss of $h$ is expected to be much higher (\wrt threshold value $c$) than the expert $h^{(\expert)}$, then one should defer.
In the sequel, we will solely consider the setting of post-hoc L2D.

\section{Divergence-Based Deferral Rules}

\cref{eq:l2d} presents a loss-based objective for learning rejectors.
Alternatively, instead of directly learning mechanisms for deferral as a binary function $r$, one can first aim to learn \emph{ideal distributions} and then define deferral decision corresponding to their density-ratios~\citep{soen2024,soen2026}.

\paragraph{Ideal Distributions}
An ideal distribution is intuitively a distribution in which a classifier-loss pair $(h, \ell)$ performs best at within a specified (regularized) distance away from the true probability measure $\meas{P}$. Distance is defined by a statistical divergence. In the following definition, we define two types of ideal distribution which differ in whether we are optimizing within a marginal or joint space.

\begin{definition}%
  \label{def:ideal-dist}
  Given a fixed classifier-loss pair $(h, \ell)$, divergence $\div$, and $\gamma > 0$:
  \begin{align}
    \intertext{\;\textbullet\quad
      An $(\div, \gamma)$-{marginal ideal distribution} $\meas{Q}_\mtype \in \triangle(\mathcal{X})$ of $(h, \ell)$ is given by
    }
    &\meas{Q}_\mtype
    \in
    \mathop{\arg\min}_{\meas{Q} \in \triangle(\mathcal{X})}
    \;
    \mathop{\expect}_{\X \sim \meas{Q}} \mathop{\expect}_{\Y \sim \pos^\star(\X)} \left[
      \ell(\X, \Y, h(\X))
    \right]
    + \gamma \cdot \div(\meas{Q} \Mid \meas{P}_{\rm x}).
    \label{eq:marginal-ideal-obj} \\
    \intertext{\;\textbullet\quad
      An $(\div, \gamma)$-{joint ideal distribution} $\meas{Q}_\jtype \in \triangle(\mathcal{X})$ of $(h, \ell)$ is given by
    }
    &\meas{Q}_\jtype(x; \gamma) = \int_{\mathcal{Y}} \bar{\meas{Q}}_\jtype(x, y) \; \dd y,
    \qquad
    \bar{\meas{Q}}_\jtype
    \in
    \mathop{\arg\min}_{\bar{\meas{Q}} \in \triangle(\mathcal{X} \times \mathcal{Y})}
    \;
    \mathop{\expect}_{(\X, \Y) \sim \bar{\meas{Q}}} \left[
      \ell(\X, \Y, h(\X))
    \right]
    + \gamma \cdot \div(\bar{\meas{Q}} \Mid \meas{P}).
    \label{eq:joint-ideal-obj}
  \end{align}
\end{definition}

Fitting our intuition, the ideal distributions defined in \cref{def:ideal-dist} are found by explicitly minimizing the loss $\ell$ of the classifier $h$. 
The regularization weights $\gamma$ control how far the ideal distributions are from the ground truth distribution $\meas{P}$.
We note that the definition abuses notation slightly, where the divergence $\div$ corresponds to both a divergence on the marginal space and also on the joint space.

\concept{%
\textbf{Generalized Variational Inference (GVI) and Distributionally Robust Optimization (DRO): }  
It has been previously noted in \citet{soen2024} that \cref{eq:marginal-ideal-obj,eq:joint-ideal-obj} resemble the optimization problems associated with GVI~\citep{zellner1988optimal,knoblauch2022optimization} and DRO~\citep{scarf1957min,duchi2021statistics,zhang2021coping}. Indeed, $\meas{Q}_{\itype}$ corresponds to solutions to a GVI problem and is analogous to an adversarial distribution in DRO.
}

As the ideal distribution in \cref{def:ideal-dist} corresponds to the distributions in which a classifier best performs at, we propose a deferral rule for L2D by comparing the ideal distributions of our classifier $h$ versus the ideal distribution of our expert classifier $h^{(\expert)}$.
More precisely, we propose considering the density-ratio between ideal distributions.

\begin{definition}%
  \label{def:dr-defer}
  Let $\itype \in \{ \mtype, \jtype \}$.
  Given $(\div, \gamma)$-ideal distributions $\meas{Q}_\itype$ and $\meas{Q}^{(\expert)}_\itype$ of $(h, \ell)$ and $(h^{(\expert)}, \ell)$, respectively, the $\div$-density-ratio deferral rule is given by
  \begin{equation}
    r_\itype(x; \tau, \gamma, \gamma^{(\expert)}) \defeq \llbracket \dratio_\itype(x; \gamma, \gamma^{(\expert)}) \leq \tau\rrbracket,
    \quad
    \textrm{where} \; \dratio_\itype(x; \gamma, \gamma^{(\expert)}) \defeq \frac{\dd \meas{Q}_\itype(x; \gamma)}{\dd \meas{Q}_\itype^{(\expert)}(x; \gamma^{(\expert)})}.
  \end{equation}
\end{definition}
An advantage of making deferral decision through density-ratios is that the functional form of $r_{\itype}$ becomes convenient once we make our choice of divergence $\div$ (\ie, by considering $\varphi$-divergences).

\concept{%
\textbf{Hypothesis Testing: }
The deferral decision of \cref{def:dr-defer} has a direct connection to hypothesis testing.
Thresholding density-ratios corresponds exactly to optimal hypothesis testing in the Neyman-Pearson sense~\citep{neyman1933,lehmann2005}, \ie, it is a likelihood ratio test. More concretely, $r_{\itype}(x; \tau, \gamma, \gamma^{(\expert)})$ corresponds to the hypothesis test
\begin{equation}%
  \label{eq:hypothesis-test}
  H_0\colon x \sim \meas{Q}^{(\expert)}_\itype(\cdot; \gamma^{(\expert)})
  \quad
  \textrm{and}
  \quad
  H_1\colon x \sim \meas{Q}_\itype(\cdot; \gamma).
\end{equation}
The rule fails to reject the null hypothesis $H_0$, and hence defers to the expert, whenever the likelihood ratio $\dd \meas{Q}_\itype / \dd \meas{Q}^{(\expert)}_\itype$ is sufficiently small, \ie, whenever $\dratio_\itype(x; \gamma, \gamma^{(\expert)}) \leq \tau$. Intuitively, this means that deferral occurs precisely when the observation $x$ is better explained by the expert's ideal distribution than by the classifier's ideal distribution. Equivalently, the test asks which of the two classifiers performs better at $x$: if $x$ is more likely under the expert's ideal distribution, then the expert's predictions are better supported there, whereas if $x$ is more likely under the classifier's ideal distribution, then the evidence favors keeping the prediction with $h$ rather than deferring to $h^{(\expert)}$.
}

Defining deferral decision via the density-ratios of ideal distributions was originally utilized for L2R~\citep{soen2024}.
In this setting, instead of considering a pair of classifiers $h$, $h^{(\expert)}$, L2R only considers a single classifier $h$.
As a result, there is only a single ideal distribution, and the distribution we compare it against is just the data measure $\meas{P}$.
Although not stated in prior work, the hypothesis testing interpretation corresponding to \cref{eq:hypothesis-test} (replacing $\meas{Q}^{(\expert)}$ with $\meas{P}$) also holds for the L2R setting.

\paragraph{KL-Divergence}
In prior work, it was shown that a canonical choice of divergence for constructing ideal distributions is the KL-divergence~\citep{soen2024,soen2026}.
In this work, we will focus on this class of ideal distributions generated by this canonical choice.
We first recall that the KL-divergence between probability measures $\meas{P} \ll \meas{Q}$ is given by~\citep{amari2000methods}: $\kl(\meas{P} \Mid \meas{Q}) \defeq {\expect}_{\Z \sim \meas{P}} \left[ \log\left( {\dd \meas{P}}/{\dd \meas{Q}}(\Z) \right) \right]$.

Utilizing the KL-divergence, both of the corresponding marginal and joint ideal distributions can be shown to be (exponential) multiplicative re-weightings of the true data distribution $\meas{P}$.

\begin{corollary}\label{cor:kl_weights}
  Given a fixed classifier-loss pair $(h, \ell)$, the $(\kl, \gamma)$-ideal distributions are multiplicative re-weightings $\meas{Q}_{\itype}(x; \gamma) = w_\itype(x; \gamma) \cdot \meas{P}_{\rm x}(x)$, with weights given by
  \begin{equation}
    w_\itype(x; \gamma)
    \defeq
    \frac{1}{Z_{\itype, \gamma}}
    \tilde{w}_\itype(x; \gamma),
    \quad
    \tilde{w}_\itype(x; \gamma)
    \defeq
    \begin{cases}
      \exp\left( -\frac{\expect_{\Y \sim \pos^\star(x)}\left[ \ell(x, \Y, h(x)) \right] }{\gamma} \right) & \textrm{if } \; \itype = \mtype \\
      \expect_{\Y \sim \bpos(x)}\left[ \exp\left( -\frac{\ell(x, \Y, h(x))}{\gamma} \right) \right] & \textrm{if } \; \itype = \jtype,
    \end{cases}
  \end{equation}
  where $Z_{\itype, \gamma}$ are normalizing constants of the distributions, \ie, $Z_{\itype, \gamma} = \expect_{\X \sim \meas{P}_{\rm x}}\left[ \tilde{w}_\itype(\X; \gamma) \right]$.
\end{corollary}

\cref{cor:kl_weights} follows directly from a more general result that holds for $\varphi$-divergences (\aka, Csisz\'ar $f$-divergences), as shown in \cref{app:varphi-div-ext}.
In this more general case, the ideal distribution generated by $\varphi$-divergences are also multiplicative re-weightings of $\meas{P}$. However, only the KL-divergence's \emph{normalization} are determined by a multiplicative constant $1 / Z_{\itype, \gamma}$. This latter property will be key in our construction of loss functions in \cref{sec:dr-cpe}.

From the functional form presented in \cref{cor:kl_weights}, it becomes apparent why taking the threshold of density-ratios suits the functional form of KL-divergence (and $\varphi$-divergence) ideal distributions: dependence on the marginal data measure $\meas{P}_{\rm x}(x)$ will cancel out in the density-ratio. As a result, for optimal ideal distributions, the rejection decision follows from a ratio of weights defined in \cref{cor:kl_weights}. 
Furthermore, in this case, the role of the regularizer $\gamma$ becomes clear: $\gamma$ corresponds to the \emph{temperature} of the multiplicative reweighting.

The choice of utilizing the KL-divergence for ideal distributions is known to be canonical as it can be shown that in the L2R setting, rejection via the density-ratio mechanism and marginal ideal distributions recovers Chow's rule~\citep[Theorem 4.2]{soen2024}. Specific to the L2R setting, it was also shown that rejectors corresponding to joint ideal distributions are more lenient than their marginal ideal distribution counterparts~\citep[Lemma 1]{soen2026}.

The shift from working with one ideal distribution to a pair of ideal distributions in L2D changes the relationship between marginal and joint ideal distributions that was previously established in L2R. The connection to Chow's rule for the marginal ideal distribution is retained; however, the analogous connection between the joint and marginal ideal distributions no longer holds for deferral mechanisms of the form in \cref{def:dr-defer}. Instead, the joint deferral mechanism can be related to Chow's rule under a data distribution with a \emph{tilted} Bayes class-posterior.

\begin{definition}\label{def:tilt}
  For a distribution $\meas{P}$, we denote its $(\ell, \gamma)$-expert-tilted distribution as
  \begin{equation}\label{eq:pos_tilted}
    \tilde{\meas{P}}_{\gamma}(x, y) \defeq \meas{P}_{\rm x}(x) \cdot {\tbpos}_y(x; \gamma),
    \quad
    {\tbpos}_y(x; \gamma) \defeq t(y \mid x; \gamma) \cdot \bpos_y(x),
  \end{equation}
  where
  \begin{equation}
    t(y \mid x; \gamma)
    \defeq
    \frac{\exp\left( - {\ell(x, y, h^{(\expert)}(x))}/{\gamma} \right)}{\expect_{\Y \sim \bpos(x)}\!\left[\exp\!\left( - {\ell(x, \Y, h^{(\expert)}(x))}/{\gamma} \right) \right]}.
  \end{equation}
\end{definition}

The informed reader may identify the transformation of the Bayes class-posterior as the exponential tilting of a distribution~\citep{escher1932} or as the Boltzmann-Gibbs transformation~\citep{delmoral2004}.

\begin{theorem}%
  \label{thm:rel-to-chow}
  For any cost $c$, there exists a $\tau > 0$ such that the KL-marginal density-ratio deferral rule satisfies $r_\mtype(x; \tau, \gamma, \gamma) = r^\star(x; c, \meas{P})$.
  Furthermore, there exists a $\tau > 0$ such that the KL-joint density-ratio deferral rule satisfies $r_\jtype(x; \tau, \gamma, \gamma) \leq r^\star(x; c, \tilde{\meas{P}}^{(\expert)}_\gamma) $.
\end{theorem}

The second part of \cref{thm:rel-to-chow} shows that the joint density-ratio deferral rule is pointwise more conservative than Chow's rule under the expert-tilted distribution $\tilde{\meas{P}}_\gamma$. In particular, whenever the joint rule defers at an input $x$, Chow's rule under $\tilde{\meas{P}}_\gamma$ also defers at $x$. Equivalently, the rejection region of the joint rule is contained in the rejection region induced by Chow's rule under the tilted measure.

Beyond its theoretical interest, the second part of \cref{thm:rel-to-chow} also provides a useful practical interpretation of deferral. In many applications, the Bayes class-posterior $\bpos$ is not, on its own, the most relevant object for deciding whether to predict or defer. Indeed, $\bpos$ only captures the (aleatoric) label uncertainty under the data distribution, whereas the deferral decision should also reflect the downstream loss and the expert's comparative advantage. The tilted posterior $\tbpos(x; \gamma)$ incorporates this additional information by reweighting labels according to the expert loss $y \mapsto \ell(x, y, h^{(\expert)}(x))$. Consequently, applying a Chow-type rule under $\tilde{\meas{P}}_\gamma$ means deferring not simply when the labels are ambiguous, but when they are ambiguous in a way that favors the expert. This can be preferable in practice, since deferral is only useful in-so-far that the expert handles the prediction task better than the deferring model. From this perspective, Chow's rule under the tilted posterior ranks examples not only by uncertainty, but by uncertainty that is especially relevant to the expert.

\section{Density-Ratio Estimation Losses from Ideal Distributions}%
\label{sec:dr-cpe}

So far we have seen that ideal distributions, together with thresholding their associated density-ratios, can yield analytical and even optimal deferral rules. However, directly applying the closed-form expressions in, \eg, \cref{cor:kl_weights}, raises two fundamental difficulties. First, both the marginal and joint ideal distributions (and hence the resulting density-ratio) depend on the Bayes class-posterior $\bpos$. Using these expressions as plug-in deferral rules would therefore require an approximation of $\bpos$ which can be suboptimal if not carefully calibrated.
Second, the deferral rule also depends on the expert ideal distribution $\meas{Q}^{(\expert)}_\itype$, which in turn requires pointwise evaluation of the expert model $h^{(\expert)}(x)$ at inference time. If the expert has already been queried on $x$, then (for many settings) one may as well use the expert's prediction directly rather than first using it only to decide whether to defer.

These two issues prevent us from directly utilizing the closed-form analytical solutions for L2D. However, instead of trying to use a plug-in estimator to evaluate ideal distribution density-ratios, we propose a method which aims to estimate the density-ratio $\dd \meas{Q}_{\itype} / \dd \meas{Q}^{(\expert)}_\itype$. To do so, we rely on links between \emph{Density-Ratio Estimation} (DRE)~\citep{sugiyama2012density} and \emph{Class-Probability Estimation} (CPE)~\citep{buja2005loss,reid2011information}, where we reduce the problem of estimating the density-ratio to a binary classification problem. Unique to our setting, the binary classification problem will be to distinguish between two reweightings of the data distribution $\meas{P}$ with weights determined by the performance of the model $h$ and expert $h^{(\expert)}$.

\paragraph{DRE to CPE}
Before specifying the CPE optimization problem for estimating the density-ratio between ideal distributions, we revise how general DRE can be expressed as a CPE problem.
Suppose that we have a strictly proper composite loss $\drloss \colon \mathbb{R} \times \{ \pm 1 \} \to \mathbb{R}$ with invertible link function $\link \colon [0, 1] \to \mathbb{R} $~\citep{reid2010}.
Recall that to estimate the density-ratio $\dd \mu / \dd \nu$ between probability measures $\mu$ and $\nu$, we can first find a (binary) scoring function $ s \colon \mathcal{X} \to \mathbb{R} $ that minimize the following CPE objective~\citep{sugiyama2008,menon2016,nock2016scaled}:
\begin{equation}
  \label{eq:dre-cpe-obj}
  \mathop{\arg\min}_{s \colon \mathcal{X} \to \mathbb{R}}
  \quad
  \pi \cdot \expect_{\mu}[(\drloss_{+1} \circ s)(\X)]
  +
  (1 - \pi) \cdot \expect_{\nu}[(\drloss_{-1} \circ s)(\X)],
\end{equation}
where $\pi \in (0, 1)$ and denote the partial losses $ \drloss_{+1}(v) \defeq \drloss(v, +1) $ and $ \drloss_{-1}(v) \defeq \drloss(v, -1) $.

The optimization of \cref{eq:dre-cpe-obj} is equivalent to a binary classification task on a dataset generated by class-conditionals defined by $\mu$ and $\nu$---\ie, $\mathbb{P}(\X = x \mid \R = +1) = \mu(x)$---and base label rate $\pi = \mathbb{P}(\R = +1)$. The label rate $\pi$ is a free parameter which the user can set, where a typical choice is to set $\pi$ based on the ratio of empirical data used to optimize \cref{eq:dre-cpe-obj}.
Key to this interpretation of the optimization problem is that we obtain class-probability estimates via the inverse link function $\hat{\mathbb{P}}(\R = +1 \mid \X = x) = (\link^{-1} \circ s)(x)$.
As a result, the scorer $s$ can then be converted to a density-ratio estimator between $\mu$ and $\nu$ via Bayes' rule~\citep{bickel2009discriminative,menon2016}
\begin{equation}
  \label{eq:dre-cpe-dr}
  \hat{\dratio}(x; s) \defeq \frac{1 - \pi}{\pi} \cdot \frac{(\link^{-1} \circ s)(x)}{1 - (\link^{-1} \circ s)(x)} \approx \frac{\dd \mu}{\dd \nu}.
\end{equation}
If our pair of loss function and link $(\drloss, \link)$ is indeed strictly proper composite, then the minimizer of \cref{eq:dre-cpe-obj} will result in a scorer $s^\star$ corresponding to the Bayes-optimal class-posterior. As a result, the density-ratio learned will be exactly $\dd \mu / \dd \nu$. See \cref{app:dre_via_cpe} for an extended exposition of the DRE to CPE learning problem reduction.

Typically, the DRE to CPE reduction introduces the invertible link function to transform a scorer $s$ into an approximate density-ratio. However, as we are interested in deferral, we only need to calculate the threshold of the approximated density-ratio. It turns out that the thresholding of density-ratios of the form \cref{eq:dre-cpe-dr} can be simplified to a thresholding of the scorer.

\begin{lemma}%
  \label{lem:threshold-scorer}
  Let $s \colon \mathcal{X} \rightarrow \mathbb{R}$ be a scorer learned via the minimization of \cref{eq:dre-cpe-obj} with loss $\drloss$ and strictly increasing, invertible link $\link$. Then for $\hat{\dratio}$ given by \cref{eq:dre-cpe-dr}, we have for $\tau > 0$ that
  \begin{equation}
    \left \llbracket \hat{\dratio}(x) \leq \tau \right \rrbracket = \left \llbracket s(x) \leq \link \left( \frac{\pi \tau}{1 - \pi + \pi \tau} \right) \right \rrbracket.
  \end{equation}
\end{lemma}

Although it appears that the deferral decision is still dependent on the link function $\link$ on the RHS of \cref{lem:threshold-scorer}, in practice one does not typically calculate an exact thresholding value.
Instead, the scorer $s$ can be evaluated over a validation dataset and the threshold value can be specified \wrt a target deferral percentage---a common approach for deferring via with scoring functions or confidence estimates~\citep{pmlr-v206-pugnana23a,pugnana2023model}.
One should note that even the ``original'' formulation of the L2D problem in \cref{eq:l2d} requires a calibration of the cost $c$---we remind that the deferral regularized $c \cdot \mathbb{P}_{\meas{P}_{\rm x}}[r(\X) = 1] $ can be interpreted as a Lagrange multiplier relaxation of a hard deferral rate constraint $ \mathbb{P}_{\meas{P}_{\rm x}}[r(\X) = 1] \leq \texttt{budget}$~\citep{franc2023optimal}.
In the worst case, this may require training multiple deferral mechanisms $r$ from scratch for each $c$ in a cross validation scheme.
As a result, \cref{lem:threshold-scorer} allows us to make deferral decision based on density-ratios by thresholding in the scorer's output space, rather than explicitly working in the space of density-ratio values.

\paragraph{The Ideal DR CPE Loss}
With the general setup of the DRE to CPE reduction, we now define a pair of density-ratio loss functions which correspond to the optimization problem of learning the density-ratio of ideal distributions. In particular, we exploit the fact that the introduced DRE framework of \cref{eq:dre-cpe-obj} allows for density-ratio estimates for any $\pi \in (0, 1)$.

\begin{definition}\label{def:ideal-loss}
  Given a DR loss $\drloss$, the corresponding $(\itype \in \{ \mtype, \jtype \})$-ideal DR CPE loss is given by
  \begin{equation}
    \mathcal{L}_{\itype}(s; \gamma, \gamma^{(\expert)}) \defeq \expect_{\meas{P}_{\rm x}}\left[\tilde{w}_\itype(\X; \gamma) \cdot (\drloss_{+1} \circ s)(\X) + \tilde{w}^{(\expert)}_\itype(\X; \gamma^{(\expert)}) \cdot (\drloss_{-1} \circ s)(\X)\right].
  \end{equation}
\end{definition}

The $\itype$-ideal DR CPE loss in \cref{def:ideal-loss} is simply derived by setting $\mu = \meas{Q}_\itype(\cdot; \gamma)$, $\nu = \meas{Q}^{(\expert)}_\itype(\cdot; \gamma^{(\expert)})$, and $\pi = Z_{\itype, \gamma} / (Z_{\itype, \gamma} + Z^{(\expert)}_{\itype, \gamma})$.
The latter choice of setting the marginal label probabilities $\pi$ conveniently eliminates the need to estimate the normalizers of the ideal distributions.

Although \cref{def:ideal-loss} defines a pair of loss functions---one for the marginal ideal distributions, one for the joint ideal distributions---practically only the joint variant can be utilized. This is due to the dependence of $\expect_{\bpos(x)}$ inside the exponential function of the weights $\tilde{w}_{\mtype}(x; \gamma) = \exp\left( -{\expect_{\pos^\star(x)}\left[ \ell(x, \Y, h(x)) \right] }/{\gamma} \right)$. Although one could attempt to empirically estimate the weights (and the corresponding loss $\mathcal{L}_{\mtype}$), the resulting estimator would be biased. Furthermore, the estimated loss would typically have a high variance as for most high dimensional input spaces $\mathcal{X}$, only a single class label $y$ is available for each observation $x$.

Despite the marginal variant of the DR CPE loss being difficult to estimate, the joint DR CPE loss $\mathcal{L}_{\jtype}(s; \gamma, \gamma^{(\expert)})$ has a convenient form:
\begin{equation}%
  \label{eq:joint-ideal-loss-simplify}
  \scalemath{0.95}{%
    \mathop{\expect}_{(\X, \Y) \sim \meas{P}}\left[\exp\left( -\frac{\ell(\X, \Y, h(\X))}{\gamma} \right) \cdot (\drloss_{+1} \circ s)(\X) + \exp\left( -\frac{\ell(\X, \Y, h^{(\expert)}(\X))}{\gamma^{(\expert)}} \right) \cdot (\drloss_{-1} \circ s)(\X)\right].
  }
\end{equation}
Indeed, as the expectation $\expect_{\bpos(x)}$ is on the outside of the exponential function, the objective $\mathcal{L}_{\jtype}$ can be calculated as an expectation of a random variable \wrt the joint data measure $\meas{P}$. From an estimation perspective, one can empirically estimate the expectation in \cref{eq:joint-ideal-loss-simplify} to yield an unbiased estimator.

A natural question is how are these two objectives related.
As the marginal ideal distributions (and thus the corresponding loss function) relate to the typical notion of optimal deferral (via \cref{thm:rel-to-chow}), the objective we would like to minimize is the marginal DR CPE loss $\mathcal{L}_{\mtype}$.
Although $\mathcal{L}_{\mtype}$ is not easy to estimate, through Jensen's inequality, one can directly show that the joint DR CPE loss $\mathcal{L}_{\jtype}$ directly upper bounds the marginal DR CPE loss $\mathcal{L}_{\mtype}$.

\begin{lemma}%
  \label{lem:obj-jensen-ordering}
  Suppose $\ell, \drloss_{+1}, \drloss_{-1} \geq 0$. Then $\mathcal{L}_{\mtype}(s; \gamma, \gamma^{(\expert)}) \leq \mathcal{L}_{\jtype}(s; \gamma, \gamma^{(\expert)})$.
\end{lemma}
From \cref{lem:obj-jensen-ordering}, one can expect that minimizing the joint DR CPE loss will reduce the corresponding marginal DR CPE loss via an upper bound.
However,
additional assumptions are needed
to establish bounds which explore the consistency of the joint DR CPE loss \wrt minimization of the \emph{original} L2D risk (\cref{eq:l2d}).
In particular, we consider \emph{strongly proper} loss functions, which have previously been used to derive consistency bounds~\citep{agarwal2014surrogate}.
The full details of the assumption, formal statements, and all proofs associated with consistency are deferred to \cref{app:consistent,app:exact-consistent}.

Given the appropriate assumptions, a general approximate consistency-style bound can be established for the shared temperature case $\gamma = \gamma^{(\expert)} > 0$.
\begin{theorem}[Informal]\label{thm:approx-consistency}
  Suppose $0 \leq \ell \leq B$, $(\drloss, \link)$ is $\beta$-strongly proper, and $\drloss_{+1}, \drloss_{-1} \geq 0$. Define $r_{c, \gamma}(x) = \llbracket (\link^{-1} \circ s)(x) \leq \tau_{c, \gamma} \rrbracket$.
  Then there exists some constant $K_{\gamma, B, c} > 0$ such that every scorer $s$ and cost $c \geq 0$ yields,
  \begin{equation}
    \mathcal{L}(h, r_{c, \gamma}) - \mathcal{L}(h, r^\star)
    \leq
    K_{\gamma, B, c} \sqrt{\frac{\exp(B / \gamma)}{\beta} \left[
        \mathcal{E}_{\jtype}(s; \gamma, \gamma) + \Delta_{\mtype \jtype}(\gamma, \gamma)
      \right]
    },
  \end{equation}
  where $\tau_{c, \gamma} \defeq \sigma\left(-\frac{c}{\gamma}\right)$ with $\sigma(t) \defeq (1 + \exp(-t))^{-1}$ denoting the sigmoid function and
  \begin{align}
    \mathcal{E}_{\itype}(s; \gamma, \gamma^{(\expert)})
    &\defeq \mathcal{L}_{\itype}(s; \gamma, \gamma^{(\expert)}) - \inf_{s'} \mathcal{L}_{\itype}(s'; \gamma, \gamma^{(\expert)}) \\
    \Delta_{\mtype \jtype}(\gamma, \gamma^{(\expert)})
    &\defeq \inf_{s} \mathcal{L}_{\jtype}(s; \gamma, \gamma^{(\expert)}) - \inf_{s'} \mathcal{L}_{\mtype}(s'; \gamma, \gamma^{(\expert)})
  \end{align}
  denoting the excess risk and marginal-joint ideal optimality gap, respectively.
\end{theorem}
\cref{thm:approx-consistency} bounds the original L2D regret using the joint excess risk and the marginal-joint ideal optimality gap.
The marginal-joint ideal optimality gap term $\Delta_{\mtype \jtype}(\gamma, \gamma)$ can be bounded by the Jensen gap in \cref{lem:obj-jensen-ordering}. This can be further shown to be upper bounded by the (weighted conditional) task loss variance of the base and expert models (\cref{eq:approx-consistency-variance-gap}).
This can be shown to vanish when the task losses are (almost surely) constant conditioned on an input $\X = x$, \eg, when the labels have no aleatoric uncertainty~\citep{hullermeier2021aleatoric}.
In this case, \cref{thm:approx-consistency} yields consistency. Similar observation can be made by observing that another sufficient condition for $\Delta_{\mtype \jtype}(\gamma, \gamma) = 0$ is that the Chow's rule and the ``tilted'' Chow's rule of \cref{thm:rel-to-chow} to be equal.

Generally $\Delta_{\mtype \jtype}(\gamma, \gamma)$ is non-zero and thus \cref{thm:approx-consistency} does not generally prove consistency.
This naturally raises the question of whether there exists a setting which would yield a consistent joint objective. The following reports the affirmative by deviating from the common temperature setting.
\begin{theorem}[Informal]\label{thm:exact-consistency}
  Suppose that $\ell$ is the zero-one loss function and the composite loss $(\drloss, \link)$ is $\beta$-strongly proper with a strictly increasing link $\link$.
  For $c \geq 0$ and $\gamma^{(\expert)} > 0$, there exists a $\gamma > 0$, $\tilde{K}_{c, \gamma, \gamma^{(\expert)}} > 0$, and $\tilde{\tau}_{c, \gamma, \gamma^{(\expert)}}$ such that the rejector $\bar{r}_{c, \gamma, \gamma^{(\expert)}}(x) = \llbracket (\link^{-1} \circ s)(x) \leq \tilde{\tau}_{c, \gamma, \gamma^{(\expert)}} \rrbracket$ satisfies
  \begin{equation}
    \mathcal{L}^{\textup{zo}}(h, \bar{r}_{c, \gamma, \gamma^{(\expert)}}) - \mathcal{L}^{\textup{zo}}(h, r^\star)
    \leq
    \frac{\tilde{K}_{c, \gamma, \gamma^{(\expert)}}}{\sqrt{\beta}}
    \sqrt{\mathcal{E}_\jtype^{\textup{zo}}(s; \gamma, \gamma^{(\expert)})},
  \end{equation}
  where the superscript ``${\textup{zo}}$'' is to indicate we are considering the zero-one loss task loss $\ell = \ell^{\textup{zo}}$.
\end{theorem}
For non-zero deferral cost $c \neq 0$, the temperatures in \cref{thm:exact-consistency} are required to be different $\gamma \neq \gamma^{(\expert)}$. The bound is an interesting counterpart to \cref{thm:approx-consistency}, where by specifying the ``task loss'' $\ell$ to be the zero-one loss function, the bound does not depend on the marginal-joint ideal optimality gap. Importantly, this bound yields an exact scenario where the joint DR CPE losses can provide exact consistency.
It should be noted that consistency results typically only hold for a fixed cost~\citet{cao2022generalizing}, \ie, consistency at one cost $c'$ does not imply simultaneously consistency for all possible costs $c \geq 0$. In our \cref{thm:approx-consistency,thm:exact-consistency}, the consistency of the scorer $s$ is also tied to a specific threshold $\tau$ (which is depends on $c$).

\concept{%
  \textbf{Expert-Comparison Estimation: }
  Previous work in the literature suggests learning deferral rules by regressing pointwise targets that compare model and expert performance~\citep{kag2022efficient,jitkrittum2023does}.
  In particular, one seeks to estimate a pointwise expert-comparison target that approximates how Chow's rule makes deferral decisions.
  For example, one considers the zero-one loss function to design a least-squares regression task on the random variable
  \begin{equation}%
    \label{eq:reg_delta_target}
    \Y^{({\rm d})} \defeq \llbracket f(\X) = \Y \rrbracket - \llbracket f^{(\expert)}(\X) = \Y \rrbracket.
  \end{equation}
  These target labels $\Y^{({\rm d})}$ simply aim to identify whenever the expert $h^{(\expert)}$ would be correct but the model $h$ would not.
  Interestingly, a direct connection can be made to the joint ideal DR CPE loss.
  \begin{proposition}%
    \label{prop:limit_to_diff}
    Suppose $\ell$ is the zero-one loss function and $\drloss$ is the squared loss function. Let $\Sigma(x) \defeq \prob_{\meas{P}}[f(\X) = \Y \mid \X = x] + \prob_{\meas{P}}[f^{(\expert)}(\X) = \Y \mid \X = x]$.
    Then, for $\meas{P}_{\rm x}$-almost every $x$ with $\Sigma(x) > 0$, the unique pointwise minimizer $s^\star_0(x)$ of the limiting risk satisfies, for any $\tau \in \mathbb{R}$,
    \begin{equation}
      \llbracket s^\star_0(x) \le \tau \rrbracket
      =
      \left \llbracket \expect[\Y^{({\rm d})} \mid \X = x] \le \tau \, \Sigma(x) \right \rrbracket.
    \end{equation}
  \end{proposition}
  An interesting interpretation of \cref{prop:limit_to_diff} is that, in the low temperature limit $\gamma \to 0$, the optimal deferral mechanism induced by the joint ideal loss with squared loss is equivalent to the optimal deferral mechanism learned by regressing on $\Y^{({\rm d})}$. The only difference is that the threshold is no longer constant, but instead rescaled by the input-dependent factor $\Sigma(x)$.
  Thus, the rule preserves the same expert-comparison ``signal'' used in prior work, with an input-dependent threshold determined by the total correctness of the model and expert.

}

\concept{%
  \textbf{Anomaly Detection (AD): }
  One definition of anomalous observations are examples which are not concentrated~\citep{ripley2007pattern,chandola2009anomaly}, \ie,
  examples that are sampled from areas of low density \wrt $\meas{P}$.
  To identify such examples, one approach in AD seeks to solve an objective similar to DRE~\cref{eq:dre-cpe-obj}. More concretely, they seek to learn a scorer $s$ via optimizing~\citep{steinwart2005classification}:
  \begin{equation}%
    \label{eq:anomaly}
    \mathcal{L}_{\rm anomaly}(s) \defeq \expect_\meas{P}[(\drloss_{+1} \circ s)(\X)] + \int (\drloss_{-1} \circ s)(x) \, \dd \mu(x),
  \end{equation}
  where $\mu$ corresponds to the reference measure of $\meas{P}$. An observation $x$ can then be classified as anomalous via the scorer $s(x)$, \ie, via a threshold $\llbracket s(x) \leq \alpha_q \rrbracket$ with $\alpha_q$ chosen to correspond to the $q$th quantile of the scorer $s$.
  A calibrated variant of the AD task can also be considered, where instead of just wanting to identify anomalous observations, we want the learned scorer $s$ (via a link function) to correspond to the cumulative distribution function of $\meas{P}$~\citep{menon2018loss}.

  This form of AD mirrors that of learning deferral rules from ideal distributions---the only change is
  what is considered a typical versus anomalous distribution.
  Here, the typical distribution is considered as the data distribution itself $\meas{P}$ and anything which deviates from it is considered anomalous, \ie, the reference measure $\mu$. Instead, in L2D both the typical and anomalous distributions are reweightings of the data distribution.
  From an optimization perspective, our deferral task is easier. In \cref{eq:anomaly}, the integral term is difficult to estimate (for general $\mu$)~\citep{menon2018loss}.

}

\section{Experiments}%
\label{sec:experiments}

\definecolor{baseline}{rgb}{0.8,0.4745098039215686,0.6549019607843137}
\definecolor{twostage}{rgb}{0.21568627450980393,0.8,0.6588235294117647}
\definecolor{estimate}{rgb}{0.7490196078431373,0.34509803921568627,0.11372549019607843}
\definecolor{drcpe}{rgb}{0.19607843137254902,0.20392156862745098,0.6235294117647059}

In the following section, we evaluate our ideal DR CPE method learning deferral decisions.\footnote{%
  Code is available at: \url{https://github.com/alexandersoen/dr-losses-for-learning-to-defer}.
} To do so, we consider a variant of the experimental model cascade setup presented in \citet{jitkrittum2023does} on corrupted CIFAR-100~\citep{cifar100} and MedMNIST~\citep{medmnistv1,medmnistv2} datasets.
We consider label noise, long-tail, and specialization variants of these datasets: 
label noise randomly flips the labels of the first 10 classes; 
long-tail subsamples 500 examples from the first 25 classes and 50 examples from the rest; 
and specialization subsamples examples at a rate of 20\% for non-specialist class labels.
We consider a three-way partition of the dataset: a training set for $h$, $h^{(\expert)}$; a clean held-out training set for $r$; and a clean test set for evaluation. Label noise and long-tail subsampling affect both classifiers. In the specialist setting, $h$ uses the full clean training partition and only $h^{(\expert)}$ subsamples non-specialist examples.
When a validation split is available, we use it to train the deferral mechanism; otherwise, we use a partition of the training set.
The model architectures are ResNet models of various sizes $\{ 8, 14, 32, 56 \}$ ($h^{(\expert)}$ typically larger) and (when applicable) deferral approaches utilize a lightweight Multi-Layer Perceptron with features constructed from probability outputs from $h$~\citep{jitkrittum2023does}.
Further details and additional experimental settings are presented in~\cref{app:exp_settings,app:extended_experiments}.
We further present ImageNet~\citep{ILSVRC15} experiments and unequal temperatures experiments in~\cref{app:imagenet} (for simplicity, we only consider the equal temperature settings $\gamma = \gamma^{(\expert)}$ here).

For comparison, we consider 3 types of baselines. First we consider a {\color{baseline} confidence-based} approach (\texttt{Conf}), which simply defers based on thresholding $h$'s perceived confidence (no deferral training)~\citep{ramaswamy2018consistent}. Second, we consider a {\color{twostage} two-stage} surrogate loss method (\texttt{TwoStage-Exp})~\citep{mao2023}. Finally, we consider two {\color{estimate} expert-comparison estimations}~\citep{kag2022efficient,jitkrittum2023does}: one which regresses on $\Y^{(d)}$ as per \cref{eq:reg_delta_target} (\texttt{Estimate-Diff01}); and one which uses the relative confidence between $h$ and $h^{(\expert)}$ by estimating $h^{(\expert)}$'s maximum predicted probability (\texttt{Estimate-MaxProb}).
For {\color{drcpe} DR CPE}, we set $\drloss$ to be the squared loss and the temperatures to be $\gamma = \gamma^{(\expert)} = 0.5$; we consider $\ell$ to be either the generalized cross-entropy (GCE)~\citep{zhang2018generalized} or a probabilistic version of the zero-one loss which consists of the $L_2$ distance between the one-hot label and the predicted probabilities (Prob01).

\cref{fig:tradeoff-summary} summarizes our experiments, where we evaluate the accuracy-deferral trade-offs (upper bounds are better).
At $0\%$ deferral rate (left) the accuracy-deferral curve reports the base model $h$ accuracy; and at $100\%$ deferral rate (right) its reports the expert model $h^{(\expert)}$ accuracy.
Accuracy-deferral curves evaluate the approaches beyond their consistency guarantees by evaluating performance across different (deferral cost) thresholds.
Each curve sweeps the threshold of a fixed scorer; changing $\gamma$ requires retraining.
Although sweeping the threshold in this manner does not adhere to the fixed deferral cost (and threshold) setting that our consistency results require (\cref{thm:approx-consistency,thm:exact-consistency}), we find that {\color{drcpe} DR CPE} performs robustly across different deferral rates regardless.
The approaches perform similarly under clean data. Over corrupted settings, {\color{drcpe} DR CPE} performs best to second best. When compared to the {\color{baseline} confidence-based} and the {\color{twostage} two-stage} approaches, {\color{drcpe} DR CPE} is consistently better. Both of these baselines perform extremely badly on the DermaMNIST specialist setting. %
The exception is that the {\color{twostage} two-stage} approach performs well clean PathMNIST and corrupted PathMNIST for smaller deferral rates (less than $\sim30\%$); however, for higher deferral rates it performs significantly worse than {\color{drcpe} DR CPE} (GCE).
The most competitive baselines are the {\color{estimate}expert-comparison estimators}. \texttt{Estimate-MaxProb} performs significantly worse than {\color{drcpe} DR CPE} in corrupted datasets, except the PathMNIST specialist setting where it performs similarly to {\color{drcpe} DR CPE} with GCE.
\texttt{Estimate-Diff01} has strong performance, beating {\color{drcpe} DR CPE} for the CIFAR-100 long-tail and DermaMNIST specialist settings, but suffers from a larger variance in the latter.
However, \texttt{Estimate-Diff01} critically fails in the PathMNIST specialist setting---which could be attributed to its tendency to overfit~\citep{jitkrittum2023does}.
{\color{drcpe} DR CPE} with Prob01 also has a loss in performance (and a larger variance) in this setting despite outperforming its GCE counterpart in other settings.

\begin{figure}[t]%
  \includegraphics[width=\textwidth,trim={0 7pt 0 5pt},clip]{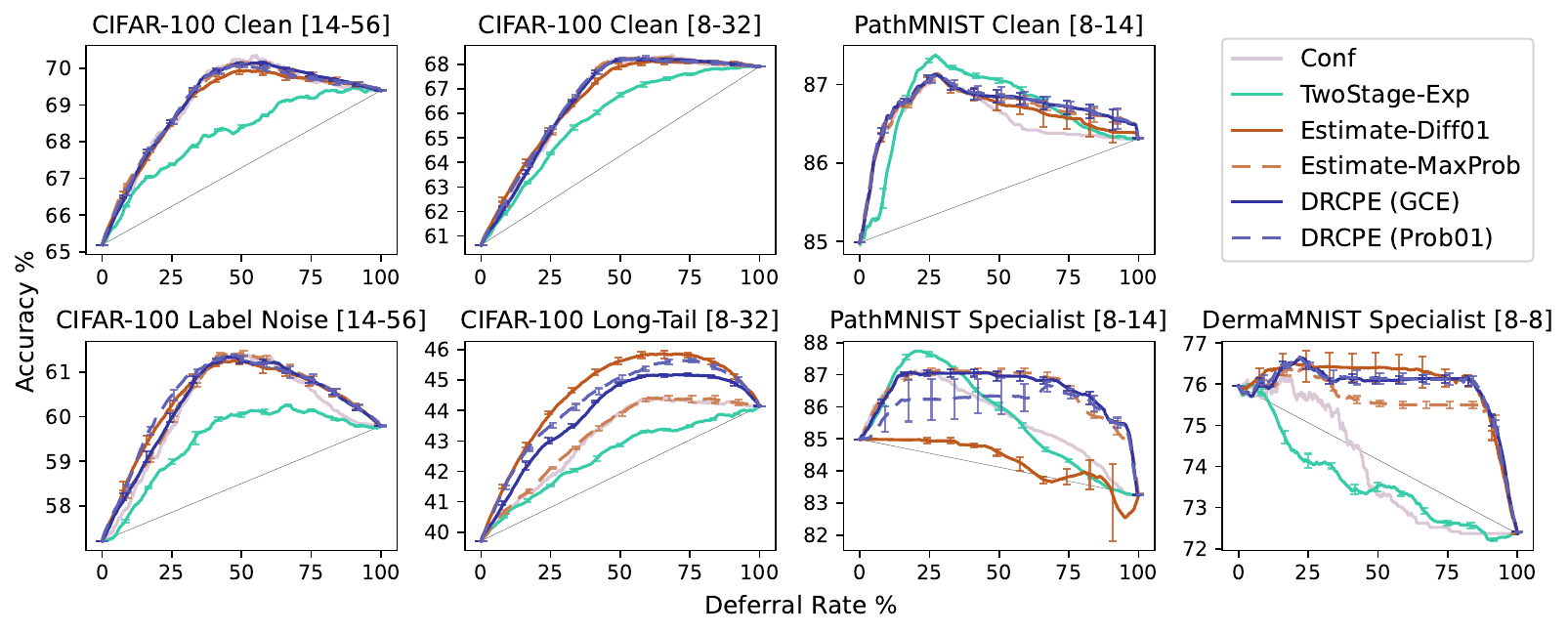}%
  \caption{%
    Accuracy-deferral trade-off curves over various experimental setups.
    The base model corresponds to $0\%$ deferral rate and the expert model corresponds to $100\%$ deferral rate.
    Curves which produce upper bounds are better.
    Titles detail the dataset and ResNet utilized, \ie, ``\texttt{[14-56]}'' $\mapsto$ $h$/$h^{(\expert)}$ = ResNet14/56. Curves depict the mean and error bars depict $\pm1$ s.t.d range (over 11 runs).
  }%
  \label{fig:tradeoff-summary}%
\end{figure}

\section{Conclusion}
We present a novel method for learning deferral mechanisms based on ideal distributions and the link between DRE and CPE. Practically, we find that the DR CPE loss function we propose can yield good accuracy-deferral trade-offs while also being robust in its performance over a variety of dataset and corruption combinations.
One avenue of future work would be to extend the initial consistency results established in \cref{thm:approx-consistency,thm:exact-consistency}, \eg, (exact) consistency for task losses $\ell$ beyond the zero-one loss.
One could attempt to achieve this by deviating away from the KL-divergence based ideal distributions and instead leverage other $\varphi$-divergences~\citep{soen2024}.
Ideally, we want to seek a consistency result that could also hold for multiple deferral costs $c \geq $ simultaneously.
So far, the narrative has been for L2D with a pair of models. Of course one could easily chain deferral mechanism, as per model cascades~\citep{wang2022wisdom,jitkrittum2023does}; however, ideally we would want to directly deal with multi-expert deferral, where there may be multiple base models and/or experts~\citep{verma2023learning,mao2023,mao2024principled,mao2025mastering,liu2026more}. An obvious connection here could be to exploit the connection between multi-distributional DRE and CPE in the multi-class setting~\citep{yu2021unified,nock2016scaled}.
Another important direction is to modify our ideal distribution framework to allow additional constraints for, \eg, fairness or limiting expert calls in training~\citep{narasimhan2024learning,charusaie2024unifying,desalvo2025budgeted}.

\section*{Acknowledgments}
The authors would like to acknowledge the NeurIPS 2026 reviews for providing helpful advice, particularly on the consistency of our approach.
The authors would like to thank Aditya Krishna Menon, Wittawat Jitkrittum, and Harikrishna Narasimhan for various fruitful discussions.

Computational resources were provided by the National Academic Infrastructure for Supercomputing in Sweden (NAISS), funded by the Swedish Research Council.
This work was funded, in part, by the Swedish research council under contracts 2023-05195 (Alexander Soen and Joakim Jald\'{e}n) and 2021-05266 (Ragnar Thobaben).
This work was also funded by Digital Futures (Alexander Soen).

\newpage

\bibliography{references}
\bibliographystyle{plainnat}

\newpage

\appendix
\startcontents[appendix]
\begingroup
\setcounter{tocdepth}{2}
\printcontents[appendix]{}{1}{\section*{Appendix Contents}}
\endgroup
\clearpage

\section{Limitations}%
\label{app:limitations}
The general bound in \cref{thm:approx-consistency} only provides ``approximate'' consistency: the upper bound still depends on a potential non-zero marginal-joint optimality gap.
To achieve an ``exact'' consistency counterpart, one further needs to assume that the task loss $\ell$ is the zero-one loss function.
The consistency results of our paper~\cref{thm:approx-consistency,thm:exact-consistency} only hold for a specific deferral cost $c \geq 0$ (and scorer threshold $\tau$). Despite this our practical results show promising results when one does not ``commit'' to a deferral costs $c$ (and $\tau$) at train time. Nevertheless, one should be careful about adapting our scorer to different thresholds at inference time.
In our experiments, we have also assumed that the deferral mechanism $r$ (and its scorer $s$) are trained using data which is \iid to the test set at inference time. Different data settings which has distribution shift could cause a non-trivial drop in performance.

\section{Broader Impact}%
\label{app:impact}
L2D inherently has impact in various downstream tasks, \eg, in health.
This can cause potential issues in scenarios of \emph{overly conservative deferral}, \ie, the overall system does not defer to a human expert when the base model is unreliable.
As such, in practice, one should ensure that the deferral rates $\tau$ selected are validated to ensure that unreliable base model predictions are not utilized at a high rate.
On the flip side, a well-calibrated L2D model has the potential of reducing the burden on a human expert.
Other impacts our and other L2D approaches could have is on the fairness and privacy of the overarching deferring system (when compared to just the base model).
Although, L2D has been previously motivated to improve fairness~\citep{madras2018predict}, our approach does not have any explicit constraints on any fairness metrics. Furthermore, deferral could impact the privacy of users, \eg, in a loan prediction system, patterns of deferral could leak certain demographic information.
These aspects should be considered when utilizing our method.

\section{Ideal Distribution with \texorpdfstring{$\varphi$-Divergences}{phi-Divergence}}%
\label{app:varphi-div-ext}

In the following section, we present a generalization of the optimal ideal distributions considered in \cref{cor:kl_weights}.
Instead of considering the KL-divergence, we will pick the family of $\varphi$-divergences (\aka, Csisz\'ar $f$-divergences) as our choice of dissimilarity $\div$ in \cref{def:ideal-dist}.
We first recall the definition of an $\varphi$-divergence.

\begin{definition}[$\varphi$-Divergence]
  Given a convex lower semi-continuous generator $\varphi \colon \mathbb{R} \to (-\infty, \infty]$ that satisfies $\varphi(1) = 0$, the corresponding $\varphi$-divergence between non-negative measures $\meas{P} \ll \meas{Q}$ is defined by
  \begin{equation}
    \fdiv{\varphi}(\meas{P} \Mid \meas{Q}) = \int_{\mathcal{Z}} \varphi\left( \frac{\dd \meas{P}}{\dd \meas{Q}}(z) \right) \dd \meas{Q}(z).
  \end{equation}
\end{definition}

We will be looking at the generic problem of solving:
\begin{equation}
  \label{eq:phi-div-opt}
  \begin{aligned}
    \min_{\dratio \colon \mathcal{Z} \to \mathbb{R}_{\ge 0}}
    \quad
    & \expect_{\meas{P}} \left[ \dratio(\Z) L(\Z) \right] + \gamma \cdot \div_{\varphi}(\dratio \cdot \meas{P} \Mid \meas{P}), \\
    \sucht
    \quad
    & \expect_{\meas{P}}[\dratio(\Z)] = 1.
  \end{aligned}
\end{equation}
where $ L \colon \mathbb{Z} \to \mathbb{R} $ and $\gamma > 0$.

\cref{eq:phi-div-opt} accounts for both marginal and joint variants of the optimization problems in \cref{def:ideal-dist}.
The only difference is that instead of optimizing for a distribution, we optimize for the density-ratio directly (similar to~\citet{dvijotham2020framework}). The corresponding ideal distribution can be recovered via $\meas{Q}(z) = \dratio(z) \meas{P}(z)$, \ie, $\dratio$ is the reweighting.
Solutions to such an equation under different assumptions on $\varphi$ are well-known and have been explored in various literatures~\citep{knoblauch2022optimization,picard2022change,soen2024,roulet2025loss}.
In what follows, we present a novel (to the best of our knowledge) solution to \cref{eq:phi-div-opt}. This will immediately give the proof of \cref{cor:kl_weights} via an instantiation of $\div_{\varphi} = \kl$.
The approach we utilize is heavily inspired by \citet{picard2022change}.

\subsection{The Result}

We first present the result and the necessary definitions required to state it.

First, we will only consider a subset of possible generators $\varphi \colon \mathbb{R} \to (-\infty, +\infty]$. Formally, we consider the set
\begin{equation}
  \Phi \defeq \left \{ \varphi \colon \mathbb{R} \to (-\infty, +\infty] \mid \varphi \textrm{ is proper convex, lower semi-continuous}, \varphi(1) = 0, \varphi(\mathbb{R}_{\geq 0}) \subset \mathbb{R} \right\}.
\end{equation}
$\Phi$ defines the class of usual $\varphi$-generators with an additional constraint that $\varphi(\mathbb{R}_{\geq 0}) \subset \mathbb{R}$. This ensures that $\varphi$ is subdifferentiable on $\mathbb{R}_{>0}$ and that the right derivative $\varphi'_+(0)$ is well-defined in $[-\infty, +\infty)$.
\begin{definition}[Right Derivative, \citet{rockafellar1970}]%
  \label{def:right-derivative}
  Let $\varphi \colon \mathbb{R} \to [-\infty, +\infty]$ be any function. Let $x$ be a point where $\varphi$ is finite. Then given that the following limit exists, the right derivative is defined as,
  \begin{equation}
    \varphi'_+(x) \defeq \lim_{h \downarrow 0} \frac{\varphi(x+h) - \varphi(x)}{h}.
  \end{equation}
\end{definition}

We also consider the standard convex conjugate.
\begin{definition}[Convex Conjugate, \citet{rockafellar1970}]
  Let $\varphi \colon \mathbb{R} \to [-\infty, +\infty]$ be any function. The convex-conjugate (Fenchel conjugate) is defined as,
  \begin{equation}
    \varphi^*(p) \defeq \sup_{x \in \mathbb{R}} \left\{ xp - \varphi(x) \right\}.
  \end{equation}
\end{definition}

We also consider a refinement of the generator class $\Phi$:
\begin{equation}
  \Phi_\uparrow \defeq \left\{ \varphi \colon \mathbb{R} \to (-\infty, +\infty] \mid \varphi \in \Phi, \varphi(\mathbb{R}_{<0}) = \{ +\infty \} \right\}.
\end{equation}
Intuitively, $\Phi_\uparrow$ is just the class of generators where negative inputs are set to $\infty$. Indeed, we can define a transform from any generator $\varphi \in \Phi$ to $\Phi_\uparrow$ via
\begin{equation}
  \varphi_\uparrow(x) \defeq
  \begin{cases}
    \varphi(x) & \text{if } x \geq 0 \\
    +\infty    & \text{otherwise}.
  \end{cases}
\end{equation}
It immediately follows that as $\varphi$ only takes the Radon-Nikodym derivative of non-negative measures in $\varphi$-divergences, the resulting divergence is invariant from the choice of generator $\varphi$ or $\varphi_\uparrow$, \ie, $\div_{\varphi} = \div_{\varphi_{\uparrow}}$. Despite this, in what follows, we will see that the properties of $\varphi_\uparrow$ are much easier to work with. For instance, the convex-conjugate of $\varphi_\uparrow$ conveniently becomes
\begin{equation}
  \label{eq:phi_uparrow_conjugate}
  \varphi_\uparrow^*(p) = \sup_{x \geq 0} \{  xp - \varphi(x) \}.
\end{equation}

With these definitions, we can now state our result.

\begin{theorem}%
  \label{thm:general-phi-div-ideal}
  Let $\varphi \in \Phi$. Suppose that there exists a constant $b \in \mathbb{R}$ such that
  \begin{equation}
    \label{eq:opt-b-finite}
    \expect_{\meas{P}} \left[ (\gamma\varphi_\uparrow)^*(b - L(\Z)) \right] < +\infty,
  \end{equation}
  the right derivative $(\varphi^*)'_+$ exists $\meas{P}$-almost surely at $\frac{b - L(\Z)}{\gamma}$,
  the quantity
  \begin{equation}
    \label{eq:opt-b-candidate-finite}
    \expect_{\meas{P}} \left[ \left| (\gamma\varphi_\uparrow)^*(b - L(\Z)) - \max\left\{ 0, (\varphi^*)'_+ \left( \frac{b - L(\Z)}{\gamma} \right) \right\} \cdot (b - L(\Z)) \right| \right] < +\infty
  \end{equation}
  is finite, and
  \begin{equation}
    \label{eq:opt-b}
    \expect_{\meas{P}} \left[ \max\left\{ 0, (\varphi^*)'_+ \left( \frac{b - L(z)}{\gamma} \right) \right\} \right] = 1
  \end{equation}
  Then optimal ratio $\dratio^\star$ of \cref{eq:phi-div-opt} is given by
  \begin{equation}
    \label{eq:general-phi-opt-dratio}
    \dratio^\star(z) = \max\left\{ 0, (\varphi^*)'_+ \left( \frac{b - L(z)}{\gamma} \right) \right\}.
  \end{equation}
\end{theorem}
It follows that \cref{cor:kl_weights} follows immediately from \cref{thm:general-phi-div-ideal} by consider the generator of the KL-divergence.

In theory, \cref{thm:general-phi-div-ideal} provides a general bridge to connect an $\varphi$-divergence to different density-ratios which would general deviate from those primarily considered in this paper (\ie, those which correspond to the KL-divergence). This naturally leads to a different set of marginal and joint ideal weights, and thus, DR loss functions (as defined in \cref{sec:dr-cpe}).

However, there is a practical catch: we can no longer exploit the functional form of the KL-divergence ideal weights to select nice marginal label probabilities $\pi$ to utilize un-normalized ideal weights. Indeed, for the KL-divergence, the normalizer $Z_{\itype, \gamma}$ is multiplicative (\cref{cor:kl_weights}). However, in the general case, the normalizer $b$ in \cref{eq:general-phi-opt-dratio} cannot be factored out.
This is a problem as we do not immediately get an unbiased objective as per \cref{def:ideal-loss,eq:joint-ideal-loss-simplify}. The best we could do is provide a two step approximation, where we first approximate the normalizer $b$ (for each model $h$ and $h^{(\expert)}$), and then create a version of \cref{def:ideal-loss} using \cref{eq:general-phi-opt-dratio}.

\subsection{Helper Lemmas}

To prove \cref{thm:general-phi-div-ideal}, we first prove a number of helper lemmas.

We first consider a simplification one can yield for the convex conjugate of $\varphi_{\uparrow}$ by examining its right derivative.
\begin{lemma}%
  \label{lem:phi_uparrow_conjugate}
  Let $\varphi \colon \mathbb{R} \to (-\infty, +\infty]$ be a proper convex, lower semi-continuous function such that $\varphi(0) < +\infty$ and $\varphi'_+(0) \in [-\infty, +\infty)$.
  \begin{equation}
    \varphi_\uparrow^*(p) =
    \begin{cases}
      -\varphi(0) & \text{if } p \leq \varphi'_+(0) \\
      \varphi^*(p) & \text{if } p > \varphi'_+(0).
    \end{cases}
  \end{equation}
\end{lemma}
\begin{proof}
  Fix $p \in \mathbb{R}$ and define $g_p(x) = xp - \varphi(x)$. Then $g_p$ is concave and
  \begin{equation*}
    \varphi_\uparrow^*(p) = \sup_{x \geq 0} g_p(x).
  \end{equation*}

  Suppose first that $\varphi'_+(0) = -\infty$. We claim that then $\varphi(x) = +\infty$ for all $x < 0$. Indeed, if there existed $x_0 < 0$ with $\varphi(x_0) < +\infty$, then for any $y > 0$ convexity would imply
  \begin{equation*}
    \frac{\varphi(y) - \varphi(0)}{y} \geq \frac{\varphi(0) - \varphi(x_0)}{-x_0} > -\infty,
  \end{equation*}
  and sending $y \downarrow 0$ would yield $\varphi'_+(0) > -\infty$ (note, \cref{def:right-derivative}), a contradiction. Hence, in this case $\varphi(x) = + \infty$ for all $x < 0$, and thus $\varphi = \varphi_\uparrow$, so $\varphi_\uparrow^*(p) = \varphi^*(p)$ for all $p \in \mathbb{R}$.

  We now proceed further with two cases with $\varphi'_+(0) > -\infty$.

  \ding{172} Suppose that $p \leq \varphi'_+(0)$. For any $x > 0$, convexity implies that
  \begin{equation*}
    \frac{\varphi(x) - \varphi(0)}{x} \geq \varphi'_+(0) \geq p.
  \end{equation*}
  Rearranging gives $ g_p(x) = xp - \varphi(x) \leq -\varphi(0) = g_p(0) $, and so $x = 0$ is a maximizer of $g_p$ over $\mathbb{R}_{\ge 0}$. Therefore, $ \varphi_\uparrow^*(p) = -\varphi(0) $.

  \ding{173} Now suppose that $p > \varphi'_+(0)$. By the definition of the right derivative, there exists $\bar{x} > 0$ such that
  \begin{equation*}
    \frac{\varphi(\bar{x}) - \varphi(0)}{\bar{x}} < p,
  \end{equation*}
  which implies $ g_p(\bar{x}) > -\varphi(0) = g_p(0) $.
  Hence, $x = 0$ is not a maximizer over $\mathbb{R}$.

  For $x < 0$ and any $y > 0$, convexity gives
  \begin{equation*}
    \frac{\varphi(0) - \varphi(x)}{-x}
    \leq
    \frac{\varphi(y) - \varphi(0)}{y}.
  \end{equation*}
  Sending $y \downarrow 0$ yields $\varphi(x) \geq \varphi(0) + \varphi'_+(0) x$,
  and thus
  \begin{equation*}
    g_p(x)
    = xp - \varphi(x)
    \leq -\varphi(0) + (p - \varphi'_+(0))x
    < -\varphi(0).
  \end{equation*}
  Since $g_p(\bar{x}) > -\varphi(0)$ for some $\bar{x} > 0$, no negative $x$ can maximize $g_p$ over $\mathbb{R}$. Therefore,
  \begin{equation*}
    \sup_{x \in \mathbb{R}} g_p(x) = \sup_{x \geq 0} g_p(x),
  \end{equation*}
  which yields $ \varphi_\uparrow^*(p) = \varphi^*(p) $.
\end{proof}

\begin{lemma}
  \label{lem:phi_uparrow_right_derivative}
  Let $\varphi \colon \mathbb{R} \to (-\infty, +\infty]$ be a proper convex, lower semi-continuous function such that $\varphi(0) < +\infty$.
  Then, at every $p \in \mathbb{R}$ where the right derivatives are defined,
  \begin{equation}
    (\varphi_{\uparrow}^*)'_+(p) = \max\{0, (\varphi^*)'_+(p)\}.
  \end{equation}
\end{lemma}
\begin{proof}
  We utilize \cref{lem:phi_uparrow_conjugate} and consider a few cases. First consider when $\varphi'_+(0) = -\infty$. From the proof of \cref{lem:phi_uparrow_conjugate}, this implies that $\varphi(x) = +\infty$ for all $x < 0$. Hence, $\varphi = \varphi_\uparrow$. Therefore,
  \begin{equation*}
    \varphi^*(p) = \sup_{x \geq 0} \{xp - \varphi(x)\}.
  \end{equation*}
  Since each affine function $p \mapsto xp - \varphi(x)$ has non-negative slope $x \geq 0$, the supremum $\varphi^*$ is non-decreasing. Thus, at every point where $(\varphi^*)'_+$ is defined, we have $(\varphi^*)'_+(p) \geq 0$, and therefore
  \begin{equation*}
    (\varphi_\uparrow^*)'_+(p)
    =
    (\varphi^*)'_+(p)
    =
    \max\{0, (\varphi^*)'_+(p)\}.
  \end{equation*}
  This proves the claim in the case $\varphi'_+(0) = -\infty$.

  We further proceed with the assumption that $\varphi'_+(0) > -\infty$.

  We note that $\varphi'_+(0) \in \partial \varphi(0)$. Indeed, for $x > 0$, convexity implies
  \begin{equation*}
    \frac{\varphi(x) - \varphi(0)}{x} \geq \varphi'_+(0),
  \end{equation*}
  and hence $\varphi(x) \geq \varphi(0) + \varphi'_+(0)x$. For $x < 0$ and any $y > 0$, convexity gives
  \begin{equation*}
    \frac{\varphi(0) - \varphi(x)}{-x}
    \leq
    \frac{\varphi(y) - \varphi(0)}{y}.
  \end{equation*}
  Sending $y \downarrow 0$ yields $\varphi(x) \geq \varphi(0) + \varphi'_+(0) x$.
  Thus, $\varphi'_+(0) \in \partial \varphi(0)$, and Fenchel-Young duality gives $ 0 \in \partial \varphi^*(\varphi'_+(0)) $.
  Hence, $\varphi'_+(0)$ is a minimizer of the convex function $\varphi^*$, and therefore
  \begin{equation*}
    (\varphi^*)'_+(p) \leq 0 \quad \text{for all } p < \varphi'_+(0),
    \qquad
    (\varphi^*)'_+(p) \geq 0 \quad \text{for all } p \geq \varphi'_+(0),
  \end{equation*}
  at every point where the right derivative is defined.

  We now further proceed via considering three cases.

  \ding{172} Suppose first that $p < \varphi'_+(0)$.
  Then $\varphi_\uparrow^*$ is constant on the interval $[p, \varphi'_+(0)]$, and hence $(\varphi_\uparrow^*)'_+(p) = 0$.
  Therefore, $ (\varphi_\uparrow^*)'_+(p) = 0 = \max\{0, (\varphi^*)'_+(p)\} $.

  \ding{173} Now suppose that $p > \varphi'_+(0)$.
  Then $\varphi_\uparrow^*(q) = \varphi^*(q)$ for all $q$ in a right-neighborhood of $p$, and therefore $ (\varphi_\uparrow^*)'_+(p) = (\varphi^*)'_+(p) $.
  Since $p > \varphi'_+(0)$, we also have $(\varphi^*)'_+(p) \geq 0$.
  Thus, $ (\varphi_\uparrow^*)'_+(p) = \max\{0, (\varphi^*)'_+(p)\} $.

  \ding{174} Finally, suppose that $p = \varphi'_+(0)$. For any $h > 0$, we have $\varphi_\uparrow^*(p+h) = \varphi^*(p+h)$, while \cref{lem:phi_uparrow_conjugate} gives $\varphi_\uparrow^*(p) = -\varphi(0)$.
  Since the argument above shows $p = \varphi'_+(0) \in \partial \varphi(0)$, Fenchel-Young duality yields $ \varphi^*(p) = -\varphi(0) = \varphi_\uparrow^*(p) $.
  Therefore,
  \begin{equation*}
    (\varphi_\uparrow^*)'_+(p)
    =
    \lim_{h \downarrow 0} \frac{\varphi_\uparrow^*(p+h) - \varphi_\uparrow^*(p)}{h}
    =
    \lim_{h \downarrow 0} \frac{\varphi^*(p+h) - \varphi^*(p)}{h}
    =
    (\varphi^*)'_+(p).
  \end{equation*}
  Since $p$ minimizes $\varphi^*$, we have $(\varphi^*)'_+(p) \geq 0$, and hence
  $ (\varphi_\uparrow^*)'_+(p) = \max\{0, (\varphi^*)'_+(p)\} $.
\end{proof}

\subsection{Proof of Result}
We now prove \cref{thm:general-phi-div-ideal}.

\begin{ProofOf}{thm:general-phi-div-ideal}
  In the following proof, we will show that $\dratio^\star$ (as per \cref{eq:general-phi-opt-dratio}) obtains an upper bound, and hence will be optimal for the specified objective~\cref{eq:phi-div-opt}.
  
  First notice that from \cref{lem:phi_uparrow_right_derivative}, we can rewrite $\dratio^\star$ as
  \begin{equation}
    \label{eq:opt-f-div-dratio-alt}
    \dratio^\star(z) = \max\left\{ 0, (\varphi^*)'_+ \left( \frac{b - L(z)}{\gamma} \right) \right\} = (\varphi_{\uparrow}^*)'_+\left( \frac{b - L(z)}{\gamma} \right) = ((\gamma \varphi_{\uparrow})^*)'_+\left( {b - L(z)} \right),
  \end{equation}
  where the middle equality follows from \cref{lem:phi_uparrow_conjugate,lem:phi_uparrow_right_derivative}:
  if $\frac{b-L(z)}{\gamma} < \varphi'_+(0)$ then $\varphi_\uparrow^*$ is locally constant, if $\frac{b-L(z)}{\gamma} > \varphi'_+(0)$ then $\varphi_\uparrow^* = \varphi^*$ locally, and if $\frac{b-L(z)}{\gamma} = \varphi'_+(0)$ then the right derivative is computed from the right and again agrees with $(\varphi^*)'_+$.
  Hence $(\varphi_\uparrow^*)'_+$ exists $\meas{P}$-almost surely at $\frac{b-L(z)}{\gamma}$.
  The last equality follows from the scaling rule of convex-conjugates and the chain rule of differentiation.

  We will now reframe the optimization objective in \cref{eq:phi-div-opt} as a maximization problem. In particular, we consider
  \begin{align*}
    \max_{\dratio \colon \mathcal{Z} \to \mathbb{R}_{\ge 0}}
    \quad
    & \expect_{\meas{P}} \left[ \dratio(\Z) (-L(\Z)) \right] - \gamma \cdot \div_{\varphi}(\dratio \cdot \meas{P} \Mid \meas{P}), \\
    \sucht
    \quad
    & \expect_{\meas{P}}[\dratio(\Z)] = 1.
  \end{align*}
  This gives us the following:
  \begin{align*}
    & \max_{\dratio \colon \mathcal{Z} \to \mathbb{R}_{\ge 0}, \expect_{\meas{P}}[\dratio(\Z)] = 1} \quad \expect_{\meas{P}} \left[ \dratio(\Z) (-L(\Z)) \right] - \gamma \cdot \div_{\varphi}(\dratio \cdot \meas{P} \Mid \meas{P}) \\
    &= \max_{\dratio \colon \mathcal{Z} \to \mathbb{R}_{\ge 0}, \expect_{\meas{P}}[\dratio(\Z)] = 1} \quad \expect_{\meas{P}} \left[ \dratio(\Z) (-L(\Z)) \right] - \gamma \cdot \div_{\varphi_\uparrow}(\dratio \cdot \meas{P} \Mid \meas{P}) \\
    &= \max_{\meas{Q} \in \triangle(\mathcal{Z})} \quad \expect_{\meas{Q}} \left[ -L(\Z) \right] - \gamma \cdot \div_{\varphi_\uparrow}(\meas{Q} \Mid \meas{P}) \\
    {}^{(\textrm{a})}&= \inf_{c \in \mathbb{R}} \quad \gamma \expect_{\meas{P}} \left[ \varphi_\uparrow^* \left( \frac{c - L(\Z)}{\gamma} \right) \right] - c \\
    {}^{(\textrm{b})}&= \inf_{c \in \mathbb{R}} \quad \expect_{\meas{P}} \left[ (\gamma\varphi_\uparrow)^* \left( c - L(\Z)\right) \right] - c \\
    {}^{(\textrm{c})}&\leq \expect_{\meas{P}} \left[ (\gamma\varphi_\uparrow)^* \left( b - L(\Z)\right) \right] - b,
  \end{align*}
  where $\textrm{(a)}$ follows from a standard convex analysis result, \eg, see \citet[Equation (22)]{lcTI}; $\textrm{(b)}$ follows from the scaling rule of convex-conjugates; and $\textrm{(c)}$ follows from taking a specific choice for $c \in \mathbb{R}$, in particular, our $b$ as per our assumption \cref{eq:opt-b}.
  The finiteness assumption \cref{eq:opt-b-finite} ensures that the right-hand side is well-defined at $c = b$.
  This yields our upper bound.

  We now consider a corresponding lower bound.
  By \cref{eq:opt-b}, $\dratio^\star$ is feasible for \cref{eq:phi-div-opt}.
  Moreover, by \cref{eq:opt-b-candidate-finite} and the Fenchel-Young identity established below, all expressions evaluated at $\dratio^\star$ are integrable, so the objective value at $\dratio^\star$ is well-defined.
  In particular, we consider
  \begin{align*}
    & \max_{\dratio \colon \mathcal{Z} \to \mathbb{R}_{\ge 0}, \expect_{\meas{P}}[\dratio(\Z)] = 1} \quad \expect_{\meas{P}} \left[ \dratio(\Z) (-L(\Z)) \right] - \gamma \cdot \div_{\varphi_\uparrow}(\dratio \cdot \meas{P} \Mid \meas{P}) \\
    &= \max_{\dratio \colon \mathcal{Z} \to \mathbb{R}_{\ge 0}, \expect_{\meas{P}}[\dratio(\Z)] = 1} \quad \expect_{\meas{P}} \left[ \dratio(\Z) (-L(\Z))  - \gamma \cdot (\varphi_\uparrow \circ \dratio)(\Z) \right] \\
    &= \max_{\dratio \colon \mathcal{Z} \to \mathbb{R}_{\ge 0}, \expect_{\meas{P}}[\dratio(\Z)] = 1} \quad \expect_{\meas{P}} \left[ \dratio(\Z) (-L(\Z))  - ((\gamma \varphi_\uparrow) \circ \dratio)(\Z) \right] \\
    &\geq \expect_{\meas{P}} \left[ \dratio^\star(\Z) (-L(\Z))  - ((\gamma \varphi_\uparrow) \circ \dratio^\star)(\Z) \right] \\
    &= \expect_{\meas{P}} \left[ \dratio^\star(\Z) (b -L(\Z))  - ((\gamma \varphi_\uparrow) \circ \dratio^\star)(\Z) \right] - \expect_{\meas{P}} \left[ \dratio^\star(\Z) b \right] \\
    &= \expect_{\meas{P}} \left[ \dratio^\star(\Z) (b -L(\Z))  - ((\gamma \varphi_\uparrow) \circ \dratio^\star)(\Z) \right] - b \cdot \expect_{\meas{P}} \left[ \dratio^\star(\Z) \right] \\
    &= \expect_{\meas{P}} \left[ \dratio^\star(\Z) (b -L(\Z))  - ((\gamma \varphi_\uparrow) \circ \dratio^\star)(\Z) \right] - b,
  \end{align*}
  where the last equality, we utilize the assumption on $b$, \cref{eq:opt-b}.

  From our assumptions, it follows that $p \mapsto (\gamma\varphi_\uparrow)^*(p)$ is proper, convex, and lower semi-continuous on $\mathbb{R}$.
  By \cref{eq:opt-f-div-dratio-alt}, $(\varphi_\uparrow^*)'_+$ exists at $\frac{b - L(z)}{\gamma}$ for $\meas{P}$-almost every $z$.
  Therefore, the scaling rule for convex conjugates implies that $((\gamma\varphi_\uparrow)^*)'_+$ exists at $b - L(z)$ for $\meas{P}$-almost every $z$.
  Moreover, \cref{eq:opt-b-finite} implies $(\gamma\varphi_\uparrow)^*(b - L(z)) < +\infty$ $\meas{P}$-almost surely.
  For any proper convex lower semi-continuous function on $\mathbb{R}$, whenever the right derivative exists at a point in its effective domain, that right derivative belongs to the subdifferential at that point.
  Hence, leveraging \cref{eq:opt-f-div-dratio-alt}, via Fenchel-Young duality, we have for $\meas{P}$-almost every $z$
  \begin{equation*}
    \dratio^\star(z) \in \partial (\gamma \varphi_\uparrow)^*(b - L(z))
    \quad \iff \quad
    b - L(z) \in \partial (\gamma \varphi_\uparrow)(\dratio^\star(z)).
  \end{equation*}
  Furthermore, Fenchel-Young equality gives for $\meas{P}$-almost every $z$
  \begin{equation*}
    (\gamma \varphi_\uparrow)^*(b - L(z)) + (\gamma \varphi_\uparrow)(\dratio^\star(z)) = (\gamma \varphi_\uparrow)^*(b - L(z)) + ((\gamma \varphi_\uparrow) \circ \dratio^\star)(z) = (b - L(z)) \cdot (\dratio^\star(z)).
  \end{equation*}
  Since \cref{eq:opt-b-candidate-finite} ensures the right-hand side below is integrable, taking expectations and substituting this identity into our lower bound yields,
  \begin{align*}
    & \max_{\dratio \colon \mathcal{Z} \to \mathbb{R}_{\ge 0}, \expect_{\meas{P}}[\dratio(\Z)] = 1} \quad \expect_{\meas{P}} \left[ \dratio(\Z) (-L(\Z)) \right] - \gamma \cdot \div_{\varphi_\uparrow}(\dratio \cdot \meas{P} \Mid \meas{P}) \\
    &\geq \expect_{\meas{P}} \left[ \dratio^\star(\Z) (b -L(\Z))  - ((\gamma \varphi_\uparrow) \circ \dratio^\star)(\Z) \right] - b \\
    &= \expect_{\meas{P}} \left[ (\gamma \varphi_\uparrow)^*(b - L(\Z)) \right] - b,
  \end{align*}
  which exactly matches our upper bound. That is, $\dratio^\star$ is optimal as per \cref{eq:phi-div-opt}.
\end{ProofOf}

\section{DRE to CPE Reduction}%
\label{app:dre_via_cpe}

The following provides a self-contained explanation of the DRE to CPE reduction. For a full exposition, see \citet{menon2016}.

\subsection{Reduction}
We first consider the problem where we wish to estimate a density-ratio.
Let $\mu, \nu \in \triangle(\mathcal{X})$. Our task is to estimate
\begin{equation}
  r^\star = \frac{\dd \mu}{\dd \nu}.
\end{equation}
We now design a binary decision problem as per the following.
Let $\pi \in (0, 1)$ be the label probability. We define the following sampling scheme:
\begin{enumerate}
  \item Sample a label $\R$ via
    \begin{equation}\label{eq:dre_cpe_label}
      \R =
      \begin{cases}
        +1 & \text{with probability } \pi; \\
        -1 & \text{with probability } 1 - \pi.
      \end{cases}
    \end{equation}
  \item Sample an example $\X$ via
    \begin{equation}\label{eq:dre_cpe_conditional}
      \X \sim
      \begin{cases}
        \mu & \text{if } \R = +1; \\
        \nu & \text{if } \R = -1.
      \end{cases}
    \end{equation}
\end{enumerate}
Denote the joint law of this sampling scheme $(\X, \R)$ as $\meas{E}$, \ie, $(\X, \R) \sim \meas{E}$.
We can now define a classification problem for this binary decision problem. Let $(\drloss, \link)$ correspond to a strictly proper composite loss function. Then for a scorer $s$, we minimize
\begin{equation}
  \mathop{\expect}_{(\X, \R) \sim \meas{E}}[\drloss(s(\X), \R)]
  =
  \pi \cdot \expect_{\mu}[(\drloss_{+1} \circ s)(\X)]
  +
  (1 - \pi) \cdot \expect_{\nu}[(\drloss_{-1} \circ s)(\X)],
\end{equation}
where we have factorized the law $\meas{E}$ into its marginal label probability \cref{eq:dre_cpe_label} and its conditional \cref{eq:dre_cpe_conditional}.
This motivates the objective in \cref{eq:dre-cpe-obj}.

As $(\drloss, \link)$ is strictly proper composite~\citep{reid2010}, the population minimizer $s^\star$ will satisfy the following Bayes-optimal class-probabilities
\begin{equation}\label{eq:dre_cpe_opt_scorer}
  (\Phi^{-1} \circ s^\star)(x) = \mathbb{P}[\R = + 1 \mid \X = x].
\end{equation}
Finally, by the construction of our binary decision problem (\cref{eq:dre_cpe_label,eq:dre_cpe_conditional}), one can verify via Bayes rule that (symmetrically)
\begin{equation}\label{eq:dre_cpe_class_posterior}
  \begin{aligned}
    \mathbb{P}[\R = +1 \mid \X = x] &= \frac{\pi \mu(x)}{\pi \mu(x) + (1 - \pi) \nu(x)} \\
    \mathbb{P}[\R = -1 \mid \X = x] &= \frac{(1-\pi) \nu(x)}{\pi \mu(x) + (1 - \pi) \nu(x)}.
  \end{aligned}
\end{equation}
Dividing the two class-probabilities in \cref{eq:dre_cpe_class_posterior} yields
\begin{equation*}
  \frac{\mathbb{P}[\R = +1 \mid \X = x]}{\mathbb{P}[\R = -1 \mid \X = x]}
  =
  \frac{\pi}{1-\pi} \frac{\mu(x)}{\nu(x)}.
\end{equation*}
Substituting in \cref{eq:dre_cpe_opt_scorer} (and taking an approximate $s$) yields \cref{eq:dre-cpe-dr}.

\subsection{Example}
We further provide a simple worked example on a two point distribution. Suppose we have $\mu = [0.3, 0.7]$, $\nu = [0.5, 0.5]$, and $\pi = 0.5$ for simplicity.
Writing the two points as $x_1$ ,$x_2$, the Bayes-optimal class-posterior is
\begin{equation*}
  \mathbb{P}[R=+1\mid X=x_1]=\frac{0.3}{0.3+0.5}=\frac{3}{8},
  \qquad
  \mathbb{P}[R=+1\mid X=x_2]=\frac{0.7}{0.7+0.5}=\frac{7}{12}.
\end{equation*}
Since $\pi=0.5$, the prior-odds correction ${\pi}/({1-\pi})$ is one. Thus, a Bayes-optimal CPE estimator yields
\begin{equation*}
  \left[ \frac{3/8}{1-3/8}, \frac{7/12}{1-7/12} \right] = \left[\frac{3}{5},\frac{7}{5}\right] = \left[\frac{0.3}{0.5},\frac{0.7}{0.5}\right] =\frac{\dd\mu}{\dd\nu}.
\end{equation*}

\section{Connection to Chow's Rule}%
\label{app:chow}

\subsection{Proof of \texorpdfstring{\cref{thm:rel-to-chow}}{Theorem 3.5}}

\begin{ProofOf}{thm:rel-to-chow}
  Fix $\gamma > 0$.
  From \cref{cor:kl_weights} we identify that the density-ratio $\dratio_{\itype}$ corresponds to the ratio of normalized weight function $w^{\cdot}_\itype$.
  Thus, for a $\tau > 0$, we have rejectors
  \begin{equation*}
    r_\itype(x; \tau, \gamma, \gamma)
    = \left \llbracket \frac{w_\itype(x; \gamma)}{w^{(\expert)}_\itype(x; \gamma)} \leq \tau \right \rrbracket
    = \left \llbracket \frac{\tilde{w}_\itype(x; \gamma)}{\tilde{w}^{(\expert)}_\itype(x; \gamma)} \leq \tau' \right \rrbracket,
  \end{equation*}
  where we are just setting $\tau' = \tau Z_{\itype, \gamma} / Z^{(\expert)}_{\itype, \gamma}$ (a positive scaling of the original threshold value).

  To show that the $\kl$-marginal density-ratio deferral rule is equivalent to Chow's Rule \wrt $\meas{P}$, we consider the following simplification:
  \begin{align*}
    \frac{\tilde{w}_\mtype(x; \gamma)}{\tilde{w}^{(\expert)}_\mtype(x; \gamma)} \leq \tau'
    & \iff
    \frac{\exp\left( -\frac{\expect_{\pos^\star(x)}\left[ \ell(x, \Y, h(x)) \right] }{\gamma} \right)}{\exp\left( -\frac{\expect_{\pos^\star(x)}\left[ \ell(x, \Y, h^{(\expert)}(x)) \right] }{\gamma} \right)} \leq \tau' \\
    & \iff
    {\exp\left( \frac{\expect_{\pos^\star(x)}\left[ \ell(x, \Y, h^{(\expert)}(x)) \right] - \expect_{\pos^\star(x)}\left[ \ell(x, \Y, h(x)) \right] }{\gamma} \right)} \leq \tau' \\
    & \iff
    \expect_{\pos^\star(x)}\left[ \ell(x, \Y, h^{(\expert)}(x)) - \ell(x, \Y, h(x)) \right] \leq \gamma \log \tau' \\
    & \iff
    \expect_{\pos^\star(x)}\left[ \ell(x, \Y, h(x)) - \ell(x, \Y, h^{(\expert)}(x)) \right] \geq -\gamma \log \tau'.
  \end{align*}
  Notice that this is exactly the condition for $r^\star(x; -\gamma \log \tau', \meas{P}) = 1$.
  For any prescribed cost $c \in \mathbb{R}$, set $\tau' = \exp(-c / \gamma) > 0$. Since $\tau' = \tau Z_{\mtype, \gamma} / Z^{(\expert)}_{\mtype, \gamma}$, choosing
  \begin{equation*}
    \tau = \exp(-c / \gamma) \frac{Z^{(\expert)}_{\mtype, \gamma}}{Z_{\mtype, \gamma}} > 0
  \end{equation*}
  yields $-\gamma \log \tau' = c$ and hence $r_\mtype(x; \tau, \gamma, \gamma) = r^\star(x; c, \meas{P})$.

  To show the connection that $\kl$-joint density-ratio deferral rule has between Chow's Rule on the expert-tilted distribution, we consider the following:
  \begin{align*}
    \frac{\tilde{w}_\jtype(x; \gamma)}{\tilde{w}^{(\expert)}_\jtype(x; \gamma)} \leq \tau'
    & \iff \frac{\expect_{\bpos(x)}\left[ \exp\left( -\frac{\ell(x, \Y, h(x))}{\gamma} \right) \right]}{\expect_{\bpos(x)}\left[ \exp\left( -\frac{\ell(x, \Y, h^{(\expert)}(x))}{\gamma} \right) \right]} \leq \tau' \\
    & \iff \expect_{\bpos(x)}\left[ t(\Y \mid x; \gamma, \meas{P}) \frac{\exp\left( -\frac{\ell(x, \Y, h(x))}{\gamma} \right)}{\exp\left( -\frac{\ell(x, \Y, h^{(\expert)}(x))}{\gamma} \right)} \right] \leq \tau' \\
    & \iff \log \expect_{\tbpos(x; \gamma)}\left[ \exp\left( \frac{\ell(x, \Y, h^{(\expert)}(x)) - \ell(x, \Y, h(x))}{\gamma} \right) \right] \leq \log \tau'.
  \end{align*}
  Using Jensen's inequality, we have
  \begin{align*}
    \frac{\tilde{w}_\jtype(x; \gamma)}{\tilde{w}^{(\expert)}_\jtype(x; \gamma)} \leq \tau'
    \implies
    \expect_{\tbpos(x; \gamma)}\left[ {\ell(x, \Y, h(x)) - \ell(x, \Y, h^{(\expert)}(x))} \right] \geq -\gamma \log \tau'.
  \end{align*}
  The RHS condition is exactly equivalent to $r^\star(x; -\gamma \log \tau', \tilde{\meas{P}}_{\gamma}) = 1$ for the posterior twisted distribution $\tilde{\meas{P}}_{\gamma}$. For any prescribed cost $c \in \mathbb{R}$, set $\tau' = \exp(-c / \gamma) > 0$, and then choose
  \begin{equation*}
    \tau = \exp(-c / \gamma) \frac{Z^{(\expert)}_{\jtype, \gamma}}{Z_{\jtype, \gamma}} > 0.
  \end{equation*}
  This gives $-\gamma \log \tau' = c$, hence $r_\jtype(x; \tau, \gamma, \gamma) \leq r^\star(x; c, \tilde{\meas{P}}_\gamma)$.
\end{ProofOf}

\section{Joint Ideal DR CPE Loss Properties}%
\label{app:joint}

In the following section, we provide a simple proof of the Jensen's ordering and also give a result to show that the gap between the marginal and joint ideal DR loss objectives is controlled by conditional task loss variances (weighted by the DR partial losses).

\subsection{Proof of \texorpdfstring{\cref{lem:obj-jensen-ordering}}{Lemma 4.3}}

\begin{ProofOf}{lem:obj-jensen-ordering}
  The result immediately follows from Jensen's inequality.
\end{ProofOf}

\subsection{Variance upper bound of Jensen's gap}

\begin{lemma}%
  \label{lem:taylor-exp}
  Let $\Z$ be a bounded random variable in $[m, M]$ for $m < M$. Furthermore, let $\expect[\Z] = 0$. Then for any $\gamma > 0$, we have that
  \begin{equation}
    \exp\left( - \frac{M}{\gamma} \right)
    \frac{1}{2\gamma^2}
    \var[\Z]
    \le
    \left \vert \expect\left[\exp\left( -\frac{\Z}{\gamma} \right)\right] - 1 \right \vert
    \le
    \exp\left( - \frac{m}{\gamma} \right)
    \frac{1}{2\gamma^2}
    \var[\Z].
  \end{equation}
\end{lemma}
\begin{proof}
  We consider the 2nd order Taylor expansion of the map $z \mapsto \exp(-z / \gamma)$ around $z = 0$ with an explicit Lagrange remainder. Given the maps sufficient smoothness, we have
  \begin{align*}
    \exp(-z / \gamma)
    &= \exp(0) - \frac{\exp(0)}{\gamma}z + \frac{\exp(- (\theta z) / \gamma)}{2\gamma^2} z^2 \\
    &= 1 - \frac{z}{\gamma} + \frac{z^2}{2\gamma^2} \exp(- (\theta z) / \gamma),
  \end{align*}
  where $\theta \in [0, 1]$.

  Re-arranging, setting $z = \Z$, and using the positivity of $z^2 / (2\gamma^2)$, we get the bound
  \begin{align*}
    -\frac{\Z}{\gamma} + \frac{\Z^2}{2\gamma^2} \exp\left( -\frac{M}{\gamma} \right)
    \le
    \exp\left(-\frac{\Z}{\gamma}\right) - 1
    \le
    -\frac{\Z}{\gamma} + \frac{\Z^2}{2\gamma^2} \exp\left( -\frac{m}{\gamma} \right).
  \end{align*}
  Finally, taking an expectation and noting that $\Z$ is a centered random variable yields the desired result.
\end{proof}

\begin{theorem}%
  \label{thm:joint-marginal-gap-formal}
  Suppose $0 \leq \ell \leq B$ and that the partial losses satisfy $\drloss_{+1} \geq 0$ and $\drloss_{-1} \geq 0$. Then
  \begin{align*}
    &\frac{\exp(-B / \gamma)}{2\gamma^2}
    \expect_{\meas{P}_{\rm x}}\left[(\drloss_{+1} \circ s)(\X) \var_{\bpos(\X)}[\ell(\X, \Y, h(\X))]\right] \\
    &\qquad\qquad
    + \frac{\exp(-B / \gamma^{(\expert)})}{2(\gamma^{(\expert)})^2}
    \expect_{\meas{P}_{\rm x}}\left[(\drloss_{-1} \circ s)(\X) \var_{\bpos(\X)}[\ell(\X, \Y, h^{(\expert)}(\X))]\right] \\
    &\leq
    \mathcal{L}_{\jtype}(s; \gamma, \gamma^{(\expert)}) - \mathcal{L}_{\mtype}(s; \gamma, \gamma^{(\expert)}) \\
    &\leq
    \frac{1}{2\gamma^2}
    \expect_{\meas{P}_{\rm x}}\left[(\drloss_{+1} \circ s)(\X) \var_{\bpos(\X)}[\ell(\X, \Y, h(\X))]\right] \\
    &\qquad\qquad
    + \frac{1}{2(\gamma^{(\expert)})^2}
    \expect_{\meas{P}_{\rm x}}\left[(\drloss_{-1} \circ s)(\X) \var_{\bpos(\X)}[\ell(\X, \Y, h^{(\expert)}(\X))]\right].
  \end{align*}
\end{theorem}
\begin{ProofOf}{thm:joint-marginal-gap-formal}
  From \cref{lem:obj-jensen-ordering} and the non-negativity of the partial losses, we have that
  \begin{equation*}
    \mathcal{L}_{\jtype}(s; \gamma, \gamma^{(\expert)}) - \mathcal{L}_{\mtype}(s; \gamma, \gamma^{(\expert)})
    = A + B,
  \end{equation*}
  where
  \begin{equation*}
    \scalemath{0.90}{%
      \begin{aligned}
        A &= \expect_{\X \sim \meas{P}_{\rm x}}\left[ (\drloss_{+1} \circ s)(\X) \expect_{\Y \sim \bpos(\X)}\left[ \exp\left( - \frac{\ell(\X, \Y, h(\X))}{\gamma} \right) - \exp\left( - \frac{\expect_{\Y \sim \bpos(\X)}[\ell(\X, \Y, h(\X))]}{\gamma} \right)\right] \right] \\
        B &= \expect_{\X \sim \meas{P}_{\rm x}}\left[ (\drloss_{-1} \circ s)(\X) \expect_{\Y \sim \bpos(\X)}\left[ \exp\left( - \frac{\ell(\X, \Y, h^{(\expert)}(\X))}{\gamma^{(\expert)}} \right) - \exp\left( - \frac{\expect_{\Y \sim \bpos(\X)}[\ell(\X, \Y, h^{(\expert)}(\X))]}{\gamma^{(\expert)}} \right)\right] \right].
      \end{aligned}
    }
  \end{equation*}
  We will bound each of these terms identically, noting that the only difference between them is that $\gamma$ is replaced by $\gamma^{(\expert)}$ and $h$ is replaced by $h^{(\expert)}$. In particular, we examine the inner expectation of $A$. \Wlog, we have
  \begin{equation*}
    \scalemath{0.91}{%
      \begin{aligned}
        & \expect_{\Y \sim \bpos(\X)}\left[ \exp\left( - \frac{\ell(\X, \Y, h(\X))}{\gamma} \right) - \exp\left( - \frac{\expect_{\Y \sim \bpos(\X)}[\ell(\X, \Y, h(\X))]}{\gamma} \right)\right] \\
        &=
        \exp\left( - \frac{\expect_{\Y \sim \bpos(\X)}[\ell(\X, \Y, h(\X))]}{\gamma} \right)
        \expect_{\Y \sim \bpos(\X)}\left[ \exp\left( - \frac{\ell(\X, \Y, h(\X))}{\gamma} + \frac{\expect_{\Y \sim \bpos(\X)}[\ell(\X, \Y, h(\X))]}{\gamma}\right) - 1 \right] \\
        &=
        \exp\left( - \frac{\expect_{\Y \sim \bpos(\X)}[\ell(\X, \Y, h(\X))]}{\gamma} \right)
        \expect_{\Y \sim \bpos(\X)}\left[ \exp\left( - \frac{\ell(\X, \Y, h(\X)) - \expect_{\Y \sim \bpos(\X)}[\ell(\X, \Y, h(\X))]}{\gamma}\right) - 1 \right].
      \end{aligned}
    }
  \end{equation*}
  Noting that for fixed $\X = x$, we can define a centered random variable
  \begin{equation*}
    \Z = \ell(\X, \Y, h(\X)) - \langle \ell \rangle,
  \end{equation*}
  where we shorthand $\langle \ell \rangle = \expect_{\Y \sim \bpos(\X)}[\ell(\X, \Y, h(\X))]$.

  Since $0 \leq \ell \leq B$, we have $\langle \ell \rangle \in [0, B]$ and therefore $\Z \in [-\langle \ell \rangle, B - \langle \ell \rangle]$.
  Thus, via \cref{lem:taylor-exp}, we have that
  \begin{align*}
    &\expect_{\Y \sim \bpos(\X)}\left[ \exp\left( - \frac{\ell(\X, \Y, h(\X))}{\gamma} \right) - \exp\left( - \frac{\expect_{\Y \sim \bpos(\X)}[\ell(\X, \Y, h(\X))]}{\gamma} \right)\right] \\
    &\in \frac{1}{2\gamma^2} \var[\ell(\X, \Y, h(\X)) \mid \X = x] \left[ \exp\left(-\frac{B}{\gamma}\right), 1 \right].
  \end{align*}
  Multiplying by the non-negative factor $(\drloss_{+1} \circ s)(x)$ and taking an outer expectation gives
  \begin{align*}
    &\frac{\exp(-B / \gamma)}{2\gamma^2}
    \expect_{\meas{P}_{\rm x}}\left[(\drloss_{+1} \circ s)(\X) \var[\ell(\X, \Y, h(\X)) \mid \X]\right] \\
    &\qquad\qquad\leq A
    \leq \\
    &\frac{1}{2\gamma^2}
    \expect_{\meas{P}_{\rm x}}\left[(\drloss_{+1} \circ s)(\X) \var[\ell(\X, \Y, h(\X)) \mid \X]\right].
  \end{align*}
  Applying the same argument to $B$ with $(\gamma^{(\expert)}, h^{(\expert)})$ in place of $(\gamma, h)$ and summing the two bounds yields the claim.
\end{ProofOf}

\section{Approximate Consistency of DR CPE Losses}%
\label{app:consistent}

In the following section, we provide formal proofs and additional result regarding the approximate consistency of the joint ideal DR CPE losses.
We formally define the additional assumptions required to prove various consistency-style results.
For this section, we focus on the setting where we have a common temperature $\gamma = \gamma^{(\expert)} > 0$.

\subsection{Assumptions and Definitions}
To prove consistency results, we utilize a standard surrogate-regret analysis utilized in \citet{agarwal2014surrogate}. To do so, we consider the class of \emph{strongly proper} loss functions which will allow us to transfer the L2D regret (derived from its risk defined in \cref{eq:l2d}) to the regret of DR CPE losses.

\begin{definition}[{\citet[Definition 7]{agarwal2014surrogate}}]\label{def:strongly-proper}
  A (proper) composite loss function $(\drloss, \link)$ is $\beta$-strongly proper composite if
  \begin{equation}\label{eq:strongly-proper}
    L(\pos, v; \drloss) - L(\pos, \Phi(\pos); \drloss)
    \geq
    \frac{\beta}{2} \cdot \left(
      \Phi^{-1}(v) - \pos
    \right)^2
  \end{equation}
  where $L(\pos, v; \drloss)$ corresponds to the pointwise loss of $\uplambda$, given by
  \begin{equation}\label{eq:pw-loss}
    L(\pos, v; \drloss) \defeq \pos \drloss_{+1}(v) + (1 - \pos) \drloss_{-1}(v).
  \end{equation}
\end{definition}
The second term on the LHS of \cref{eq:strongly-proper} for a proper composite loss corresponds to the pointwise Bayes risk of the loss function, \ie, it is the minimum obtainable pointwise loss possible.
The definition essentially ensure strong concavity of the Bayes pointwise risk~\citep[Theorem 10]{agarwal2014surrogate}.

When we examine the pointwise minimizer of the marginal ideal objective $\mathcal{L}_\mtype(s; \gamma, \gamma)$ with shared temperature $\gamma = \gamma^{(\expert)} > 0$. In this shared temperature case, we can explicitly state the optimal scorer in terms of the un-normalized marginal reweightings.
\begin{lemma}\label{lem:marg-opt-scorer-equal-temp}
  Let $(\drloss, \link)$ be a proper composite loss function.
  Then, the following is a minimizer for $s_\mtype^\star \in \inf_s \mathcal{L}_\mtype(s; \gamma, \gamma)$:
  \begin{equation}
    s_\mtype^\star(x; \gamma, \gamma) = \Phi(p_\mtype(x; \gamma)),
  \end{equation}
  where
  \begin{equation}\label{eq:induced-weight-marginal}
    p_\mtype(x; \gamma) \defeq \frac{\tilde{w}_\mtype(x; \gamma)}{\tilde{w}_\mtype(x; \gamma) + \tilde{w}_\mtype^{(\expert)}(x; \gamma)}.
  \end{equation}
\end{lemma}
\begin{proof}
  We write the marginal DR loss objective as
  \begin{align}
    &\mathcal{L}_\mtype(s; \gamma, \gamma) \nonumber \\
    &= \expect_{\meas{P}_{\rm x}}\left[\tilde{w}_\mtype(\X; \gamma) \cdot (\drloss_{+1} \circ s)(\X) + \tilde{w}^{(\expert)}_\mtype(\X; \gamma) \cdot (\drloss_{-1} \circ s)(\X)\right] \nonumber \\
    &= \expect_{\meas{P}_{\rm x}}\left[
      \left( \tilde{w}_\mtype(\X; \gamma) + \tilde{w}_\mtype^{(\expert)}(\X; \gamma) \right)
      \cdot
      \left(
        p_\mtype(\X; \gamma) \cdot (\drloss_{+1} \circ s)(\X) + (1-p_\mtype(\X; \gamma)) \cdot (\drloss_{-1} \circ s)(\X)
      \right)
    \right] \nonumber \\
    &= \expect_{\meas{P}_{\rm x}}\left[
      \left( \tilde{w}_\mtype(\X; \gamma) + \tilde{w}_\mtype^{(\expert)}(\X; \gamma) \right)
      \cdot
      L(p_\mtype(\X; \gamma), s(\X); \drloss)
    \right]. \label{eq:mtype-drloss-risk-as-pointwise}
  \end{align}
  However, as $(\drloss, \link)$ is a proper composite loss function, then the pointwise optimal scorer must be given by $s^\star_\mtype = \Phi \circ p_\mtype(\cdot; \gamma)$, as required.
\end{proof}

A further helpful lemma that we will use is to relate the Chow's optimal deferral mechanism with the induced marginal ideal distribution probability $p_\mtype$ in \cref{eq:induced-weight-marginal}.
\begin{lemma}\label{lem:marg-opt-rejector-equal-temp}
  The optimal rejector given by \cref{lem:chow} $r^\star(x; c, \meas{P})$ can be stated as
  \begin{equation}
    r^\star(x; c, \meas{P})
    = \left\llbracket p_\mtype(x; \gamma) \leq \sigma\left(- \frac{c}{\gamma} \right) \right\rrbracket,
  \end{equation}
  where $\sigma(t) \defeq (1 + \exp(-t))^{-1}$ is the sigmoid function.
\end{lemma}
\begin{proof}
  Notice that we directly have
  \begin{align}
    p_\mtype(x; \gamma)
    &= \frac{\tilde{w}_\mtype(x; \gamma)}{\tilde{w}_\mtype(x; \gamma) + \tilde{w}_\mtype^{(\expert)}(x; \gamma)} \nonumber \\
    &= \frac{\exp\left( -\frac{\expect_{\Y \sim \pos^\star(x)}\left[ \ell(x, \Y, h(x)) \right] }{\gamma} \right)}{\exp\left( -\frac{\expect_{\Y \sim \pos^\star(x)}\left[ \ell(x, \Y, h(x)) \right] }{\gamma} \right) + \exp\left( -\frac{\expect_{\Y \sim \pos^\star(x)}\left[ \ell(x, \Y, h^{(\expert)}(x)) \right] }{\gamma} \right)} \nonumber \\
    &= \left( 1 + \exp\left( \frac{\expect_{\Y \sim \pos^\star(x)}\left[ \ell(x, \Y, h(x)) - \ell(x, \Y, h^{(\expert)}(x)) \right] }{\gamma} \right) \right)^{-1} \nonumber \\
    &= \sigma\left( - \frac{\expect_{\Y \sim \pos^\star(x)}\left[ \ell(x, \Y, h(x)) - \ell(x, \Y, h^{(\expert)}(x)) \right] }{\gamma} \right). \label{eq:p-mtype-as-sigmoid}
  \end{align}
  Finally, the invertibility of $\sigma$ gives the desired result.
\end{proof}

For convenience, we also define the follow risk related quantities for the DR CPE loss functions (for $\itype \in \{ \jtype, \mtype \}$):
\begin{align}
  \mathcal{L}_{\itype}^\star(\gamma, \gamma^{(\expert)})
  &\defeq \inf_s \mathcal{L}_{\itype}(s; \gamma, \gamma^{(\expert)}); \\
  \mathcal{E}_{\itype}(s; \gamma, \gamma^{(\expert)})
  &\defeq \mathcal{L}_{\itype}(s; \gamma, \gamma^{(\expert)}) - \mathcal{L}_{\itype}^\star(\gamma, \gamma^{(\expert)}); \\
  \Delta_{\mtype\jtype}(\gamma, \gamma^{(\expert)})
  &\defeq \mathcal{L}_{\jtype}^\star(\gamma, \gamma^{(\expert)}) - \mathcal{L}_{\mtype}^\star(\gamma, \gamma^{(\expert)}),
\end{align}
which correspond to the Bayes-optimal risk, excess risk, and joint-marginal optimality gap of the DR CPE losses, respectively.
The infimum is over all measurable scorers.
As per the notation of the main text, we make the dependence of the composite loss $(\drloss, \link)$ implicit.

\subsection{Helper Lemma}

To prove the general consistency result, we first prove a general lemma which connects the deferral regret with the explicit disagreements of deferral decisions. This corresponds to a generalization of \citet[Corollary 3.2]{jitkrittum2023does}.

\begin{lemma}%
  \label{lem:regret-to-deferral-decisions}
  For a deferral cost $c \geq 0$ and any rejector $r \colon \mathcal{X} \to \{ 0, 1 \}$, the regret is given by
  \begin{equation}
    \mathcal{L}(h, r; c) - \mathcal{L}(h, r^\star; c) = \expect_{\X \sim \meas{P}_{\rm x}}\left[
      \vert \alpha(\X) \vert \cdot \left \llbracket
      r(\X) \neq r^\star(\X; c, \meas{P})
      \right \rrbracket
    \right]
  \end{equation}
  where $\alpha(x) = \expect_{\Y \sim \bpos(x)} \left[ \ell(x, \Y, h(x)) - \ell(x, \Y, h^{(\expert)}(x)) \right] - c$.
\end{lemma}
\begin{proof}
  For a fixed classification model $h \colon \mathcal{X} \to \mathbb{R}^L$ and expert model $h^{(\expert)} \colon \mathcal{X} \to \mathbb{R}^L$, one can write the L2D risk (\cref{eq:l2d}) as
  \begin{align*}
    \mathcal{L}(h, r; c)
    &= \expect_{\X, \Y \sim \meas{P}} \left[ (1 - r(\X)) \cdot \ell(\X, \Y, h(\X)) + r(\X) \cdot \ell(\X, \Y, h^{(\expert)}(\X)) \right] + c \cdot \prob_{\X \sim \meas{P}_{\rm x}}[r(\X) = 1] \\
    &= \expect_{\X, \Y \sim \meas{P}}\left[
      \ell(\X, \Y, h(\X))
    \right]
    +
    \expect_{\X, \Y \sim \meas{P}} \left[ r(\X) \left( \ell(\X, \Y, h^{(\expert)}(\X)) - \ell(\X, \Y, h(\X)) + c \right) \right] \\
    &= \expect_{\X, \Y \sim \meas{P}}\left[
      \ell(\X, \Y, h(\X))
    \right]
    -
    \expect_{\X \sim \meas{P}_{\rm x}} \left[ r(\X) \alpha(\X) \right].
  \end{align*}
  In the second equality, we are using the fact that $r$ returns binary outputs.

  Thus, the regret is given by
  \begin{align*}
    \mathcal{L}(h, r; c) - \mathcal{L}(h, r^\star; c)
    &=
    \expect_{\X \sim \meas{P}_{\rm x}} \left[ \alpha(\X) \cdot \left(r^\star(\X; c, \meas{P}) - r(\X) \right) \right].
  \end{align*}
  Finally, from \cref{lem:chow}, notice that $\alpha(x) \geq 0$ iff $r^\star(x; c, \meas{P}) = 1$. Otherwise, $r^\star(x; c, \meas{P}) = 0$.
  Thus we have,
  \begin{align*}
    \alpha(x) \cdot \left(r^\star(x; c, \meas{P}) - r(x) \right)
    &=
    \begin{cases}
      \vert \alpha(x) \vert \cdot \left(r^\star(x; c, \meas{P}) - r(x) \right) & \alpha(x) \geq 0, r^\star(x; c, \meas{P}) = 1 \\
      \vert \alpha(x) \vert \cdot \left(r(x) - r^\star(x; c, \meas{P}) \right) & \alpha(x) \leq 0, r^\star(x; c, \meas{P}) = 0 \\
    \end{cases} \\
    &=
    \begin{cases}
      \vert \alpha(x) \vert \cdot \left(1 - r(x) \right) & \alpha(x) \geq 0, r^\star(x; c, \meas{P}) = 1 \\
      \vert \alpha(x) \vert \cdot \left(r(x) - 0 \right) & \alpha(x) \leq 0, r^\star(x; c, \meas{P}) = 0 \\
    \end{cases} \\
    &=
    \begin{cases}
      \vert \alpha(x) \vert \cdot \llbracket r(x) = 0 \rrbracket & \alpha(x) \geq 0, r^\star(x; c, \meas{P}) = 1 \\
      \vert \alpha(x) \vert \cdot \llbracket r(x) = 1 \rrbracket & \alpha(x) \leq 0, r^\star(x; c, \meas{P}) = 0 \\
    \end{cases} \\
    &=
    \begin{cases}
      \vert \alpha(x) \vert \cdot \llbracket r(x) \neq r^\star(x) \rrbracket & \alpha(x) \geq 0, r^\star(x; c, \meas{P}) = 1 \\
      \vert \alpha(x) \vert \cdot \llbracket r(x) \neq r^\star(x) \rrbracket & \alpha(x) \leq 0, r^\star(x; c, \meas{P}) = 0 \\
    \end{cases} \\
    &= \vert \alpha(x) \vert \cdot \llbracket r(x) \neq r^\star(x; c, \meas{P}) \rrbracket.
  \end{align*}
  Taking an expectation over $x$ yields the desired result.
\end{proof}

\subsection{Proof of \texorpdfstring{\cref{thm:approx-consistency}}{Theorem 4.4} (Approximate Consistency Bound)}

\begin{theorem}\label{thm:general-consistency-formal}
  Suppose that the loss $\ell$ is bounded in $[0, B]$, the composite loss $(\drloss, \link)$ is $\beta$-strongly proper with non-negative partial losses $\drloss_{+1}, \drloss_{-1} \geq 0$, and we consider a shared temperature $\gamma = \gamma^{(\expert)} > 0$.
  Define the rejector
  \begin{equation}
    r_{c, \gamma}(x) \defeq \llbracket (\link^{-1} \circ s)(x) \leq \tau_{c, \gamma} \rrbracket,
    \qquad
    \tau_{c, \gamma} \defeq \sigma\left(-\frac{c}{\gamma}\right).
  \end{equation}
  Then
  \begin{equation}
    \mathcal{L}(h, r_{c, \gamma}; c) - \mathcal{L}(h, r^\star; c)
    \leq
    K_{\gamma, B, c} \sqrt{\frac{\exp(B / \gamma)}{\beta} \left[
        \mathcal{E}_{\jtype}(s; \gamma, \gamma) + \Delta_{\mtype \jtype}(\gamma, \gamma)
      \right]
    },
  \end{equation}
  where
  \begin{equation}
    K_{\gamma, B, c} \defeq \frac{\gamma}{\sigma\left( -\frac{\max\{ B, c \}}{\gamma} \right) \left(1- \sigma\left( -\frac{\max\{ B, c \}}{\gamma} \right)\right)}.
  \end{equation}
\end{theorem}
\begin{proof}
  The proof follows similarly to \citet[Theorem 13]{agarwal2014surrogate}.

  From strong-properness (\cref{def:strongly-proper}), \cref{eq:mtype-drloss-risk-as-pointwise}, and \cref{lem:marg-opt-scorer-equal-temp}, we can lower bound the marginal DR loss excess risk:
  \begin{equation}\label{eq:approx-consistency-excess-simplify}
    \begin{aligned}
      &\mathcal{E}_\mtype(s; \gamma, \gamma) \\
      &= \mathcal{L}_{\mtype}(s; \gamma, \gamma) - \mathcal{L}_{\mtype}^\star(\gamma, \gamma) \\
      &= \mathcal{L}_{\mtype}(s; \gamma, \gamma) - \mathcal{L}_{\mtype}(s^{\star}_\mtype; \gamma, \gamma) \\
      &= \expect_{\meas{P}_{\rm x}}\left[
        \left( \tilde{w}_\mtype(\X; \gamma) + \tilde{w}_\mtype^{(\expert)}(\X; \gamma) \right)
        \cdot
        \left(
          L(p_\mtype(\X; \gamma), s(\X); \drloss)
          -
          L(p_\mtype(\X; \gamma), s^\star_\mtype(\X); \drloss)
        \right)
      \right] \\
      &= \expect_{\meas{P}_{\rm x}}\left[
        \left( \tilde{w}_\mtype(\X; \gamma) + \tilde{w}_\mtype^{(\expert)}(\X; \gamma) \right)
        \cdot
        \left(
          L(p_\mtype(\X; \gamma), s(\X); \drloss)
          -
          L(p_\mtype(\X; \gamma), \Phi(p_\mtype(\X; \gamma)); \drloss)
        \right)
      \right] \\
      &\geq \frac{\beta}{2} \cdot \expect_{\meas{P}_{\rm x}}\left[
        \left( \tilde{w}_\mtype(\X; \gamma) + \tilde{w}_\mtype^{(\expert)}(\X; \gamma) \right)
        \cdot
        \left(
          \Phi^{-1}(s(\X)) - p_\mtype(x; \gamma)
        \right)^2
      \right].
    \end{aligned}
  \end{equation}
  Furthermore, from the bounded loss $\ell \in [0, B]$ assumption, we have that
  \begin{equation}
    \tilde{w}_\mtype(\X; \gamma) + \tilde{w}_\mtype^{(\expert)}(\X; \gamma) \in \left[2 \exp\left( - \frac{B}{\gamma} \right), 2 \right].
  \end{equation}
  Thus we have
  \begin{equation}\label{eq:squared-quad-bound-equal-temp}
    \expect_{\meas{P}_{\rm x}}\left[
      \left(
        \Phi^{-1}(s(\X)) - p_\mtype(x; \gamma)
      \right)^2
    \right]
    \leq
    \frac{\exp(B / \gamma)}{\beta}
    \cdot
    \mathcal{E}_\mtype(s; \gamma, \gamma)
  \end{equation}

  On the other hand, we also seek to bound $\vert \alpha(x) \vert$. Notice that we can state $\vert \alpha(x) \vert$ in terms of the logit function (\ie, the inverse sigmoid $\sigma^{-1}$):
  \begin{align*}
    \vert \alpha(x) \vert
    &= \left \vert
    \expect_{\Y \sim \bpos(x)} \left[ \ell(x, \Y, h(x)) - \ell(x, \Y, h^{(\expert)}(x)) \right] - c
    \right \vert \\
    &= \gamma \cdot \left \vert
    \frac{\expect_{\Y \sim \bpos(x)} \left[ \ell(x, \Y, h(x)) - \ell(x, \Y, h^{(\expert)}(x)) \right]}{\gamma} - \frac{c}{\gamma}
    \right \vert \\
    &= \gamma \cdot \left \vert
    \sigma^{-1}\left( p_\mtype(x; \gamma) \right) - \sigma^{-1} \left( \sigma \left( - \frac{c}{\gamma} \right) \right)
    \right \vert
  \end{align*}
  where we are using \cref{eq:p-mtype-as-sigmoid}.

  Again by the boundedness of $\ell$, we have that $p_\mtype \in [\sigma(-B / \gamma), \sigma(B / \gamma)]$. Furthermore, one will note that we have that both
  \begin{equation}
    p_\mtype, \sigma \left( - \frac{c}{\gamma} \right) \in \left[ \sigma\left( -\frac{\max\{ B, c \}}{\gamma} \right), 1- \sigma\left( -\frac{\max\{ B, c \}}{\gamma} \right) \right].
  \end{equation}
  Thus, by the differentiability of the logit function $\sigma^{-1}$ and the mean-value-theorem, we have that
  \begin{align}
    \vert \alpha(x) \vert
    &\leq \frac{\gamma}{\sigma\left( -\frac{\max\{ B, c \}}{\gamma} \right) \left(1- \sigma\left( -\frac{\max\{ B, c \}}{\gamma} \right)\right)} \cdot \left \vert p_\mtype(x; \gamma) - \sigma\left(-\frac{c}{\gamma}\right) \right \vert \nonumber \\
    &= K_{\gamma, B, c} \cdot \left \vert p_\mtype(x; \gamma) - \tau_{c, \gamma} \right \vert. \label{eq:alpha-bound-with-mvt}
  \end{align}

  Combining \cref{lem:regret-to-deferral-decisions,eq:alpha-bound-with-mvt} yields
  \begin{align*}
    \mathcal{L}(h, r_{c, \gamma}; c) - \mathcal{L}(h, r^\star; c)
    &= \expect_{\X \sim \meas{P}_{\rm x}}\left[
      \vert \alpha(\X) \vert \cdot \left \llbracket
      r_{c, \gamma}(\X) \neq r^\star(\X; c, \meas{P})
      \right \rrbracket
    \right] \\
    &\leq K_{\gamma, B, c} \cdot
    \expect_{\X \sim \meas{P}_{\rm x}}\left[
      \left \vert p_\mtype(x; \gamma) - \tau_{c, \gamma}
      \right \vert
      \cdot \left \llbracket
      r_{c, \gamma}(\X) \neq r^\star(\X; c, \meas{P})
      \right \rrbracket
    \right].
  \end{align*}
  Finally, notice that as $r_{c, \gamma}(x) = \llbracket (\link^{-1} \circ s)(x) \leq \tau_{c,\gamma} \rrbracket$, whenever $r_{c, \gamma}$ and $r^\star$ disagree, then $p_\mtype(x)$ and $(\link^{-1} \circ s)(x)$ must lie on the opposite sides of $\tau_{c, \gamma}$ (see \cref{lem:marg-opt-rejector-equal-temp}). Hence for $x$ satisfying $r_{c, \gamma}(x) \neq r^\star(x; c, \meas{P})$, we have
  \begin{equation*}
    \left \vert p_\mtype(x; \gamma) - \tau_{c, \gamma} \right \vert \leq
    \left \vert p_\mtype(x; \gamma) - (\link^{-1} \circ s)(x) \right \vert.
  \end{equation*}
  Thus using \cref{eq:squared-quad-bound-equal-temp} and Cauchy--Schwartz, we have
  \begin{align*}
    &\mathcal{L}(h, r_{c, \gamma}; c) - \mathcal{L}(h, r^\star; c) \\
    &\leq K_{\gamma, B, c} \cdot
    \expect_{\X \sim \meas{P}_{\rm x}}\left[
      \left \vert p_\mtype(x; \gamma) - \tau_{c, \gamma}
      \right \vert
      \cdot \left \llbracket
      r_{c, \gamma}(\X) \neq r^\star(\X; c, \meas{P})
      \right \rrbracket
    \right] \\
    &\leq K_{\gamma, B, c} \cdot
    \expect_{\X \sim \meas{P}_{\rm x}}\left[
      \left \vert p_\mtype(x; \gamma) - (\link^{-1} \circ s)(x) \right \vert
      \cdot \left \llbracket
      r_{c, \gamma}(\X) \neq r^\star(\X; c, \meas{P})
      \right \rrbracket
    \right] \\
    &\leq K_{\gamma, B, c} \cdot
    \expect_{\X \sim \meas{P}_{\rm x}}\left[
      \left \vert p_\mtype(x; \gamma) - (\link^{-1} \circ s)(x) \right \vert
    \right] \\
    &\leq K_{\gamma, B, c} \cdot
    \sqrt{
      \expect_{\X \sim \meas{P}_{\rm x}}\left[
        \left ( p_\mtype(x; \gamma) - (\link^{-1} \circ s)(x) \right )^2
    \right]} \\
    &\leq K_{\gamma, B, c} \cdot
    \sqrt{
      \frac{\exp(B / \gamma)}{\beta}
      \cdot
      \mathcal{E}_\mtype(s; \gamma, \gamma)
    } \\
    &= K_{\gamma, B, c} \cdot
    \sqrt{
      \frac{\exp(B / \gamma)}{\beta}
      \cdot
      \left[
        \mathcal{E}_\jtype(s; \gamma, \gamma)
        +
        \mathcal{E}_\mtype(s; \gamma, \gamma)
        -
        \mathcal{E}_\jtype(s; \gamma, \gamma)
      \right]
    } \\
    &= K_{\gamma, B, c} \cdot
    \sqrt{
      \frac{\exp(B / \gamma)}{\beta}
      \cdot
      \left[
        \mathcal{E}_\jtype(s; \gamma, \gamma)
        +
        \mathcal{L}^\star_\jtype(\gamma, \gamma)
        -
        \mathcal{L}^\star_\mtype(\gamma, \gamma)
        -
        \mathcal{L}_\jtype(s; \gamma, \gamma)
        +
        \mathcal{L}_\mtype(s; \gamma, \gamma)
      \right]
    } \\
    &\leq K_{\gamma, B, c} \cdot
    \sqrt{
      \frac{\exp(B / \gamma)}{\beta}
      \cdot
      \left[
        \mathcal{E}_\jtype(s; \gamma, \gamma)
        +
        \Delta_{\mtype \jtype}(\gamma, \gamma)
      \right]
    }.
  \end{align*}
  The final inequality follows from \cref{lem:obj-jensen-ordering}.
\end{proof}

\subsection{Variance and Finite-Sample Bounds}
\label{app:variance-sample-complexity}

In the following, we consider different methods for controlling the general approximate consistency upper bound presented in \cref{thm:approx-consistency}. Although we do not explicitly state it, the sample complexity discussion directly applies to the results of \cref{thm:exact-consistency} (and its extended discussion in \cref{app:exact-consistent}).

\paragraph{Variance control of the approximation term.}
Under \cref{thm:approx-consistency}, we present a short discussion on how the marginal-joint ideal optimality gap can be upper bounded by a variance term. To do this, we leverage \cref{thm:joint-marginal-gap-formal}.

Under the assumptions of \cref{thm:general-consistency-formal}, let $s^\star_\mtype = \link \circ p_\mtype(\cdot; \gamma)$ be the marginal optimal scorer. Now short-hand its weighted variance:
\begin{equation}\label{eq:approx-consistency-weighted-variance}
  \scalemath{0.92}{%
    V^\star_\mtype \defeq \expect_{\X \sim \meas{P}_{\rm x}}\left[
      (\drloss_{+1} \circ s^\star_\mtype)(\X) \cdot \var_{\bpos(\X)}[\ell(\X, \Y, h(\X))]
      + (\drloss_{-1} \circ s^\star_\mtype)(\X) \cdot \var_{\bpos(\X)}[\ell(\X, \Y, h^{(\expert)}(\X))]
    \right].
  }
\end{equation}
Thus, \cref{thm:joint-marginal-gap-formal} gives
\begin{equation}\label{eq:approx-consistency-variance-gap}
  0 \leq \Delta_{\mtype\jtype}(\gamma, \gamma)
  \leq \mathcal{L}_\jtype(s^\star_\mtype; \gamma, \gamma) - \mathcal{L}_\mtype(s^\star_\mtype; \gamma, \gamma)
  \leq \frac{V^\star_\mtype}{2\gamma^2}.
\end{equation}
Note that the factor $1/\gamma^2$ can make the bound loose as $\gamma \to 0$. Thus, the low temperature connection in \cref{prop:limit_to_diff} does not establish near-optimality in that limit.

\paragraph{Finite-sample bounds.}
Through standard methods (\eg, Rademacher complexity), we can convert the consistency bound in \cref{thm:approx-consistency} (and \cref{thm:exact-consistency}) into finite sample bounds.

Fix the base model $h$, the expert $h^{(\expert)}$, and the temperatures (which may be unequal).
Let $\Z_i=(\X_i,\Y_i)$ for $i \in \{1,\ldots,n\}$ denote $n$ \iid samples from $\meas{P}$.
Denote a hypothesis class for the scorers by $\mathcal{S}$. We define
\begin{align*}
  g_s(x,y; \gamma, \gamma^{(\expert)})
  &\defeq \exp\left(-\frac{\ell(x,y,h(x))}{\gamma}\right) (\drloss_{+1}\circ s)(x)
  + \exp\left(-\frac{\ell(x,y,h^{(\expert)}(x))}{\gamma^{(\expert)}}\right) (\drloss_{-1}\circ s)(x), \\
  \mathcal{L}_\jtype(s; \gamma, \gamma^{(\expert)})
  &= \expect_{\Z\sim\meas{P}}[g_s(\Z; \gamma, \gamma^{(\expert)})],
  \qquad
  \widehat{\mathcal{L}}_\jtype(s; \gamma, \gamma^{(\expert)}) \defeq \frac{1}{n}\sum_{i=1}^n g_s(\Z_i; \gamma, \gamma^{(\expert)}).
\end{align*}
Define the approximation error, empirical optimization error, and uniform estimation errors, respectively, by
\begin{align*}
  A_{\mathcal{S}, \gamma, \gamma^{(\expert)}}
  &\defeq \inf_{s\in\mathcal{S}}\mathcal{L}_\jtype(s; \gamma, \gamma^{(\expert)}) - \inf_s\mathcal{L}_\jtype(s; \gamma, \gamma^{(\expert)}), \\
  \epsilon_{\rm opt}(\hat{s}; \gamma, \gamma^{(\expert)})
  &\defeq \widehat{\mathcal{L}}_\jtype(\widehat{s}; \gamma, \gamma^{(\expert)}) - \inf_{s\in\mathcal{S}}\widehat{\mathcal{L}}_\jtype(s; \gamma, \gamma^{(\expert)}),
  \qquad \widehat{s}\in\mathcal{S}, \\
  U_{n, \gamma, \gamma^{(\expert)}}
  &\defeq \sup_{s\in\mathcal{S}}\left\vert\widehat{\mathcal{L}}_\jtype(s; \gamma, \gamma^{(\expert)})-\mathcal{L}_\jtype(s; \gamma, \gamma^{(\expert)})\right\vert.
\end{align*}
Let $\mathcal{G}_\jtype=\{g_s \colon s\in\mathcal{S}\}$. We define the (absolute) Rademacher complexity~\citep[Definition 2]{bartlett2002rademacher}
\begin{equation}\label{eq:finite-sample-rademacher}
  \mathfrak{R}_n(\mathcal{G}_\jtype)
  \defeq \expect_{\Z_{1:n},\varepsilon_{1:n}}\left[
    \sup_{s\in\mathcal{S}}\left\vert\frac{1}{n}\sum_{i=1}^n\varepsilon_i g_s(\Z_i)\right\vert
  \right],
\end{equation}
where the $\varepsilon_i$ are independent uniform signs. Define $\mathfrak{R}_n(\mathcal{S})$ analogously using $s(\X_i)$.
All complexity constants below use the absolute Rademacher convention and $1/n$ normalization in \cref{eq:finite-sample-rademacher}.

We can now provide a standard excess risk upper bound consisting of the discussed errors.

\begin{proposition}\label{prop:finite-sample-joint-excess}
  Suppose that, for every $s\in\mathcal{S}$, $g_s(\Z)$ almost surely lies in a common interval of length $B_\jtype < \infty$. Then, for $\delta\in(0,1)$, with probability at least $1-\delta$,
  \begin{equation}\label{eq:finite-sample-joint-excess}
    \mathcal{E}_\jtype(\widehat{s}; \gamma, \gamma^{(\expert)})
    \leq A_{\mathcal{S}, \gamma, \gamma^{(\expert)}} + \epsilon_{\rm opt}(\hat{s}; \gamma, \gamma^{(\expert)})
    + 4\mathfrak{R}_n(\mathcal{G}_\jtype)
    + 2B_\jtype\sqrt{\frac{\log(2/\delta)}{2n}}.
  \end{equation}
\end{proposition}
\begin{proof}
  The proof is based on the strategy of~\citet[Theorem 8]{bartlett2002rademacher}.

  One can compute the standard upper bound on the excess risk as
  \begin{equation*}
    \mathcal{E}_\jtype(\widehat{s}; \gamma, \gamma^{(\expert)})
    \leq
    A_{\mathcal{S}, \gamma, \gamma^{(\expert)}} + 2U_{n, \gamma, \gamma^{(\expert)}} + \epsilon_{\rm opt}(\hat{s}; \gamma, \gamma^{(\expert)}).
  \end{equation*}
  The symmetrization and McDiarmid argument \citet[Theorem 8]{bartlett2002rademacher}, applied in both directions after rescaling the loss range to $[0, 1]$ yields
  \begin{equation*}
    U_{n, \gamma, \gamma^{(\expert)}} \leq 2\mathfrak{R}_n(\mathcal{G}_\jtype)
    + B_\jtype\sqrt{\frac{\log(2/\delta)}{2n}}.
  \end{equation*}
  Combining the bounds yields the result.
\end{proof}

For the common temperature setting of \cref{thm:approx-consistency}, denote $r_{\widehat{s},c,\gamma}(x)=\llbracket(\link^{-1}\circ\widehat{s})(x)\leq\tau_{c,\gamma}\rrbracket$.
Combining \cref{thm:approx-consistency,prop:finite-sample-joint-excess} yields
\begin{align}\label{eq:finite-sample-approximate-consistency}
  &\mathcal{L}(h,r_{\widehat{s},c,\gamma};c)-\mathcal{L}(h,r^\star;c)\\
  &\leq K_{\gamma,B,c}\sqrt{\frac{\exp(B/\gamma)}{\beta}
    \left[
      A_{\mathcal{S}, \gamma, \gamma} + \epsilon_{\rm opt}(\hat{s}; \gamma, \gamma)
      + 4\mathfrak{R}_n(\mathcal{G}_\jtype)
      + 2B_\jtype\sqrt{\frac{\log(2/\delta)}{2n}}
      +\Delta_{\mtype\jtype}(\gamma,\gamma)
  \right]}. \nonumber
\end{align}

\paragraph{Scorer complexity.}
One can further simplify the bound in \cref{eq:finite-sample-approximate-consistency} if additional assumptions are made over the scorer hypothesis class $s \in \mathcal{S}$.

Suppose $\vert s \vert \leq M$ uniformly over $s\in\mathcal{S}$. If $t\mapsto g_t(z)$ is $L_g$-Lipschitz on $[-M,M]$ and $g_0$ is bounded, we can use a contraction argument on $g_s-g_0$ to yield~\citep[Theorem 12(4)]{bartlett2002rademacher}
\begin{equation}\label{eq:finite-sample-contraction}
  \mathfrak{R}_n(\mathcal{G}_\jtype)
  \leq 2L_g\mathfrak{R}_n(\mathcal{S})
  + \sqrt{\frac{\expect[g_0(\Z; \gamma)^2]}{n}},
\end{equation}
where $g_t(z)$ denotes the per-example loss with scalar score $t$.

It should be noted that $\mathfrak{R}_n(\mathcal{S})$  for standard model classes will shrink as the number of samples increase. For example, under appropriate assumptions, ReLU MLP can be shown to have its Rademacher term shrink at a rate of $\bigoh(n^{-1/2})$~\citep[Theorem 1]{golowich2018size}. Under the square-root transfer of \cref{thm:approx-consistency} (and \cref{thm:exact-consistency}), this yields an overall rate of $\bigoh(n^{-1/4})$ for the corresponding term---a fast rate would require further assumptions.
If $A_{\mathcal{S}, \gamma, \gamma}=\epsilon_{\rm opt}(\hat{s}; \gamma, \gamma)=\Delta_{\mtype\jtype}=0$, \cref{eq:finite-sample-approximate-consistency} gives an $n^{-1/4}$ L2D regret rate and the sufficient sample scaling $n=\mathcal{O}(\epsilon^{-4}\log(2/\delta))$ for regret at most $\epsilon$.

\section{Exact Consistency of DR CPE Losses}%
\label{app:exact-consistent}

In the following, we prove the exact consistency result of \cref{thm:exact-consistency}. To do so, we will construct a mapping between deferral costs $c$ and temperatures $\gamma, \gamma^{(\expert)}$ which will result in an exact consistency bound. Throughout this section, we consider $\ell$ to be the zero-one loss function. We first give a bound for strongly proper composite DR losses and then specialize to the LSIF~\citep{kanamori2009least}.

\subsection{Choice of loss functions}

The LSIF loss function is described by the following partials~\citep{kanamori2009least}:
\begin{equation}\label{eq:lsif-partials}
  \drloss_{+1}(v) = -v;
  \qquad
  \text{and}
  \qquad
  \drloss_{-1}(v) = \frac{1}{2} v^2.
\end{equation}

For the zero-one loss function, we define the following quantities:
\begin{equation}
  p_h(x) \defeq \prob_{\meas{P}}[f(\X)=\Y \mid \X = x];
  \qquad
  \text{and}
  \qquad
  p_e(x) \defeq \prob_{\meas{P}}[f^{(\expert)}(\X)=\Y \mid \X = x].
\end{equation}
Notably, when considering a zero-one loss function for $\ell$, the following conditional risks simplify:
\begin{align*}
  \expect_{\Y \sim \bpos(x)} \left[ \ell(x, \Y, h(x)) \right] = 1 - p_h(x), \qquad
  \expect_{\Y \sim \bpos(x)} \left[ \ell(x, \Y, h^{(\expert)}(x)) \right] = 1 - p_e(x).
\end{align*}
As a result, we can simplify $\alpha$ in \cref{lem:regret-to-deferral-decisions} for the zero-one loss to
\begin{equation}\label{eq:zero-one-alpha-simplification}
  \alpha(x) = p_e(x) - p_h(x) - c.
\end{equation}
Furthermore, we can simplify the joint ideal weights when $\ell$ is the zero-one loss function:
\begin{equation}\label{eq:zero-one-loss-simplified-weights}
  \begin{aligned}
    \tilde{w}_\jtype(x; \gamma)
    &= \expect_{\Y \sim \bpos(x)}\left[ \exp\left( -\frac{\ell(x, \Y, h(x))}{\gamma} \right) \right] \\
    &= \kappa + (1 - \kappa) \cdot p_{h}(x) \\
    \\
    \tilde{w}_\jtype^{(\expert)}(x; \gamma^{(\expert)})
    &= \expect_{\Y \sim \bpos(x)}\left[ \exp\left( -\frac{\ell(x, \Y, h^{(\expert)}(x))}{\gamma^{(\expert)}} \right) \right] \\
    &= \kappa^{(\expert)} + (1 - \kappa^{(\expert)})  \cdot p_{e}(x),
  \end{aligned}
\end{equation}
where for convenience we shorthand
\begin{equation}
  \kappa \defeq \exp(-1/\gamma),
  \qquad \text{and} \qquad
  \kappa^{(\expert)} \defeq \exp(-1/\gamma^{(\expert)}).
\end{equation}

These choices of loss functions match that of \cref{prop:lsif}.

\subsection{Mapping between temperature and deferral costs}

In the following, we provide a mapping between temperatures $\gamma, \gamma^{(\expert)}$ and deferral costs $c$.

\begin{lemma}\label{lem:temp-mapping}
  Let $\gamma^{(\expert)} > 0$ and $c \geq 0$. There exists a unique $\gamma \geq \gamma^{(\expert)}$ such that
  \begin{equation}
    \frac{\kappa}{1 - \kappa} - \frac{\kappa^{(\expert)}}{1 - \kappa^{(\expert)}} = c.
  \end{equation}
\end{lemma}
\begin{proof}
  Letting $\frac{\exp(- 1/\gamma^{(\expert)})}{1 - \exp(- 1/\gamma^{(\expert)})} + c$ be constant and solving for $\gamma$ yields the desired result immediately.
\end{proof}

\begin{lemma}\label{lem:zero-one-loss-opt}
  Suppose that the temperature mapping of \cref{lem:temp-mapping} holds.
  Let $\ell$ be the zero-one loss function, $(\drloss, \link)$ be a strictly proper composite binary loss, and let $s^\star_\jtype$ minimize $\mathcal{L}_\jtype(s; \gamma, \gamma^{(\expert)})$.
  Then
  \begin{equation}\label{eq:zero-one-loss-opt-scorer}
    (\link^{-1} \circ s^\star_\jtype)(x) = \frac{\tilde{w}_\jtype(x; \gamma)}{\tilde{w}_\jtype(x; \gamma) + \tilde{w}_\jtype^{(\expert)}(x; \gamma^{(\expert)})}
  \end{equation}
  and the following thresholding rejector is equivalent to Chow's Rule
  \begin{equation}\label{eq:zero-one-loss-opt-rejector}
    \bar{r}^{\star}_{c, \gamma, \gamma^{(\expert)}}(x)
    \defeq
    \left\llbracket (\link^{-1} \circ s^\star_\jtype)(x) \leq \frac{\bar{\tau}_{c, \gamma, \gamma^{(\expert)}}}{1 + \bar{\tau}_{c, \gamma, \gamma^{(\expert)}}} \right\rrbracket
    =
    r^\star(x; c, \meas{P})
  \end{equation}
  where
  \begin{equation}
    \bar{\tau}_{c, \gamma, \gamma^{(\expert)}} \defeq \frac{1 - \kappa}{1 - \kappa^{(\expert)}}.
  \end{equation}
\end{lemma}
\begin{proof}
  We first prove \cref{eq:zero-one-loss-opt-scorer} before showing the Chow's Rule equivalence.
  We follow a similar proof to \cref{lem:marg-opt-scorer-equal-temp}. Simplifying the risk (using the superscript ``${\textup{zo}}$'' to denote the usage of the zero-one task loss $\ell = \ell^{\textup{zo}}$), we have
  \begin{align*}
    \mathcal{L}^{\textup{zo}}_\jtype(s; \gamma, \gamma^{(\expert)})
    &= \expect_{\meas{P}_{\rm x}}\left[\tilde{w}_\jtype(\X; \gamma) \cdot (\drloss_{+1} \circ s)(\X) + \tilde{w}^{(\expert)}_\jtype(\X; \gamma^{(\expert)}) \cdot (\drloss_{-1} \circ s)(\X)\right] \\
    &= \expect_{\meas{P}_{\rm x}}\left[ \left( \tilde{w}_\jtype(\X; \gamma) + \tilde{w}_\jtype^{(\expert)}(\X; \gamma^{(\expert)}) \right) \cdot L\left( \frac{\tilde{w}_\jtype(\X; \gamma)}{\tilde{w}_\jtype(\X; \gamma) + \tilde{w}_\jtype^{(\expert)}(\X; \gamma^{(\expert)})}, s(\X); \drloss \right) \right].
  \end{align*}
  By strict properness, \cref{eq:zero-one-loss-opt-scorer} yields the unique pointwise minimizer.

  Now we examine the Chow's Rule equivalence.
  Assuming that the temperature mapping of \cref{lem:temp-mapping} holds, we can rewrite the following relationship between the joint ideal weights and the threshold:
  \begin{align*}
    &\tilde{w}_\jtype(x; \gamma) - \bar{\tau}_{c, \gamma, \gamma^{(\expert)}} \cdot \tilde{w}_\jtype^{(\expert)}(x; \gamma^{(\expert)}) \\
    &= \kappa + ( 1 - \kappa ) \cdot p_{h}(x) - \bar{\tau}_{c, \gamma, \gamma^{(\expert)}} \cdot \left( \kappa^{(\expert)} + (1 - \kappa^{(\expert)})  \cdot p_{e}(x) \right) \\
    &= \kappa - (1 - \kappa) \cdot \frac{\kappa^{(\expert)}}{1 - \kappa^{(\expert)}} + (1 - \kappa) \cdot \left( p_{h}(x) - p_{e}(x) \right) \\
    &= (1 - \kappa) \left( \frac{\kappa}{1 - \kappa} - \frac{\kappa^{(\expert)}}{1 - \kappa^{(\expert)}} + p_{h}(x) - p_{e}(x) \right) \\
    &= (1 - \kappa) \left( c + p_{h}(x) - p_{e}(x) \right) \\
    &= - (1 - \kappa) \alpha(x),
  \end{align*}
  where we are using \cref{eq:zero-one-alpha-simplification} to identify $\alpha$.

  As the expert weights are strictly positive, we have
  \begin{equation*}
    \frac{\tilde{w}_\jtype(x; \gamma)}{\tilde{w}_\jtype^{(\expert)}(x; \gamma^{(\expert)})} \leq \bar{\tau}_{c, \gamma, \gamma^{(\expert)}} \iff - (1 - \kappa) \alpha(x) \leq 0 \iff \alpha(x) \geq 0.
  \end{equation*}
  Finally, we consider
  \begin{equation}\label{eq:zero-one-loss-weights-alpha-simplify}
    \begin{aligned}
      & \frac{\tilde{w}_\jtype(x; \gamma)}{\tilde{w}_\jtype(x; \gamma) + \tilde{w}_\jtype^{(\expert)}(x; \gamma^{(\expert)})} \leq \frac{\bar{\tau}_{c, \gamma, \gamma^{(\expert)}}}{1 + \bar{\tau}_{c, \gamma, \gamma^{(\expert)}}} \\
      \iff \quad &
      \tilde{w}_\jtype(x; \gamma) + \bar{\tau}_{c, \gamma, \gamma^{(\expert)}} \cdot \tilde{w}_\jtype(x; \gamma)
      \leq \bar{\tau}_{c, \gamma, \gamma^{(\expert)}} \cdot \tilde{w}_\jtype(x; \gamma) + \bar{\tau}_{c, \gamma, \gamma^{(\expert)}} \cdot \tilde{w}_\jtype^{(\expert)}(x; \gamma^{(\expert)}) \\
      \iff \quad &
      \tilde{w}_\jtype(x; \gamma)
      \leq \bar{\tau}_{c, \gamma, \gamma^{(\expert)}} \cdot \tilde{w}_\jtype^{(\expert)}(x; \gamma^{(\expert)}) \\
      \iff \quad &
      \frac{\tilde{w}_\jtype(x; \gamma)}{\tilde{w}_\jtype^{(\expert)}(x; \gamma^{(\expert)})}
      \leq \bar{\tau}_{c, \gamma, \gamma^{(\expert)}} \iff \alpha(x) \geq 0.
    \end{aligned}
  \end{equation}
  Combined with \cref{lem:chow}, we see that thresholding the scorer as per \cref{eq:zero-one-loss-opt-rejector} yields Chow's rule.
\end{proof}

\subsection{Proof of \texorpdfstring{\cref{thm:exact-consistency}}{Theorem 4.5} (Exact Consistency Bound)}

The following provides a proof of \cref{thm:exact-consistency} while also stating the formal statement. The exact relation between costs, temperatures, and thresholding parameters are explicitly given.

\begin{theorem}\label{thm:exact-consistency-formal}
  Suppose that the loss $\ell$ is the zero-one loss function, the composite loss $(\drloss, \link)$ is $\beta$-strongly proper with $\beta > 0$, and the temperatures $\gamma, \gamma^{(\expert)}$ and deferral cost $c \geq 0$ satisfy the condition of \cref{lem:temp-mapping}.
  Define the rejector
  \begin{equation}
    \bar{r}_{c, \gamma, \gamma^{(\expert)}}(x) = \left \llbracket (\link^{-1} \circ s)(x) \leq \frac{\bar{\tau}_{c, \gamma, \gamma^{(\expert)}}}{(1 + \bar{\tau}_{c, \gamma, \gamma^{(\expert)}})} \right\rrbracket.
  \end{equation}
  Then
  \begin{equation}
    \mathcal{L}^{\textup{zo}}(h, \bar{r}_{c, \gamma, \gamma^{(\expert)}}; c) - \mathcal{L}^{\textup{zo}}(h, r^\star; c)
    \leq \frac{2}{\sqrt{\beta}}\left( \frac{1}{1 - \kappa} + \frac{1}{1 - \kappa^{(\expert)}} \right) \sqrt{\mathcal{E}_\jtype^{\textup{zo}}(s; \gamma, \gamma^{(\expert)})},
  \end{equation}
  where the superscript ``${\textup{zo}}$'' is to indicate we are considering the zero-one loss task loss $\ell = \ell^{\textup{zo}}$.
\end{theorem}
\begin{proof}
  For convenience, let $r = \bar{r}_{c, \gamma, \gamma^{(\expert)}}$ and denote
  \begin{equation}
    p(x) \defeq \frac{\tilde{w}_\jtype(x; \gamma)}{\tilde{w}_\jtype(x; \gamma) + \tilde{w}_\jtype^{(\expert)}(x; \gamma^{(\expert)})}.
  \end{equation}

  From the proof of \cref{lem:zero-one-loss-opt} (and mirroring \cref{eq:approx-consistency-excess-simplify}), we have that
  \begin{align*}
    \mathcal{E}_\jtype^{\textup{zo}}(s; \gamma, \gamma^{(\expert)})
    &= \expect_{\meas{P}_{\rm x}}\Big[ \left( \tilde{w}_\jtype(\X; \gamma) + \tilde{w}_\jtype^{(\expert)}(\X; \gamma^{(\expert)}) \right) \\
    &\qquad\qquad\qquad\cdot \left( L\left( p(\X), s(\X); \drloss \right) - L\left( p(\X), \link(p(\X)); \drloss \right) \right) \Big] \\
    &\geq \frac{\beta}{2} \cdot \expect_{\meas{P}_{\rm x}}\left[ \left( \tilde{w}_\jtype(\X; \gamma) + \tilde{w}_\jtype^{(\expert)}(\X; \gamma^{(\expert)}) \right) \cdot \left( (\link^{-1} \circ s)(\X) - p(\X) \right)^2 \right],
  \end{align*}
  where we use strong properness.

  Further notice from a similar computation to \cref{eq:zero-one-loss-weights-alpha-simplify}, we have that
  \begin{align*}
    \left( \tilde{w}_\jtype(x; \gamma) + \tilde{w}_\jtype^{(\expert)}(x; \gamma^{(\expert)}) \right) \cdot \left( p(x) - \frac{\bar{\tau}_{c, \gamma, \gamma^{(\expert)}}}{1 + \bar{\tau}_{c, \gamma, \gamma^{(\expert)}}} \right)
    &=
    \frac{\tilde{w}_\jtype(x; \gamma) - \bar{\tau}_{c, \gamma, \gamma^{(\expert)}} \cdot \tilde{w}_\jtype^{(\expert)}(x; \gamma^{(\expert)})}{1 + \bar{\tau}_{c, \gamma, \gamma^{(\expert)}}} \\
    &= - \frac{1 - \kappa}{1 + \bar{\tau}_{c, \gamma, \gamma^{(\expert)}}} \alpha(x).
  \end{align*}

  Now, similar to the proof of \cref{thm:general-consistency-formal} we examine the values of $p(x)$ and $(\link^{-1} \circ s)(x)$ when $r^\star(x) \neq \bar{r}_{c, \gamma, \gamma^{(\expert)}}(x)$. Through \cref{lem:zero-one-loss-opt}, the rejectors only disagree when $p(x)$ and $(\link^{-1} \circ s)(x)$ lie on opposite sides of the threshold, hence we have $\vert p(x) - \bar{\tau}_{c, \gamma, \gamma^{(\expert)}} / (1 + \bar{\tau}_{c, \gamma, \gamma^{(\expert)}}) \vert \leq \vert p(x) - (\link^{-1} \circ s)(x) \vert$.

  Together, we have
  \begin{align}
    &\mathcal{L}^{\textup{zo}}(h, \bar{r}_{c, \gamma, \gamma^{(\expert)}}; c) - \mathcal{L}^{\textup{zo}}(h, r^\star; c) \nonumber \\
    &= \expect_{\X \sim \meas{P}_{\rm x}}\left[ \vert \alpha(\X) \vert \cdot \left \llbracket r(\X) \neq r^\star(\X; c, \meas{P}) \right \rrbracket \right] \nonumber \\
    &= \frac{1 + \bar{\tau}_{c, \gamma, \gamma^{(\expert)}}}{1 - \kappa} \cdot \expect_{\X \sim \meas{P}_{\rm x}}\Big[ \left( \tilde{w}_\jtype(\X; \gamma) + \tilde{w}_\jtype^{(\expert)}(\X; \gamma^{(\expert)}) \right) \nonumber \\
    &\qquad\qquad\qquad\cdot \left \vert p(\X) - \frac{\bar{\tau}_{c, \gamma, \gamma^{(\expert)}}}{1 + \bar{\tau}_{c, \gamma, \gamma^{(\expert)}}} \right \vert \cdot \left \llbracket r(\X) \neq r^\star(\X; c, \meas{P}) \right \rrbracket \Big] \nonumber \\
    &\leq \frac{1 + \bar{\tau}_{c, \gamma, \gamma^{(\expert)}}}{1 - \kappa} \cdot \expect_{\X \sim \meas{P}_{\rm x}}\left[ \left( \tilde{w}_\jtype(\X; \gamma) + \tilde{w}_\jtype^{(\expert)}(\X; \gamma^{(\expert)}) \right) \cdot \left \vert p(\X) - (\link^{-1} \circ s)(\X) \right \vert \right] \nonumber \\
    &\leq \frac{1 + \bar{\tau}_{c, \gamma, \gamma^{(\expert)}}}{1 - \kappa} \cdot \sqrt{\frac{2}{\beta} \cdot \expect_{\X \sim \meas{P}_{\rm x}}\left[ \tilde{w}_\jtype(\X; \gamma) + \tilde{w}_\jtype^{(\expert)}(\X; \gamma^{(\expert)}) \right] \cdot \mathcal{E}_\jtype^{\textup{zo}}(s; \gamma, \gamma^{(\expert)})} \nonumber \\
    &\leq \frac{1 + \bar{\tau}_{c, \gamma, \gamma^{(\expert)}}}{1 - \kappa} \cdot \sqrt{\frac{4}{\beta} \mathcal{E}_\jtype^{\textup{zo}}(s; \gamma, \gamma^{(\expert)})}. \nonumber
  \end{align}
  We use Cauchy--Schwarz and strong-properness.
  The final statement follows from a simplification of $\frac{1 + \bar{\tau}_{c, \gamma, \gamma^{(\expert)}}}{1 - \kappa}$.
\end{proof}

We can obtain a sharper bound without using strong properness by considering the LSIF loss function.

\begin{corollary}\label{cor:exact-consistency-lsif}
  Suppose that the loss $\ell$ is the zero-one loss function, $\drloss$ is the LSIF loss in \cref{eq:lsif-partials}, and the temperatures $\gamma, \gamma^{(\expert)}$ and deferral cost $c \geq 0$ satisfy the condition of \cref{lem:temp-mapping}. Define the rejector
  \begin{equation}\label{eq:lsif-exact-rejector}
    \bar{r}_{c, \gamma, \gamma^{(\expert)}}(x)
    \defeq \left\llbracket s(x) \leq \bar{\tau}_{c, \gamma, \gamma^{(\expert)}} \right\rrbracket.
  \end{equation}
  Then,
  \begin{equation}\label{eq:lsif-exact-consistency-bound}
    \mathcal{L}^{\textup{zo}}(h, \bar{r}_{c, \gamma, \gamma^{(\expert)}}; c) - \mathcal{L}^{\textup{zo}}(h, r^\star; c)
    \leq \frac{\sqrt{2}}{1 - \exp(-1 / \gamma)} \sqrt{\mathcal{E}_\jtype^{\lsif}(s; \gamma, \gamma^{(\expert)})},
  \end{equation}
  where the superscript ``$\lsif$'' is to indicate we are considering the LSIF density-ratio loss.
\end{corollary}
\begin{proof}
  Given the LSIF partial losses in \cref{eq:lsif-partials}, we have that
  \begin{equation*}
    \mathcal{L}_\jtype^{\lsif}(s; \gamma, \gamma^{(\expert)})
    = \expect_{\X \sim \meas{P}_{\rm x}}\left[
      -\tilde{w}_\jtype(\X; \gamma) \cdot s(\X)
      + \frac{1}{2} \tilde{w}_\jtype^{(\expert)}(\X; \gamma^{(\expert)}) \cdot s(\X)^2
    \right].
  \end{equation*}
  Completing the square yields the unique pointwise optimal scorer:
  \begin{equation*}
    s^\star_\jtype(x)
    = \frac{\tilde{w}_\jtype(x; \gamma)}{\tilde{w}_\jtype^{(\expert)}(x; \gamma^{(\expert)})}.
  \end{equation*}
  Thus we have the excess risk
  \begin{align}\label{eq:lsif-excess-simplify}
    \mathcal{E}_\jtype^{\lsif}(s; \gamma, \gamma^{(\expert)})
    &= \mathcal{L}_\jtype^{\lsif}(s; \gamma, \gamma^{(\expert)}) - \mathcal{L}_\jtype^{\lsif}(s^\star_\jtype; \gamma, \gamma^{(\expert)}) \nonumber \\
    &= \frac{1}{2} \cdot \expect_{\X \sim \meas{P}_{\rm x}}\left[
      \tilde{w}_\jtype^{(\expert)}(\X; \gamma^{(\expert)}) \cdot \left( s(\X) - s^\star_\jtype(\X) \right)^2
    \right].
  \end{align}

  From the proof of \cref{lem:zero-one-loss-opt}, we have
  \begin{align*}
    \tilde{w}_\jtype^{(\expert)}(x; \gamma^{(\expert)}) \cdot \left( s^\star_\jtype(x) - \bar{\tau}_{c, \gamma, \gamma^{(\expert)}} \right)
    &= \tilde{w}_\jtype(x; \gamma) - \bar{\tau}_{c, \gamma, \gamma^{(\expert)}} \cdot \tilde{w}_\jtype^{(\expert)}(x; \gamma^{(\expert)}) \\
    &= -(1 - \kappa) \alpha(x).
  \end{align*}
  Thus, thresholding $s^\star_\jtype$ at $\bar{\tau}_{c, \gamma, \gamma^{(\expert)}}$ yields Chow's rule~\cref{lem:chow}. Whenever $\bar{r}_{c, \gamma, \gamma^{(\expert)}}(x) \neq r^\star(x; c, \meas{P})$, the scores $s(x)$ and $s^\star_\jtype(x)$ must lie on the opposite sides of the threshold. Hence
  $
  \left\vert s^\star_\jtype(x) - \bar{\tau}_{c, \gamma, \gamma^{(\expert)}} \right\vert
  \leq \left\vert s(x) - s^\star_\jtype(x) \right\vert.
  $
  Applying \cref{lem:regret-to-deferral-decisions}, followed by Cauchy--Schwarz, and \cref{eq:lsif-excess-simplify}, we have that
  \begin{align*}
    &\mathcal{L}^{\textup{zo}}(h, \bar{r}_{c, \gamma, \gamma^{(\expert)}}; c) - \mathcal{L}^{\textup{zo}}(h, r^\star; c) \\
    &\leq \frac{1}{1 - \kappa} \cdot \expect_{\X \sim \meas{P}_{\rm x}}\left[
      \tilde{w}_\jtype^{(\expert)}(\X; \gamma^{(\expert)}) \cdot \left\vert s(\X) - s^\star_\jtype(\X) \right\vert
    \right] \\
    &\leq \frac{1}{1 - \kappa} \cdot \sqrt{
      \expect_{\X \sim \meas{P}_{\rm x}}\left[\tilde{w}_\jtype^{(\expert)}(\X; \gamma^{(\expert)})\right]
    } \cdot \sqrt{
      \expect_{\X \sim \meas{P}_{\rm x}}\left[
        \tilde{w}_\jtype^{(\expert)}(\X; \gamma^{(\expert)}) \cdot \left( s(\X) - s^\star_\jtype(\X) \right)^2
      \right]
    } \\
    &\leq \frac{\sqrt{2}}{1 - \kappa} \cdot \sqrt{\mathcal{E}_\jtype^{\lsif}(s; \gamma, \gamma^{(\expert)})}.
  \end{align*}
  As required.
\end{proof}

\section{Connections to Expert-Comparison Estimation}%
\label{app:connections}

The joint ideal DR CPE losses connect to various expert-comparison estimation approaches. In the following, we prove \cref{prop:limit_to_diff} and also consider a few different instantiations of $\ell$ and $\drloss$ which yields other interesting expert-comparison estimation approaches.

\begin{lemma}
  \label{lem:limit_exp_weights}
  Let $c \geq 0$. Then
  \begin{align*}
    \lim_{\gamma \to 0^+} \exp\left( -\frac{c}{\gamma} \right) &=
    \begin{cases}
      0 & c > 0, \\
      1 & c = 0.
    \end{cases}
    \intertext{ Furthermore, }
    \lim_{\gamma \to \infty} \exp\left( -\frac{c}{\gamma} \right) &= 1.
  \end{align*}
\end{lemma}
\begin{proof}
  If $c > 0$, then $c / \gamma \to \infty$ as $\gamma \to 0^+$, and therefore
  $-c / \gamma \to -\infty$. Since $\exp(x) \to 0$ as $x \to -\infty$, we obtain
  \begin{equation*}
    \lim_{\gamma \to 0^+} \exp\left( -\frac{c}{\gamma} \right) = 0.
  \end{equation*}
  If instead $c = 0$, then for every $\gamma > 0$ we have
  \begin{equation*}
    \exp\left( -\frac{c}{\gamma} \right) = \exp(0) = 1,
  \end{equation*}
  and hence
  \begin{equation*}
    \lim_{\gamma \to 0^+} \exp\left( -\frac{c}{\gamma} \right) = 1.
  \end{equation*}
  Finally, as $\gamma \to \infty$ we have $c / \gamma \to 0$, so by continuity of
  the exponential function,
  \begin{equation*}
    \lim_{\gamma \to \infty} \exp\left( -\frac{c}{\gamma} \right) = \exp(0) = 1.
  \end{equation*}
\end{proof}

\begin{corollary}
  \label{cor:square_limit}
  Suppose that $\ell$ the zero-one loss function. Then
  \begin{align*}
    \lim_{\gamma \to 0^+} \mathcal{L}_{\jtype}(s; \gamma, \gamma)
    = \expect_{\meas{P}}\left[
      \llbracket f(\X) = \Y \rrbracket \cdot (\drloss_{+1} \circ s)(\X)
      + \llbracket f^{(\expert)}(\X) = \Y \rrbracket \cdot (\drloss_{-1} \circ s)(\X)
    \right],
  \end{align*}
\end{corollary}
\begin{proof}
  As $\ell$ is the zero-one loss, for every $(x, y)$ we have
  \begin{equation*}
    \ell(x, y, h(x)) = 1 - \llbracket f(x) = y \rrbracket,
    \qquad
    \ell(x, y, h^{(\expert)}(x)) = 1 - \llbracket f^{(\expert)}(x) = y \rrbracket.
  \end{equation*}
  Hence, by the preceding \cref{lem:limit_exp_weights},
  \begin{align*}
    \lim_{\gamma \to 0^+} \exp\left( -\frac{\ell(\X, \Y, h(\X))}{\gamma} \right)
    &= \llbracket f(\X) = \Y \rrbracket, \\
    \lim_{\gamma \to 0^+} \exp\left( -\frac{\ell(\X, \Y, h^{(\expert)}(\X))}{\gamma} \right)
    &= \llbracket f^{(\expert)}(\X) = \Y \rrbracket.
  \end{align*}

  Therefore, from \cref{eq:joint-ideal-loss-simplify},
  \begin{align*}
    \lim_{\gamma \to 0^+} \mathcal{L}_{\jtype}(s; \gamma, \gamma)
    &= \lim_{\gamma \to 0^+}
    \expect_{\meas{P}}\left[
      \exp\left( -\frac{\ell(\X, \Y, h(\X))}{\gamma} \right) (\drloss{+1} \circ s)(\X) \right .\\
      &\hspace{4.8em} \left.
      + \exp\left( -\frac{\ell(\X, \Y, h^{(\expert)}(\X))}{\gamma} \right) (\drloss{-1} \circ s)(\X)
    \right] \\
    &= \expect_{\meas{P}}\left[
      \llbracket f(\X) = \Y \rrbracket \cdot (\drloss_{+1} \circ s)(\X)
      + \llbracket f^{(\expert)}(\X) = \Y \rrbracket \cdot (\drloss_{-1} \circ s)(\X)
    \right],
  \end{align*}
  where we used dominated convergence in the last step.
\end{proof}

\subsection{Proof of \texorpdfstring{\cref{prop:limit_to_diff}}{Proposition 4.6}}

\begin{ProofOf}{prop:limit_to_diff}
  For the squared loss, the partial losses are~\citep{menon2016}
  \begin{equation*}
    \drloss_{+1}(v) = (v - 1)^2,
    \qquad
    \drloss_{-1}(v) = (v + 1)^2.
  \end{equation*}
  Therefore, from \cref{cor:square_limit}
  \begin{align*}
    \lim_{\gamma \to 0^+} \mathcal{L}_{\jtype}(s; \gamma, \gamma)
    = \expect_{\meas{P}}\left[
      \llbracket f(\X) = \Y \rrbracket (s(\X) - 1)^2
      + \llbracket f^{(\expert)}(\X) = \Y \rrbracket (s(\X) + 1)^2
    \right].
  \end{align*}

  To identify the optimal scorer of the limiting objective, we condition on $\X = x$ and consider the pointwise conditional risk
  \begin{align*}
    R_x(v)
    &=
    \expect\left[
      \llbracket f(\X) = \Y \rrbracket (v - 1)^2
      +
      \llbracket f^{(\expert)}(\X) = \Y \rrbracket (v + 1)^2
      \mid \X = x
    \right] \\
    &= p_h(x) (v - 1)^2 + p_e(x) (v + 1)^2,
  \end{align*}
  where
  \begin{equation*}
    p_h(x) = \prob_{\meas{P}}[f(\X)=\Y \mid \X = x],
    \qquad
    p_e(x) = \prob_{\meas{P}}[f^{(\expert)}(\X)=\Y \mid \X = x].
  \end{equation*}
  Equivalently, with
  \begin{equation*}
    \Delta(x) = \expect[\Y^{({\rm d})} \mid \X = x] = p_h(x) - p_e(x),
    \qquad
    \Sigma(x) = p_h(x) + p_e(x),
  \end{equation*}
  we can write
  \begin{equation*}
    R_x(v) = \Sigma(x) v^2 - 2 \Delta(x) v + \Sigma(x).
  \end{equation*}
  Therefore, whenever $\Sigma(x) > 0$, the unique minimizer satisfies
  \begin{equation*}
    \frac{\partial R_x(v)}{\partial v}
    =
    2 \Sigma(x) v - 2 \Delta(x) = 0,
  \end{equation*}
  and hence
  \begin{equation*}
    s_0^\star(x) = \frac{\Delta(x)}{\Sigma(x)}.
  \end{equation*}
  If $\Sigma(x) = 0$, then $R_x(v) = 0$ for all $v$, so any value of $v$ is
  optimal. Finally, for any $\tau \in \mathbb{R}$ and any $x$ with
  $\Sigma(x) > 0$, we have
  \begin{equation*}
    s_0^\star(x) \le \tau
    \iff
    \Delta(x) \le \tau \, \Sigma(x),
  \end{equation*}
  as claimed.
\end{ProofOf}

\subsection{Ratio Expert Comparison}

One interesting variant of \cref{prop:limit_to_diff} is when we change the exact DRE loss $\drloss$ to other well-known DRE loss functions.

\begin{proposition}
  \label{prop:lsif}
  Suppose $\ell$ is the zero-one loss function and $\drloss$ is the LSIF~\citep{kanamori2009least} loss function.  
  Then the unique pointwise minimizer $s^\star_0(x)$ of the limiting risk satisfies, for any $\tau \in \mathbb{R}$,
  \begin{equation}
    \llbracket s^\star_0(x) \le \tau \rrbracket
    =
    \left \llbracket
        \prob_{\meas{P}}[f(\X)=\Y \mid \X = x] \leq \tau \prob_{\meas{P}}[f^{(\expert)}(\X)=\Y \mid \X = x]
    \right \rrbracket.
  \end{equation}
\end{proposition}
\begin{proof}
  For the LSIF loss function, the partial losses are ~\cite{menon2016},
  \begin{equation*}
    \drloss_{+1}(v) = -v,
    \qquad
    \drloss_{-1}(v) = \frac{1}{2} v^2.
  \end{equation*}
  Mirroring the proof of \cref{prop:limit_to_diff}, we have the pointwise condition risk
  \begin{equation*}
    R_x(v) = - p_h(x) v + \frac{1}{2} p_e(x) v^2.
  \end{equation*}
  Therefore, taking the minimizer, we have that
  \begin{equation*}
    \frac{\partial R_x(v)}{\partial v}
    =
    - p_h(x) + p_e(x) v = 0,
  \end{equation*}
  and hence,
  \begin{equation*}
    s^\star_0(x) = \frac{p_h(x)}{p_e(x)}.
  \end{equation*}
  The result follows immediately.
\end{proof}

\begin{proposition}
  \label{prop:kliep}
  Suppose $\ell$ is the zero-one loss function and $\drloss$ is the KLIEP~\citep{sugiyama2008} loss function.  
  Then the unique pointwise minimizer $s^\star_0(x)$ of the limiting risk satisfies, for any $\tau \in \mathbb{R}$,
  \begin{equation}
    \llbracket s^\star_0(x) \le \tau \rrbracket
    =
    \left \llbracket
        \prob_{\meas{P}}[f(\X)=\Y \mid \X = x] \leq \tau \prob_{\meas{P}}[f^{(\expert)}(\X)=\Y \mid \X = x]
    \right \rrbracket.
  \end{equation}
\end{proposition}
\begin{proof}
  For the KLIEP loss function, the partial losses are ~\cite{menon2016},
  \begin{equation*}
    \drloss_{+1}(v) = - \log v,
    \qquad
    \drloss_{-1}(v) = v
  \end{equation*}
  Mirroring the proof of \cref{prop:limit_to_diff}, we have the pointwise condition risk
  \begin{equation*}
    R_x(v) = - p_h(x) \log v + \frac{1}{2} p_e(x) v.
  \end{equation*}
  Therefore, taking the minimizer, we have that
  \begin{equation*}
    \frac{\partial R_x(v)}{\partial v}
    =
    - \frac{1}{v} p_h(x) + p_e(x) = 0,
  \end{equation*}
  and hence,
  \begin{equation*}
    s^\star_0(x) = \frac{p_h(x)}{p_e(x)}.
  \end{equation*}
  The result follows immediately.
\end{proof}

It is particularly interesting that two of (arguably) the most well-known DRE losses~\citep{sugiyama2008,kanamori2009least} yield the same optimal deferral mechanism, as per \cref{prop:kliep,prop:lsif}.
This is an interesting variant of the usual ``difference'' which is optimized for via \cref{eq:reg_delta_target}. The optimal scorer is also very intuitive: we want to reject whenever our base model's (conditional) accuracy is worse than what the expert would have (up to discount scaling $\tau$).

The logistic loss also yields a similar optimal deferral mechanism.

\begin{proposition}
  \label{prop:logistic}
  Suppose $\ell$ is the zero-one loss function and $\drloss$ is the logistic loss function.  
  Then the unique pointwise minimizer $s^\star_0(x)$ of the limiting risk satisfies, for any $\tau \in \mathbb{R}$,
  \begin{equation}
    \llbracket s^\star_0(x) \le \tau \rrbracket
    =
    \left \llbracket
        \prob_{\meas{P}}[f(\X)=\Y \mid \X = x] \leq \exp(\tau) \prob_{\meas{P}}[f^{(\expert)}(\X)=\Y \mid \X = x]
    \right \rrbracket.
  \end{equation}
\end{proposition}
\begin{proof}
  For the logistic loss function, the partial losses are ~\cite{menon2016},
  \begin{equation*}
    \drloss_{+1}(v) = \log(1 + \exp(-v)),
    \qquad
    \drloss_{-1}(v) = \log(1 + \exp(v))
  \end{equation*}
  Mirroring the proof of \cref{prop:limit_to_diff}, we have the pointwise condition risk
  \begin{equation*}
    R_x(v) = p_h(x) \log(1 + \exp(-v)) + p_e(x) \log(1 + \exp(v)).
  \end{equation*}
  Therefore, taking the minimizer, we have that
  \begin{align*}
    \frac{\partial R_x(v)}{\partial v}
    &=
    - \frac{\exp(-v)}{1+\exp(-v)} p_h(x) + \frac{\exp(v)}{1+\exp(v)} p_e(x) \\
    &=
    - \frac{1}{1+\exp(v)} p_h(x) + \frac{\exp(v)}{1+\exp(v)} p_e(x) \\
    &= 0.
  \end{align*}
  Thus we have
  \begin{equation*}
    - p_h(x) + \exp(v) p_e(x) = 0
  \end{equation*}
  and hence,
  \begin{equation*}
    s^\star_0(x) = \log \left(\frac{p_h(x)}{p_e(x)} \right).
  \end{equation*}
  The result follows immediately.
\end{proof}

The optimal deferral mechanisms are essentially the same ratio thresholds reparameterized when comparing \cref{prop:kliep,prop:lsif} versus \cref{prop:logistic}. Despite this, as in practice we are thresholding the scorers (as per \cref{lem:threshold-scorer}), KLIEP and LSIF may be more preferable: a small change in the scorer for the logistic loss (\cref{prop:logistic}) can cause a large shift in the decision \wrt the accuracy of base versus expert.

\section{Experiment Settings}%
\label{app:exp_settings}

The following section details the experimental setup in more details. All implementations were done using \texttt{PyTorch} for models and training; and \texttt{tensorflow\_datasets} (\texttt{tfds}) to manage datasets.
The CIFAR-100 and MedMNIST experiments were completed via a single NVIDIA Tesla V100.
The ImageNet experiments are described in \cref{app:imagenet}.
To produce error bars, we repeat deferral model training (keeping the base and expert models fixed) over 11 random seeds. The \texttt{Conf} baseline does not have error bars as it is derived directly from the base model's confidence.

\subsection{Base Datasets}

We consider three datasets:
\begin{itemize}[leftmargin=10pt, itemindent=0pt, parsep=0pt]
  \item CIFAR-100 (no license): The standard vision dataset consisting of 50,000 training images and 10,000 test images with 100 classes.
  \item PathMNIST (CC BY 4.0): A colon pathology image dataset consisting of 89,996 training images;  10,004 validation images; and 7,180 test images. The dataset has 9 class labels.
  \item DermaMNIST (CC BY-NC 4.0): A smaller dermatoscopy image dataset consisting of 7,007 training images; 1,003 validation images; and 2,005 test images. The dataset has 7 class labels.
  \item ImageNet (custom research license): A standard vision dataset consisting of 1,281,167 training images and 50,000 validation images with 1,000 classes.
\end{itemize}

As CIFAR-100 does not have an explicit validation set, we follow \citet{jitkrittum2023does} and generate a validation set consisting of a subselection of 5,000 examples from the training set. In particular, we subsample 50 examples for each of the 100 classes to yield 45,000 training examples and 5,000 validation examples.

\subsection{Corruptions}

We consider three types of corruptions:
\begin{itemize}[leftmargin=10pt, itemindent=0pt, parsep=0pt]
  \item Label noise: For $k \in \{ 10, 25 \}$, the first $k$ classes are randomly relabeled.
  \item Long-tail: For $h \in \{ 50, 25 \}$, we specify the first $h$ classes as head classes, and the remaining classes as tail classes. We then subsample the dataset with $500$ examples for the head class and $50$ examples for the tail class.
  \item Specialist: We specify a list of specialists class labels. All examples for the specialist class are taken. For all other classes, we subsample $p \in \{ 20, 10 \}$ percent of their examples.
\end{itemize}

In the label noise and long-tail settings, subsampling are applied to the training partitions (used by both $h$ and $h^{(\expert)}$).
In the specialist setting, $h$ is trained on the clean training partition, while only $h^{(\expert)}$ subsamples examples from non-specialist classes.
In each setting, the held-out partition used to train the deferral scorer and the test set remain clean.

For the specialist class labels, we use the following for each dataset:
\begin{itemize}[leftmargin=10pt, itemindent=0pt, parsep=0pt]
  \item CIFAR-100: The specialist is on mammals, which consists of 25/100 classes. This can be depicted by the following class label dictionary:
    \begin{small}
  \begin{verbatim}
MAMMAL_SPECIALIST_FINE_LABELS_BY_COARSE = {
  "aquatic_mammals": ("beaver", "dolphin", "otter", "seal", "whale"),
  "large_carnivores": ("bear", "leopard", "lion", "tiger", "wolf"),
  "large_omnivores_and_herbivores": (
    "camel",
    "cattle",
    "chimpanzee",
    "elephant",
    "kangaroo",
  ),
  "medium_mammals": ("fox", "porcupine", "possum", "raccoon", "skunk"),
  "small_mammals": ("hamster", "mouse", "rabbit", "shrew", "squirrel"),
}
  \end{verbatim}
    \end{small}
  \item PathMNIST: The specialist is on class labels associated with tumors, which consists of 3/9 classes. Its labels are given by:
    \begin{small}
  \begin{verbatim}
PATH_SPECIALIST = [
  "lymphocytes",
  "cancer-associated stroma",
  "colorectal adenocarcinoma epithelium"
]
  \end{verbatim}
    \end{small}

  \item DermaMNIST: The specialist is on class labels associated with melanocytic lesions, which consists of 2/7 classes. Its labels are given by:

    \begin{small}
  \begin{verbatim}
DERMA_SPECIALIST = [
  "melanoma",
  "melanocytic nevi",
]
  \end{verbatim}
    \end{small}
\end{itemize}

\subsection{Parameterized Deferral Models}

We follow \citet{jitkrittum2023does}'s post-hoc deferral-model setup for all learned deferral baselines. In all such methods, the base classifier $h$ and expert $h^{(\expert)}$ are first trained and then frozen; only a separate deferral scorer is fit on the clean deferral-training split.

The parameterized deferral methods learn a lightweight MLP that takes features derived solely from the base model's predictive probability vector $\pos(x) \in \Delta(\mathcal{Y})$. Following \citet{jitkrittum2023does}, we use a feature map $v(\pos(x)) \in \mathbb{R}^{\vert \mathcal{Y} \vert +11}$ consisting of:

\begin{enumerate}
  \item The entropy of $\pos(x)$;
  \item The top-10 predicted probabilities; and
  \item A one-hot encoding of the predicted class $\argmax_{y'} \pos_{y'}(x)$.
\end{enumerate}

The learned deferral mechanism uses an architecture of
\begin{equation*}
  g(x) = (\mathrm{FC}_1 \circ \mathrm{FC}_{2^4,\mathrm{ReLU}} \circ \mathrm{FC}_{2^6,\mathrm{ReLU}} \circ v \circ \pos)(x),
\end{equation*}
where $\mathrm{FC}_k$ denotes a fully connected layer with $k$ outputs. A rejector then is derived from a thresholding on $g$ (which depends on the approach).

\subsection{Optimizers}

We use a combination of SGD and Adam~\citep{kingma2014adam} optimizers.

\begin{itemize}[leftmargin=10pt, itemindent=0pt, parsep=0pt]
  \item {Base models}
    \begin{itemize}
      \item All base classifiers are trained with {SGD} using momentum $0.9$ and weight decay $10^{-4}$.
      \item {CIFAR-100}: learning rate $1.0$, with linear warmup for 15 epochs, then multi-step decay at epochs \{96, 192, 224\}. Models are trained for $256$ epochs.
      \item {PathMNIST} and {DermaMNIST}: learning rate $0.1$, with linear warmup for 5 epochs, then multi-step decay at epochs \{60, 90, 110\}.  Models are trained for $128$ epochs.
      \item {ImageNet}: MobileNetV2~\citep{sandler2018mobilenetv2} uses a learning rate of $0.05$, with linear warmup for 5 epochs, then multi-step decay at epochs \{30, 60, 80\}; EfficientNet-B0~\citep{tan2019efficientnet} uses a learning rate of $0.1$ with cosine annealing over $128$ epochs. Both models are trained for $128$ epochs.
    \end{itemize}

  \item {Deferral models}
    \begin{itemize}
      \item All {learned} deferral models use {Adam}.
      \item The optimizer settings are consistent across methods: learning rate $7 \times 10^{-4}$ and weight decay $10^{-3}$.
      \item Deferral models are trained for $32$ epochs.
    \end{itemize}
\end{itemize}

\subsection{Evaluation}%
\label{app:exp_evaluation}

To produce the curves in \cref{sec:experiments}, we sort the decision scores and defer the first $k$ examples, with smaller scores favoring deferral. Ties retain the evaluation order through a stable sort, and tabulated values are interpolated at the specified rates. Each curve keeps the scorer and its training hyperparameters fixed. The \texttt{Random} baseline is the expected accuracy under uniform random routing, \ie, the linear interpolation between the base and expert accuracies.

\subsection{Baselines}

We go into more detail regarding the baselines we consider. %

\paragraph{Confidence-Based Deferral}
Perhaps the simplest approach to deferral is to utilize the base model $h$'s perceived confidence of its prediction.
At inference time, one can simply threshold this confidence (\texttt{Conf}):
\begin{equation}
  \textrm{Conf}(x) = \max_{y} h_y(x),
\end{equation}
that is, $\llbracket \textrm{Conf}(x) \leq \tau \rrbracket$ for some threshold $\tau \in [0, 1]$.

\paragraph{Expert-Comparison Estimation}

These approaches are based on simple prediction targets (typically trained via regression, \ie, squared loss functions)~\citep{kag2022efficient,jitkrittum2023does}. As a result, we end up with a model $g$ that approximates a constructed random variable. Deferral can then be established by thresholding $g$, \ie, $r(x) = \llbracket d(g(x)) \leq \tau \rrbracket$ for some $\tau$ which we can tune and $d$ is a decision function.

\begin{itemize}[leftmargin=10pt, itemindent=0pt, parsep=0pt]
  \item \textbf{Difference in zero-one accuracy} (\texttt{Estimate-Diff01}):
    This baseline considers the difference in zero-one accuracy. This is exactly the target discussed in \cref{eq:reg_delta_target}. For completeness, we repeat the target random variable that is defined:
    \begin{equation*}
      \Y^{({\rm d})}
      \defeq
      \llbracket f(\X)=\Y \rrbracket
      -
      \llbracket f^{(\expert)}(\X)=\Y \rrbracket.
    \end{equation*}
    Thus, negative values correspond to examples on which the expert is correct and the base model is not.
    The decision function in this case is just the identity $d(z) = z$.

  \item \textbf{Expert maximum confidence} (\texttt{Estimate-MaxProb})
    This baseline regresses the expert's maximum predicted class probability:
    \begin{equation}
      \label{eq:expert_conf_target}
      \P^{({\rm m})}
      \defeq
      \max_{y \in \mathcal{Y}} \, \pos^{(\expert)}_{y}(\X).
    \end{equation}
    At test time, the learned score is compared against the base model's maximum confidence,
    \begin{equation}
      d(g(x)) = \max_{y \in \mathcal{Y}} \, \pos_{y}(x) - g(x),
    \end{equation}
    where we remind that $\pos$ is the base model and $g(x)$ approximates $\max_{y \in \mathcal{Y}} \, \pos^{(\expert)}_{y}(\X)$ via \cref{eq:expert_conf_target}.
    Thus, deferral is favored when the expert appears more confident than the base model.

\end{itemize}

\paragraph{Two-Stage Surrogate Losses}
We adapt \citet{mao2023} for single expert deferral. In particular, we utilize their modified exponential loss function, which is given by (\texttt{TwoStage-Exp}):
\begin{equation}
  \begin{aligned}
    \mathcal{L}_{\texttt{two-stage}}(s; c) =
    \expect_{\meas{P}} \Bigg[
      &\llbracket f(\X) = \Y \rrbracket \exp\left( s(\X) - \max_{y} h_y(\X) \right) \\
      +
      (
      &\llbracket f^{(\expert)}(\X) = \Y \rrbracket - c) \exp\left( \max_{y} h_y(\X) - s(\X) \right)
    \Bigg].
  \end{aligned}
\end{equation}
Typically, one defers from thresholding $s$ to the base model's confidence, \ie, $\llbracket s(x) > \max_{y} h_y(x) \rrbracket$. We consider a modification by thresholding over additional values to create a deferral curve $\llbracket s(x) > \max_{y} h_y(x) + \tau \rrbracket$ for $\tau \in \mathbb{R}$. For the reported comparisons, we use $c=0.05$ and vary $\tau$ with the scorer fixed.

\section{Extended Experiments}%
\label{app:extended_experiments}

In the following, we present an extension of the experiments presented in \cref{sec:experiments}.
For corrupted datasets, in this appendix section, we signify the strength of the corruption via parenthesis, \eg, ``PathMNIST Specialist(20)'' retains $20\%$ of the non-specialist training examples for the expert, together with all specialist training examples.

\cref{tab:summary} presents a tabulated version of \cref{fig:tradeoff-summary}. We highlight approaches which are closer to random deferral (the linear interpolation between 0\% and 100\% deferral rates). As one can see, although DR CPE approaches may not be the best approach for all dataset and deferral combination, it is the most robust to failure (least red) for even its Prob01 approach.

We test over alternate $\ell$'s used to define the weights in \cref{eq:joint-ideal-loss-simplify}: we consider probabilistic zero-one loss; Generalized Cross Entropy (GCE)~\citep{zhang2018generalized}; and a top-$k$ ($k=5$) zero-one loss. These results are summarized in \cref{app:tradeoff-summary}.
From these plots, we see that for some loss functions, like the probabilistic zero-one loss, there are setting where a higher $\gamma = 1$ temperature can cause a reduction in performance, see PathMNIST Specialist(20). GCE has better performance in the $\gamma = 1$ setting, but it suffers from a significant drop in performance compared to the $\gamma = 0.5$ counterpart and also has a high variance (similar to Prob01 and top-$k$ variants for PathMNIST Specialist(20) with $\gamma = 0.5$).

We also consider changing the target evaluation criterion from accuracy to top-$k$ ($k=5$). These results are summarized in \cref{app:topk-summary}---we focus on CIFAR-100 for this scenario due to its larger number of ($\vert \mathcal{Y} \vert = 100$) class labels.
An interesting finding of these results is that performance in accuracy does not necessarily result in the corresponding top-$k$ performance, nor vice-versa. Some of the most dramatic changes here are the performance of the \texttt{TwoStage-Exp} approach, which does much better in top-$k$.
This is even when compared against DR CPE with objective specified to top-$k$ accuracy. Despite this, the top-$k$ variant of DR CPE performs well against its GCE variant for specialist CIFAR-100 settings.

Finally, we consider a variation of \cref{fig:tradeoff-summary}, where we vary the density-ratio loss $\drloss$ utilized. From \cref{prop:kliep,prop:lsif,prop:logistic}, the Bayes-optimal scorer (and thus deferral curves) should be the same. However, in practice we find that the logistic loss performs consistently worse when compared to other loss functions and suffers from an extremely high variance. The squared loss seems to be the best performing with LSIF and KLIEP performing similarly.

\begin{table}[t]
\caption{Tabulated deferral accuracy results of \cref{fig:tradeoff-summary}. Bolded entries are the best performing approaches if they lie outside of the 1 standard deviation ranges of all other approaches (of different types). Red entries are approaches which are closer to random than the best performing approach.}
\label{tab:summary}
\centering
\small
\setlength{\tabcolsep}{4pt}
\renewcommand{\arraystretch}{1.12}
\begin{adjustbox}{max width=\textwidth}
\begin{tabular}{llccccccc}
\toprule
Dataset & Method & 5\% & 10\% & 15\% & 20\% & 25\% & 50\% & 75\% \\
\midrule
\multirow{7}{*}{\rotatebox[origin=c]{90}{\begin{tabular}{@{}c@{}}CIFAR-100\\Clean\\\texttt{[14-56]}\end{tabular}}} & Conf & $66.14 \pm 0.00$ & $66.67 \pm 0.00$ & $67.58 \pm 0.00$ & $\mathbf{68.21 \pm 0.00}$ & $68.47 \pm 0.00$ & $70.16 \pm 0.00$ & $69.64 \pm 0.00$ \\
 & TwoStage-Exp & $\textcolor{red}{65.71 \pm 0.02}$ & $66.45 \pm 0.02$ & $66.93 \pm 0.02$ & $67.14 \pm 0.02$ & $\textcolor{red}{67.39 \pm 0.03}$ & $\textcolor{red}{68.41 \pm 0.03}$ & $69.18 \pm 0.01$ \\
 & Estimate-Diff01 & $66.12 \pm 0.07$ & $66.79 \pm 0.07$ & $67.38 \pm 0.10$ & $67.98 \pm 0.10$ & $68.50 \pm 0.09$ & $69.93 \pm 0.08$ & $69.72 \pm 0.07$ \\
 & Estimate-MaxProb & $\mathbf{66.22 \pm 0.04}$ & $\mathbf{66.98 \pm 0.05}$ & $67.56 \pm 0.04$ & $68.01 \pm 0.06$ & $68.53 \pm 0.07$ & $70.13 \pm 0.06$ & $69.79 \pm 0.09$ \\
 & DRCPE (GCE) & $66.01 \pm 0.03$ & $66.74 \pm 0.07$ & $67.51 \pm 0.06$ & $68.03 \pm 0.06$ & $68.54 \pm 0.03$ & $70.12 \pm 0.05$ & $69.84 \pm 0.05$ \\
 & DRCPE (Prob01) & $66.14 \pm 0.07$ & $66.84 \pm 0.05$ & $67.58 \pm 0.08$ & $68.07 \pm 0.07$ & $68.54 \pm 0.06$ & $70.07 \pm 0.05$ & $69.73 \pm 0.06$ \\
 & Random & $65.40 \pm 0.00$ & $65.61 \pm 0.00$ & $65.82 \pm 0.00$ & $66.03 \pm 0.00$ & $66.24 \pm 0.00$ & $67.29 \pm 0.00$ & $68.34 \pm 0.00$ \\
\midrule
\multirow{7}{*}{\rotatebox[origin=c]{90}{\begin{tabular}{@{}c@{}}CIFAR-100\\Clean\\\texttt{[8-32]}\end{tabular}}} & Conf & $61.54 \pm 0.00$ & $62.46 \pm 0.00$ & $63.37 \pm 0.00$ & $64.18 \pm 0.00$ & $65.11 \pm 0.00$ & $68.17 \pm 0.00$ & $68.18 \pm 0.00$ \\
 & TwoStage-Exp & $61.41 \pm 0.03$ & $\textcolor{red}{62.08 \pm 0.04}$ & $62.98 \pm 0.04$ & $63.55 \pm 0.07$ & $64.36 \pm 0.06$ & $66.74 \pm 0.08$ & $67.58 \pm 0.04$ \\
 & Estimate-Diff01 & $61.76 \pm 0.07$ & $62.82 \pm 0.09$ & $63.77 \pm 0.10$ & $64.56 \pm 0.12$ & $65.31 \pm 0.08$ & $67.91 \pm 0.12$ & $68.09 \pm 0.07$ \\
 & Estimate-MaxProb & $61.57 \pm 0.05$ & $62.62 \pm 0.08$ & $63.63 \pm 0.05$ & $64.30 \pm 0.08$ & $65.15 \pm 0.05$ & $68.24 \pm 0.06$ & $68.16 \pm 0.04$ \\
 & DRCPE (GCE) & $61.50 \pm 0.06$ & $62.39 \pm 0.07$ & $63.33 \pm 0.09$ & $64.31 \pm 0.06$ & $65.23 \pm 0.06$ & $68.18 \pm 0.06$ & $68.16 \pm 0.04$ \\
 & DRCPE (Prob01) & $61.70 \pm 0.08$ & $62.71 \pm 0.12$ & $63.63 \pm 0.06$ & $64.54 \pm 0.09$ & $65.43 \pm 0.07$ & $68.19 \pm 0.05$ & $68.10 \pm 0.04$ \\
 & Random & $60.99 \pm 0.00$ & $61.36 \pm 0.00$ & $61.72 \pm 0.00$ & $62.09 \pm 0.00$ & $62.45 \pm 0.00$ & $64.27 \pm 0.00$ & $66.09 \pm 0.00$ \\
\midrule
\multirow{7}{*}{\rotatebox[origin=c]{90}{\begin{tabular}{@{}c@{}}PathMNIST\\Clean\\\texttt{[8-14]}\end{tabular}}} & Conf & $86.02 \pm 0.00$ & $86.46 \pm 0.00$ & $86.73 \pm 0.00$ & $86.91 \pm 0.00$ & $87.03 \pm 0.00$ & $86.66 \pm 0.00$ & $86.36 \pm 0.00$ \\
 & TwoStage-Exp & $\textcolor{red}{85.29 \pm 0.08}$ & $85.92 \pm 0.08$ & $86.64 \pm 0.08$ & $\mathbf{87.10 \pm 0.05}$ & $\mathbf{87.31 \pm 0.03}$ & $\mathbf{87.05 \pm 0.03}$ & $86.59 \pm 0.03$ \\
 & Estimate-Diff01 & $85.98 \pm 0.06$ & $86.59 \pm 0.06$ & $86.74 \pm 0.03$ & $86.87 \pm 0.05$ & $87.06 \pm 0.04$ & $86.75 \pm 0.08$ & $86.57 \pm 0.13$ \\
 & Estimate-MaxProb & $\mathbf{86.05 \pm 0.02}$ & $86.47 \pm 0.01$ & $86.72 \pm 0.01$ & $86.94 \pm 0.01$ & $87.04 \pm 0.01$ & $86.84 \pm 0.05$ & $86.65 \pm 0.03$ \\
 & DRCPE (GCE) & $86.02 \pm 0.04$ & $86.60 \pm 0.05$ & $86.75 \pm 0.03$ & $86.93 \pm 0.03$ & $87.08 \pm 0.04$ & $86.85 \pm 0.08$ & $86.71 \pm 0.09$ \\
 & DRCPE (Prob01) & $86.00 \pm 0.05$ & $86.58 \pm 0.03$ & $86.75 \pm 0.03$ & $86.89 \pm 0.04$ & $87.06 \pm 0.04$ & $86.87 \pm 0.10$ & $86.72 \pm 0.07$ \\
 & Random & $85.05 \pm 0.00$ & $85.12 \pm 0.00$ & $85.18 \pm 0.00$ & $85.25 \pm 0.00$ & $85.32 \pm 0.00$ & $85.65 \pm 0.00$ & $85.98 \pm 0.00$ \\
\midrule
\multirow{7}{*}{\rotatebox[origin=c]{90}{\begin{tabular}{@{}c@{}}CIFAR-100\\Label-Noise\\\texttt{[14-56]}\end{tabular}}} & Conf & $\textcolor{red}{57.58 \pm 0.00}$ & $\textcolor{red}{58.08 \pm 0.00}$ & $58.74 \pm 0.00$ & $59.13 \pm 0.00$ & $59.76 \pm 0.00$ & $61.32 \pm 0.00$ & $60.60 \pm 0.00$ \\
 & TwoStage-Exp & $\textcolor{red}{57.36 \pm 0.02}$ & $\textcolor{red}{57.84 \pm 0.03}$ & $\textcolor{red}{58.25 \pm 0.04}$ & $\textcolor{red}{58.74 \pm 0.03}$ & $\textcolor{red}{58.98 \pm 0.06}$ & $60.05 \pm 0.05$ & $60.14 \pm 0.02$ \\
 & Estimate-Diff01 & $57.93 \pm 0.07$ & $58.69 \pm 0.10$ & $59.28 \pm 0.07$ & $59.80 \pm 0.06$ & $60.24 \pm 0.11$ & $61.23 \pm 0.06$ & $60.78 \pm 0.09$ \\
 & Estimate-MaxProb & $57.85 \pm 0.04$ & $58.41 \pm 0.05$ & $58.88 \pm 0.05$ & $59.36 \pm 0.05$ & $59.95 \pm 0.07$ & $\mathbf{61.41 \pm 0.06}$ & $60.85 \pm 0.07$ \\
 & DRCPE (GCE) & $57.85 \pm 0.09$ & $58.38 \pm 0.07$ & $58.92 \pm 0.07$ & $59.44 \pm 0.12$ & $59.98 \pm 0.09$ & $61.26 \pm 0.06$ & $60.82 \pm 0.06$ \\
 & DRCPE (Prob01) & $57.96 \pm 0.07$ & $58.64 \pm 0.09$ & $59.45 \pm 0.11$ & $\mathbf{60.01 \pm 0.06}$ & $\mathbf{60.47 \pm 0.11}$ & $61.31 \pm 0.06$ & $60.84 \pm 0.10$ \\
 & Random & $57.35 \pm 0.00$ & $57.48 \pm 0.00$ & $57.61 \pm 0.00$ & $57.73 \pm 0.00$ & $57.86 \pm 0.00$ & $58.51 \pm 0.00$ & $59.15 \pm 0.00$ \\
\midrule
\multirow{7}{*}{\rotatebox[origin=c]{90}{\begin{tabular}{@{}c@{}}CIFAR-100\\Long-Tail\\\texttt{[8-32]}\end{tabular}}} & Conf & $\textcolor{red}{40.32 \pm 0.00}$ & $\textcolor{red}{40.70 \pm 0.00}$ & $\textcolor{red}{41.15 \pm 0.00}$ & $\textcolor{red}{41.43 \pm 0.00}$ & $\textcolor{red}{41.84 \pm 0.00}$ & $43.91 \pm 0.00$ & $\textcolor{red}{44.31 \pm 0.00}$ \\
 & TwoStage-Exp & $\textcolor{red}{40.21 \pm 0.03}$ & $\textcolor{red}{40.70 \pm 0.01}$ & $\textcolor{red}{40.93 \pm 0.02}$ & $\textcolor{red}{41.24 \pm 0.05}$ & $\textcolor{red}{41.54 \pm 0.04}$ & $\textcolor{red}{42.87 \pm 0.05}$ & $\textcolor{red}{43.48 \pm 0.03}$ \\
 & Estimate-Diff01 & $40.83 \pm 0.06$ & $41.80 \pm 0.05$ & $42.60 \pm 0.08$ & $\mathbf{43.28 \pm 0.07}$ & $\mathbf{43.91 \pm 0.05}$ & $\mathbf{45.63 \pm 0.12}$ & $\mathbf{45.82 \pm 0.07}$ \\
 & Estimate-MaxProb & $\textcolor{red}{40.32 \pm 0.04}$ & $\textcolor{red}{40.80 \pm 0.05}$ & $\textcolor{red}{41.21 \pm 0.05}$ & $\textcolor{red}{41.60 \pm 0.05}$ & $\textcolor{red}{42.10 \pm 0.04}$ & $43.88 \pm 0.06$ & $\textcolor{red}{44.38 \pm 0.11}$ \\
 & DRCPE (GCE) & $40.44 \pm 0.05$ & $41.18 \pm 0.07$ & $41.99 \pm 0.14$ & $42.57 \pm 0.12$ & $43.00 \pm 0.10$ & $44.86 \pm 0.06$ & $45.18 \pm 0.04$ \\
 & DRCPE (Prob01) & $40.77 \pm 0.04$ & $41.76 \pm 0.06$ & $42.50 \pm 0.07$ & $43.09 \pm 0.07$ & $43.51 \pm 0.08$ & $45.08 \pm 0.07$ & $45.64 \pm 0.07$ \\
 & Random & $39.93 \pm 0.00$ & $40.15 \pm 0.00$ & $40.37 \pm 0.00$ & $40.60 \pm 0.00$ & $40.82 \pm 0.00$ & $41.92 \pm 0.00$ & $43.03 \pm 0.00$ \\
\midrule
\multirow{7}{*}{\rotatebox[origin=c]{90}{\begin{tabular}{@{}c@{}}PathMNIST\\Specialist\\\texttt{[8-14]}\end{tabular}}} & Conf & $\mathbf{85.96 \pm 0.00}$ & $86.49 \pm 0.00$ & $86.91 \pm 0.00$ & $87.10 \pm 0.00$ & $87.08 \pm 0.00$ & $85.95 \pm 0.00$ & $\textcolor{red}{84.83 \pm 0.00}$ \\
 & TwoStage-Exp & $85.57 \pm 0.06$ & $86.56 \pm 0.06$ & $\mathbf{87.30 \pm 0.07}$ & $\mathbf{87.74 \pm 0.08}$ & $\mathbf{87.64 \pm 0.07}$ & $86.02 \pm 0.12$ & $\textcolor{red}{84.32 \pm 0.06}$ \\
 & Estimate-Diff01 & $\textcolor{red}{84.99 \pm 0.00}$ & $\textcolor{red}{84.98 \pm 0.01}$ & $\textcolor{red}{84.98 \pm 0.01}$ & $\textcolor{red}{84.97 \pm 0.02}$ & $\textcolor{red}{84.94 \pm 0.09}$ & $\textcolor{red}{84.58 \pm 0.10}$ & $\textcolor{red}{83.83 \pm 0.25}$ \\
 & Estimate-MaxProb & $85.86 \pm 0.03$ & $86.53 \pm 0.03$ & $86.86 \pm 0.02$ & $87.08 \pm 0.05$ & $87.09 \pm 0.10$ & $87.05 \pm 0.13$ & $86.57 \pm 0.08$ \\
 & DRCPE (GCE) & $85.68 \pm 0.10$ & $86.38 \pm 0.05$ & $87.01 \pm 0.05$ & $87.05 \pm 0.06$ & $87.05 \pm 0.10$ & $86.96 \pm 0.14$ & $86.57 \pm 0.07$ \\
 & DRCPE (Prob01) & $\textcolor{red}{85.14 \pm 0.24}$ & $85.77 \pm 0.42$ & $86.22 \pm 0.65$ & $86.23 \pm 0.64$ & $86.23 \pm 0.63$ & $86.33 \pm 0.55$ & $86.40 \pm 0.15$ \\
 & Random & $84.90 \pm 0.00$ & $84.81 \pm 0.00$ & $84.72 \pm 0.00$ & $84.64 \pm 0.00$ & $84.55 \pm 0.00$ & $84.12 \pm 0.00$ & $83.68 \pm 0.00$ \\
\midrule
\multirow{7}{*}{\rotatebox[origin=c]{90}{\begin{tabular}{@{}c@{}}DermaMNIST\\Specialist\\\texttt{[8-8]}\end{tabular}}} & Conf & $76.01 \pm 0.00$ & $\textcolor{red}{75.71 \pm 0.00}$ & $76.11 \pm 0.00$ & $\textcolor{red}{75.76 \pm 0.00}$ & $\textcolor{red}{75.66 \pm 0.00}$ & $\textcolor{red}{73.32 \pm 0.00}$ & $\textcolor{red}{72.52 \pm 0.00}$ \\
 & TwoStage-Exp & $75.97 \pm 0.06$ & $\textcolor{red}{75.66 \pm 0.05}$ & $\textcolor{red}{74.94 \pm 0.22}$ & $\textcolor{red}{74.38 \pm 0.15}$ & $\textcolor{red}{74.14 \pm 0.17}$ & $\textcolor{red}{73.50 \pm 0.11}$ & $\textcolor{red}{72.62 \pm 0.06}$ \\
 & Estimate-Diff01 & $76.02 \pm 0.15$ & $76.32 \pm 0.21$ & $76.40 \pm 0.27$ & $76.50 \pm 0.27$ & $76.43 \pm 0.38$ & $76.40 \pm 0.36$ & $76.15 \pm 0.21$ \\
 & Estimate-MaxProb & $\textcolor{red}{75.81 \pm 0.06}$ & $76.00 \pm 0.12$ & $76.17 \pm 0.14$ & $76.37 \pm 0.09$ & $76.41 \pm 0.06$ & $75.54 \pm 0.07$ & $75.49 \pm 0.09$ \\
 & DRCPE (GCE) & $\textcolor{red}{75.72 \pm 0.07}$ & $\textcolor{red}{75.94 \pm 0.12}$ & $76.34 \pm 0.09$ & $76.52 \pm 0.10$ & $76.46 \pm 0.10$ & $76.10 \pm 0.08$ & $76.13 \pm 0.09$ \\
 & DRCPE (Prob01) & $\textcolor{red}{75.77 \pm 0.11}$ & $76.03 \pm 0.17$ & $76.47 \pm 0.07$ & $76.59 \pm 0.09$ & $76.45 \pm 0.08$ & $76.14 \pm 0.14$ & $76.12 \pm 0.14$ \\
 & Random & $75.78 \pm 0.00$ & $75.60 \pm 0.00$ & $75.42 \pm 0.00$ & $75.24 \pm 0.00$ & $75.06 \pm 0.00$ & $74.16 \pm 0.00$ & $73.27 \pm 0.00$ \\
\bottomrule
\end{tabular}
\end{adjustbox}
\end{table}

\begin{figure}[th]%
  \includegraphics[width=\textwidth]{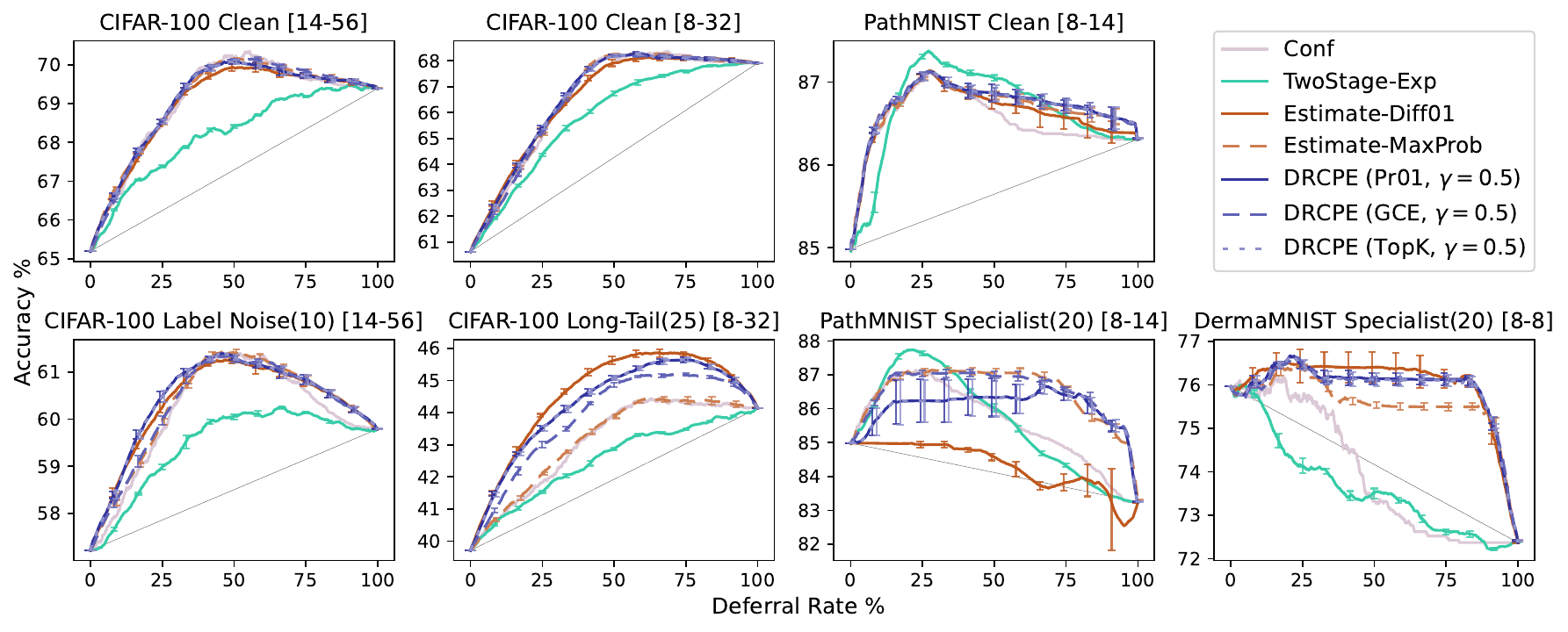}
  \includegraphics[width=\textwidth]{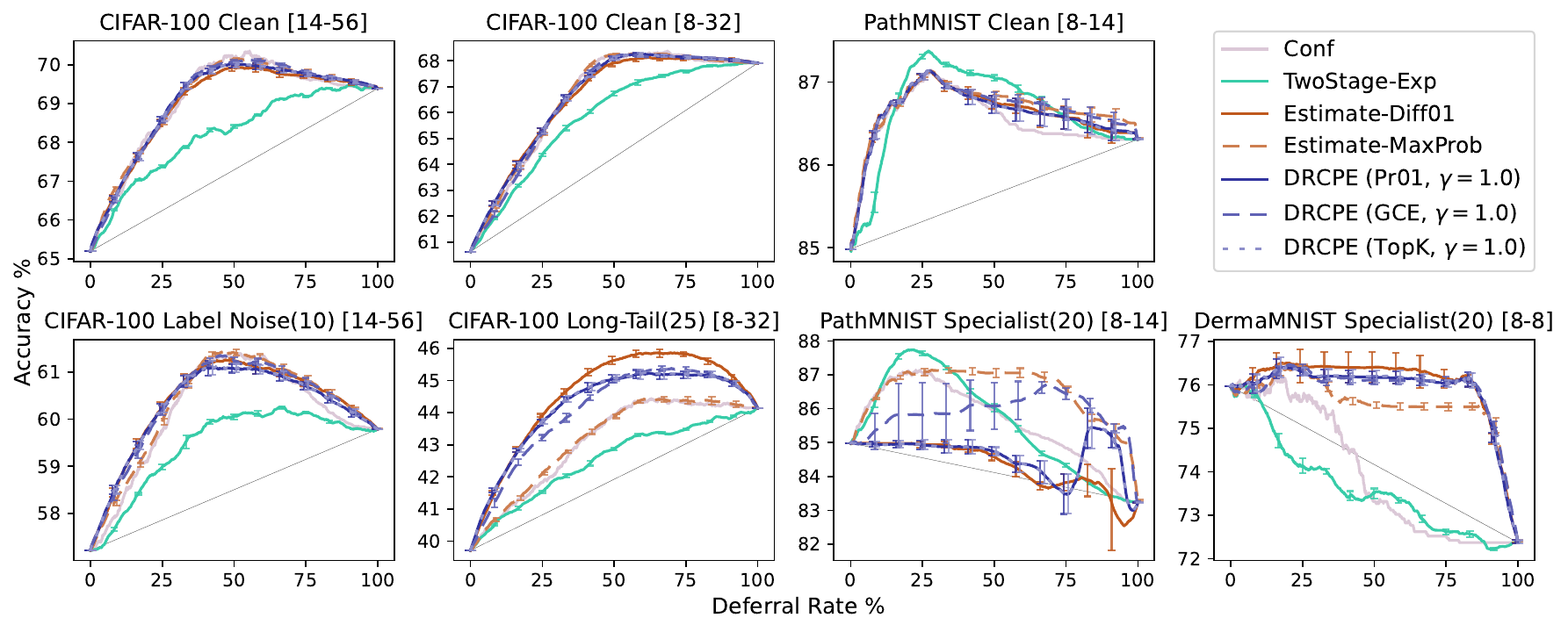}
  \caption{Accuracy-deferral trade-off plots over different $\ell$s and $\gamma \in \{ 0.5, 1.0 \}$.}%
  \label{app:tradeoff-summary}
\end{figure}

\begin{figure}[th]%
  \includegraphics[width=\textwidth]{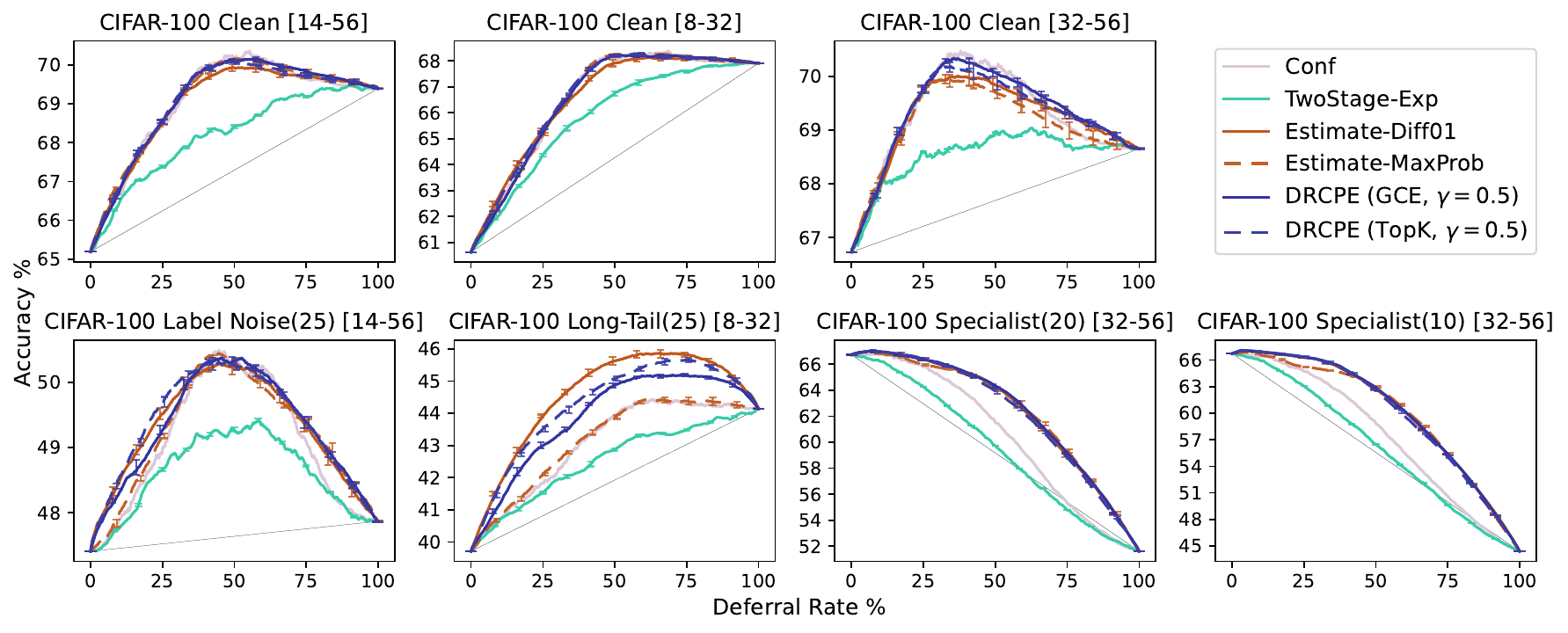}
  \includegraphics[width=\textwidth]{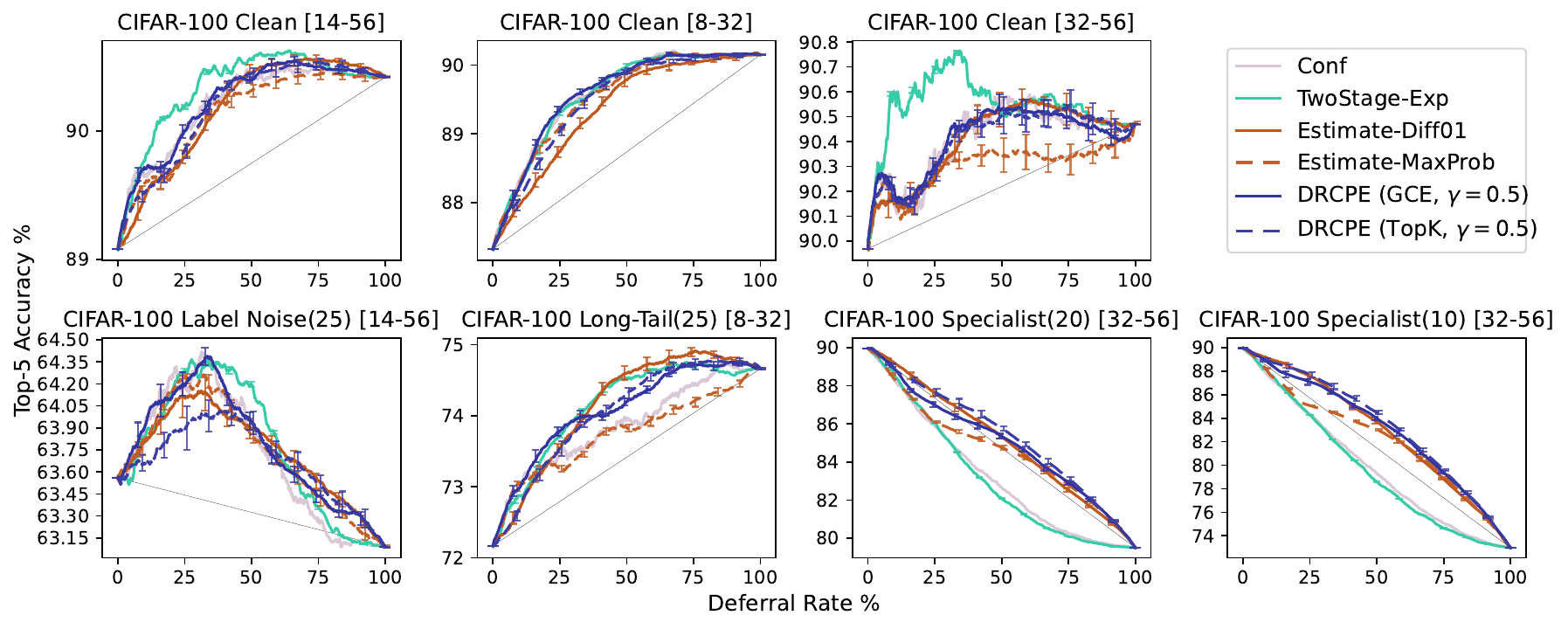}
  \caption{Accuracy-deferral trade-off plots over both accuracy and top-$k$ ($k=5$) performance.}%
  \label{app:topk-summary}
\end{figure}

\begin{figure}[th]%
  \includegraphics[width=\textwidth]{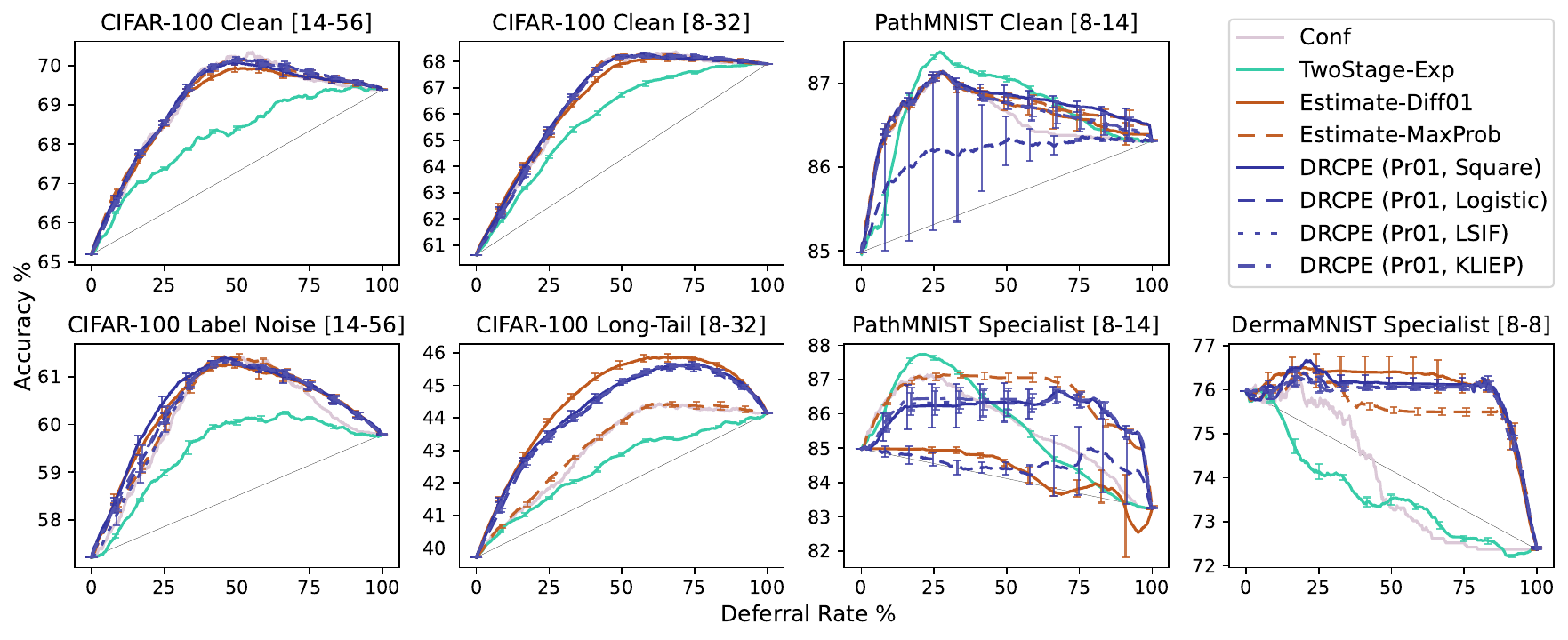}
  \includegraphics[width=\textwidth]{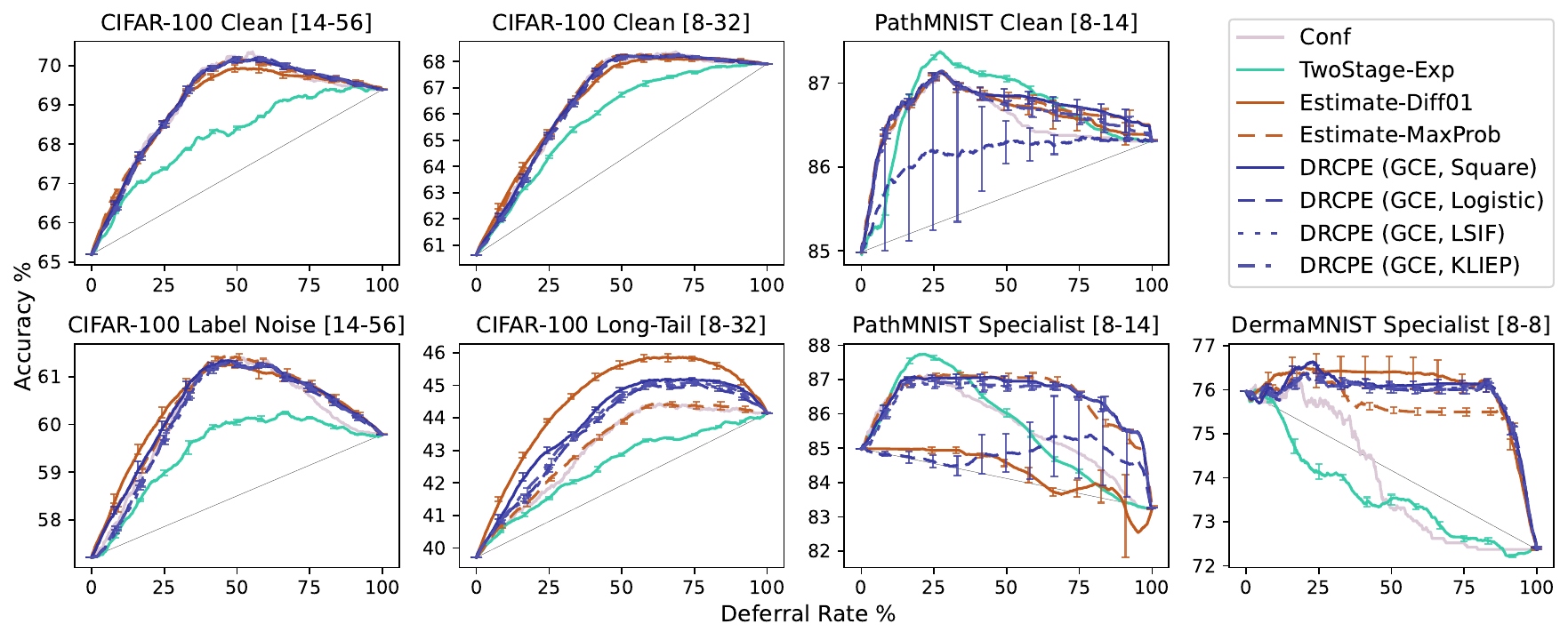}
  \caption{Accuracy-deferral trade-off plots over different DR partial losses (with fixed $\gamma = 0.5$).}%
  \label{app:partials-summary}
\end{figure}

\clearpage
\section{ImageNet Experiments}%
\label{app:imagenet}

In the following section, we consider the dog-specialist ImageNet setting examined in \citet{jitkrittum2023does}.
The base model $h$ is a MobileNetV2~\citep{sandler2018mobilenetv2} and the expert $h^{(\expert)}$ is an EfficientNet-B0~\citep{kag2022efficient}.
The base model is trained on the clean training split; while the expert trains on all dog examples and a percentage of $p \in \{ 100, 4, 2 \} \%$ non-dog training examples.
The deferral mechanism and test set are constructed using the validation set, with a 20:80 split.

\subsection{Unequal Temperatures}%
\label{app:imagenet-temps}
We consider an experimental setting where we consider unequal temperatures. In particular, we use a specific setting which theoretically match the exact consistency setting described in \cref{thm:exact-consistency}.
More specifically, we consider the LSIF~\citep{kanamori2009least} DR loss $\drloss$, which yields a nice consistency result as per \cref{cor:exact-consistency-lsif}.
To adhere to this setting, we pick pairs of temperatures $\gamma, \gamma^{(\expert)}$ (and costs) which corresponds to different deferral costs $c$ (as per the mapping lemma in \cref{lem:temp-mapping}): we fix $\gamma=0.5$ and vary $\gamma^{(\expert)}\in\{0.1,0.2,0.3,0.4,0.5\}$.

We consider some ``task losses'' $\ell$ which are related to the consistency setting of \cref{cor:exact-consistency-lsif}. We use the zero-one loss denoted as \texttt{DRCPE (Hard01)} (this correspond exactly to theory). We also use a probabilistic variant denoted by \texttt{DRCPE (Prob01)}---which uses the $L_2$ distance between one-hot labels and predicted probabilities. As also consider \texttt{DRCPE (Squared)} which is the main experiment's configuration (\cref{sec:experiments}) with equal temperatures $\gamma = \gamma^{(\expert)} = 0.5$ with squared DR loss $\drloss$ and \texttt{Prob01} task loss $\ell$.

\cref{tab:imagenet_100,tab:imagenet_4,tab:imagenet_2} summarize these unequal temperature experiments over all tested specialist percentages.
For the $100\%$ (non-specialist) regime, \texttt{Estimate-MaxProb} and \texttt{Estimate-Diff01} perform best or similarly to DR CPE.
For the other percentage dog-specialist percentages, \texttt{DRCPE (Squared)} (the main experiment's settings) consistently perform best; with the exception of the $4\%$ at $25\%$ deferral rate, where LSIF with \texttt{Prob01} and $(0.5,0.4)$ performs best.
In general, temperatures close to the equal temperature regime performs better and \texttt{Prob01} typically performs better than their \texttt{Hard01} counterparts for the non-$100\%$ setting (using the LSIF loss).

In general, these results show that despite \texttt{DRCPE Hard01} being the setting which has exact consistency bounds (\cref{app:exact-consistent}), in practice, for finite sample settings, deviation from theory can yield improvements.

\begin{table}[htbp]
  \centering
  \caption{%
  ImageNet dog-specialist setting with $p = 100\%$. Base/expert accuracies: $67.57\%$/$70.79\%$. Entries are accuracy (\%) means $\pm$ on standard deviation over $11$ random seeds. Bold are highest means.}
  \label{tab:imagenet_100}
  \small
  \setlength{\tabcolsep}{4pt}
  \begin{adjustbox}{max width=\textwidth}
    \begin{tabular}{@{}llcccc@{}}
      \toprule
      & & \multicolumn{4}{c}{Deferral rate (\%)} \\
      \cmidrule(l){3-6}
      Approach & $(\gamma,\gamma^{(\expert)})$ & $25$ & $50$ & $75$ & $90$ \\
      \midrule
      Random & -- & $68.38 \pm 0.00$ & $69.18 \pm 0.00$ & $69.99 \pm 0.00$ & $70.47 \pm 0.00$ \\
      Conf & -- & $70.01 \pm 0.00$ & $70.92 \pm 0.00$ & $70.84 \pm 0.00$ & $70.80 \pm 0.00$ \\
      TwoStage-Exp & -- & $69.37 \pm 0.01$ & $70.33 \pm 0.02$ & $70.65 \pm 0.01$ & $70.76 \pm 0.00$ \\
      Estimate-Diff01 & -- & $70.17 \pm 0.11$ & $\mathbf{71.01 \pm 0.05}$ & $70.88 \pm 0.04$ & $\mathbf{70.82 \pm 0.02}$ \\
      Estimate-MaxProb & -- & $\mathbf{70.36 \pm 0.03}$ & $\mathbf{71.01 \pm 0.02}$ & $\mathbf{70.89 \pm 0.02}$ & $\mathbf{70.82 \pm 0.01}$ \\
      \midrule
      DRCPE (Hard01) & $(0.5, 0.1)$ & $68.22 \pm 0.26$ & $68.49 \pm 0.23$ & $69.29 \pm 0.13$ & $70.11 \pm 0.05$ \\
      DRCPE (Hard01) & $(0.5, 0.2)$ & $68.38 \pm 0.27$ & $68.64 \pm 0.28$ & $69.36 \pm 0.14$ & $70.13 \pm 0.05$ \\
      DRCPE (Hard01) & $(0.5, 0.3)$ & $69.44 \pm 0.21$ & $69.74 \pm 0.27$ & $69.95 \pm 0.20$ & $70.34 \pm 0.09$ \\
      DRCPE (Hard01) & $(0.5, 0.4)$ & $70.25 \pm 0.04$ & $71.00 \pm 0.02$ & $70.87 \pm 0.02$ & $\mathbf{70.82 \pm 0.01}$ \\
      DRCPE (Hard01) & $(0.5, 0.5)$ & $70.05 \pm 0.03$ & $70.92 \pm 0.02$ & $70.85 \pm 0.01$ & $70.80 \pm 0.01$ \\
      \midrule
      DRCPE (Prob01) & $(0.5, 0.1)$ & $67.53 \pm 0.01$ & $67.65 \pm 0.02$ & $68.69 \pm 0.03$ & $69.96 \pm 0.01$ \\
      DRCPE (Prob01) & $(0.5, 0.2)$ & $67.54 \pm 0.01$ & $67.75 \pm 0.03$ & $68.86 \pm 0.04$ & $70.00 \pm 0.02$ \\
      DRCPE (Prob01) & $(0.5, 0.3)$ & $68.10 \pm 0.23$ & $68.32 \pm 0.18$ & $69.19 \pm 0.10$ & $70.07 \pm 0.03$ \\
      DRCPE (Prob01) & $(0.5, 0.4)$ & $70.21 \pm 0.10$ & $70.97 \pm 0.03$ & $70.85 \pm 0.01$ & $70.80 \pm 0.00$ \\
      DRCPE (Prob01) & $(0.5, 0.5)$ & $69.97 \pm 0.07$ & $70.90 \pm 0.01$ & $70.84 \pm 0.00$ & $70.80 \pm 0.00$ \\
      \midrule
      DRCPE (Squared) & $(0.5, 0.5)$ & $70.17 \pm 0.15$ & $70.96 \pm 0.05$ & $70.86 \pm 0.02$ & $\mathbf{70.82 \pm 0.01}$ \\
      \bottomrule
    \end{tabular}
  \end{adjustbox}
\end{table}

\begin{table}[tbp]
  \centering
  \caption{%
  ImageNet dog-specialist setting with $p = 4\%$. Base/expert accuracies: $67.57\%$/$43.07\%$. Entries are accuracy (\%) means $\pm$ on standard deviation over $11$ random seeds. Bold are highest means.}
  \label{tab:imagenet_4}
  \small
  \setlength{\tabcolsep}{4pt}
  \begin{adjustbox}{max width=\textwidth}
    \begin{tabular}{@{}llcccc@{}}
      \toprule
      & & \multicolumn{4}{c}{Deferral rate (\%)} \\
      \cmidrule(l){3-6}
      Approach & $(\gamma,\gamma^{(\expert)})$ & $25$ & $50$ & $75$ & $90$ \\
      \midrule
      Random & -- & $61.45 \pm 0.00$ & $55.32 \pm 0.00$ & $49.20 \pm 0.00$ & $45.52 \pm 0.00$ \\
      Conf & -- & $63.69 \pm 0.00$ & $56.87 \pm 0.00$ & $48.28 \pm 0.00$ & $44.32 \pm 0.00$ \\
      TwoStage-Exp & -- & $62.08 \pm 0.02$ & $54.44 \pm 0.04$ & $47.42 \pm 0.01$ & $44.46 \pm 0.01$ \\
      Estimate-Diff01 & -- & $66.12 \pm 0.46$ & $61.23 \pm 0.46$ & $53.86 \pm 0.39$ & $48.01 \pm 0.18$ \\
      Estimate-MaxProb & -- & $65.25 \pm 0.06$ & $59.49 \pm 0.36$ & $52.32 \pm 0.52$ & $47.36 \pm 0.33$ \\
      \midrule
      DRCPE (Prob01) & $(0.5, 0.1)$ & $64.75 \pm 0.32$ & $57.07 \pm 0.49$ & $48.84 \pm 0.17$ & $44.74 \pm 0.04$ \\
      DRCPE (Prob01) & $(0.5, 0.2)$ & $65.22 \pm 0.36$ & $57.79 \pm 0.56$ & $49.13 \pm 0.25$ & $44.81 \pm 0.06$ \\
      DRCPE (Prob01) & $(0.5, 0.3)$ & $65.91 \pm 0.15$ & $59.03 \pm 0.36$ & $49.81 \pm 0.33$ & $44.99 \pm 0.11$ \\
      DRCPE (Prob01) & $(0.5, 0.4)$ & $\mathbf{66.51 \pm 0.11}$ & $60.53 \pm 0.17$ & $52.47 \pm 0.38$ & $46.87 \pm 0.43$ \\
      DRCPE (Prob01) & $(0.5, 0.5)$ & $65.61 \pm 0.27$ & $60.87 \pm 0.41$ & $53.69 \pm 0.39$ & $47.98 \pm 0.23$ \\
      \midrule
      DRCPE (Hard01) & $(0.5, 0.1)$ & $64.67 \pm 0.38$ & $57.26 \pm 0.50$ & $48.96 \pm 0.22$ & $44.78 \pm 0.05$ \\
      DRCPE (Hard01) & $(0.5, 0.2)$ & $64.71 \pm 0.38$ & $57.31 \pm 0.52$ & $48.99 \pm 0.23$ & $44.78 \pm 0.05$ \\
      DRCPE (Hard01) & $(0.5, 0.3)$ & $64.79 \pm 0.41$ & $57.45 \pm 0.56$ & $49.10 \pm 0.26$ & $44.81 \pm 0.06$ \\
      DRCPE (Hard01) & $(0.5, 0.4)$ & $64.93 \pm 0.41$ & $57.91 \pm 0.54$ & $49.54 \pm 0.31$ & $44.96 \pm 0.13$ \\
      DRCPE (Hard01) & $(0.5, 0.5)$ & $65.75 \pm 0.33$ & $59.30 \pm 0.36$ & $50.93 \pm 0.41$ & $45.90 \pm 0.39$ \\
      \midrule
      DRCPE (Squared) & $(0.5, 0.5)$ & $66.46 \pm 0.25$ & $\mathbf{61.61 \pm 0.27}$ & $\mathbf{54.26 \pm 0.21}$ & $\mathbf{48.23 \pm 0.08}$ \\
      \bottomrule
    \end{tabular}
  \end{adjustbox}
\end{table}

\begin{table}[tbp]
  \centering
  \caption{%
  ImageNet dog-specialist setting with $p = 2\%$. Base/expert accuracies: $67.57\%$/$34.10\%$. Entries are accuracy (\%) means $\pm$ on standard deviation over $11$ random seeds. Bold are highest means.}
  \label{tab:imagenet_2}
  \small
  \setlength{\tabcolsep}{4pt}
  \begin{adjustbox}{max width=\textwidth}
    \begin{tabular}{@{}llcccc@{}}
      \toprule
      & & \multicolumn{4}{c}{Deferral rate (\%)} \\
      \cmidrule(l){3-6}
      Approach & $(\gamma,\gamma^{(\expert)})$ & $25$ & $50$ & $75$ & $90$ \\
      \midrule
      Random & -- & $59.20 \pm 0.00$ & $50.84 \pm 0.00$ & $42.47 \pm 0.00$ & $37.45 \pm 0.00$ \\
      Conf & -- & $62.94 \pm 0.00$ & $54.44 \pm 0.00$ & $43.10 \pm 0.00$ & $36.91 \pm 0.00$ \\
      TwoStage-Exp & -- & $61.18 \pm 0.04$ & $51.56 \pm 0.04$ & $41.83 \pm 0.02$ & $36.76 \pm 0.01$ \\
      Estimate-Diff01 & -- & $65.42 \pm 0.51$ & $58.57 \pm 0.57$ & $48.15 \pm 0.42$ & $40.28 \pm 0.25$ \\
      Estimate-MaxProb & -- & $64.79 \pm 0.14$ & $57.48 \pm 0.39$ & $47.00 \pm 0.42$ & $39.74 \pm 0.21$ \\
      \midrule
      DRCPE (Hard01) & $(0.5, 0.1)$ & $63.29 \pm 0.51$ & $52.26 \pm 0.72$ & $41.23 \pm 0.31$ & $36.13 \pm 0.07$ \\
      DRCPE (Hard01) & $(0.5, 0.2)$ & $63.39 \pm 0.52$ & $52.40 \pm 0.74$ & $41.29 \pm 0.32$ & $36.14 \pm 0.08$ \\
      DRCPE (Hard01) & $(0.5, 0.3)$ & $63.94 \pm 0.42$ & $53.23 \pm 0.63$ & $41.68 \pm 0.41$ & $36.24 \pm 0.11$ \\
      DRCPE (Hard01) & $(0.5, 0.4)$ & $64.83 \pm 0.22$ & $54.69 \pm 0.41$ & $42.97 \pm 0.59$ & $36.79 \pm 0.32$ \\
      DRCPE (Hard01) & $(0.5, 0.5)$ & $65.55 \pm 0.11$ & $56.97 \pm 0.39$ & $46.15 \pm 0.43$ & $39.03 \pm 0.34$ \\
      \midrule
      DRCPE (Prob01) & $(0.5, 0.1)$ & $63.31 \pm 0.48$ & $52.01 \pm 0.63$ & $41.03 \pm 0.18$ & $36.06 \pm 0.05$ \\
      DRCPE (Prob01) & $(0.5, 0.2)$ & $64.12 \pm 0.32$ & $53.25 \pm 0.62$ & $41.47 \pm 0.29$ & $36.17 \pm 0.07$ \\
      DRCPE (Prob01) & $(0.5, 0.3)$ & $65.19 \pm 0.05$ & $55.45 \pm 0.18$ & $43.68 \pm 0.50$ & $36.94 \pm 0.36$ \\
      DRCPE (Prob01) & $(0.5, 0.4)$ & $65.62 \pm 0.06$ & $57.93 \pm 0.14$ & $47.68 \pm 0.16$ & $40.19 \pm 0.07$ \\
      DRCPE (Prob01) & $(0.5, 0.5)$ & $65.18 \pm 0.20$ & $58.40 \pm 0.25$ & $48.08 \pm 0.27$ & $40.33 \pm 0.13$ \\
      \midrule
      DRCPE (Squared) & $(0.5, 0.5)$ & $\mathbf{65.93 \pm 0.13}$ & $\mathbf{59.18 \pm 0.27}$ & $\mathbf{48.73 \pm 0.17}$ & $\mathbf{40.62 \pm 0.07}$ \\
      \bottomrule
    \end{tabular}
  \end{adjustbox}
\end{table}

\clearpage
\subsection{Non-Dog Subgroup Evaluation}%
\label{app:imagenet_nondog}

We further evaluate the $4\%$ specialist setting on only non-dog examples. Here the base model $h$ is substantially more accurate than the expert $h^{(\expert)}$.
\cref{tab:imagenet_nondog} summarizes these results.
Here, \texttt{DRCPE (Squared)} performs best with similar insights to \cref{app:imagenet-temps}.

\begin{table}[htbp]
  \centering
  \caption{%
  ImageNet dog-specialist setting with $p = 4\%$ evaluated on only non-dog examples. Base/expert non-dog accuracies: $67.14\%$/$38.27\%$. Entries are accuracy (\%) means $\pm$ on standard deviation over $11$ random seeds. Bold are highest means.}
  \label{tab:imagenet_nondog}
  \small
  \setlength{\tabcolsep}{4pt}
  \begin{adjustbox}{max width=\textwidth}
    \begin{tabular}{@{}llcccc@{}}
      \toprule
      & & \multicolumn{4}{c}{Deferral rate (\%)} \\
      \cmidrule(l){3-6}
      Approach & $(\gamma,\gamma^{(\expert)})$ & $25$ & $50$ & $75$ & $90$ \\
      \midrule
      Random & -- & $59.92 \pm 0.00$ & $52.71 \pm 0.00$ & $45.49 \pm 0.00$ & $41.16 \pm 0.00$ \\
      Conf & -- & $62.36 \pm 0.00$ & $53.87 \pm 0.00$ & $43.77 \pm 0.00$ & $39.51 \pm 0.00$ \\
      TwoStage-Exp & -- & $61.07 \pm 0.04$ & $51.76 \pm 0.03$ & $43.18 \pm 0.01$ & $39.74 \pm 0.01$ \\
      Estimate-Diff01 & -- & $63.08 \pm 0.33$ & $57.24 \pm 0.51$ & $49.33 \pm 0.41$ & $43.27 \pm 0.18$ \\
      Estimate-MaxProb & -- & $62.80 \pm 0.03$ & $55.10 \pm 0.35$ & $47.72 \pm 0.58$ & $42.61 \pm 0.34$ \\
      \midrule
      DRCPE (Hard01) & $(0.5, 0.1)$ & $62.42 \pm 0.08$ & $53.41 \pm 0.27$ & $44.26 \pm 0.12$ & $40.00 \pm 0.04$ \\
      DRCPE (Hard01) & $(0.5, 0.2)$ & $62.41 \pm 0.08$ & $53.43 \pm 0.28$ & $44.28 \pm 0.13$ & $40.00 \pm 0.04$ \\
      DRCPE (Hard01) & $(0.5, 0.3)$ & $62.40 \pm 0.05$ & $53.49 \pm 0.28$ & $44.36 \pm 0.16$ & $40.02 \pm 0.05$ \\
      DRCPE (Hard01) & $(0.5, 0.4)$ & $62.43 \pm 0.04$ & $53.70 \pm 0.29$ & $44.73 \pm 0.25$ & $40.15 \pm 0.11$ \\
      DRCPE (Hard01) & $(0.5, 0.5)$ & $62.63 \pm 0.06$ & $54.81 \pm 0.29$ & $46.11 \pm 0.46$ & $41.10 \pm 0.41$ \\
      \midrule
      DRCPE (Prob01) & $(0.5, 0.1)$ & $62.45 \pm 0.05$ & $53.17 \pm 0.22$ & $44.10 \pm 0.13$ & $39.97 \pm 0.04$ \\
      DRCPE (Prob01) & $(0.5, 0.2)$ & $62.54 \pm 0.06$ & $53.59 \pm 0.37$ & $44.34 \pm 0.19$ & $40.02 \pm 0.05$ \\
      DRCPE (Prob01) & $(0.5, 0.3)$ & $62.64 \pm 0.05$ & $54.48 \pm 0.30$ & $44.87 \pm 0.28$ & $40.18 \pm 0.10$ \\
      DRCPE (Prob01) & $(0.5, 0.4)$ & $62.93 \pm 0.13$ & $56.21 \pm 0.27$ & $47.78 \pm 0.47$ & $42.08 \pm 0.46$ \\
      DRCPE (Prob01) & $(0.5, 0.5)$ & $62.66 \pm 0.16$ & $56.81 \pm 0.47$ & $49.16 \pm 0.42$ & $43.26 \pm 0.24$ \\
      \midrule
      DRCPE (Squared) & $(0.5, 0.5)$ & $\mathbf{63.32 \pm 0.33}$ & $\mathbf{57.68 \pm 0.32}$ & $\mathbf{49.77 \pm 0.24}$ & $\mathbf{43.52 \pm 0.07}$ \\
      \bottomrule
    \end{tabular}
  \end{adjustbox}
\end{table}

\end{document}